\documentclass[twoside]{article} 
\usepackage[preprint]{aistats2027}

\usepackage{amsmath,amsfonts,bm}

\def\eqref#1{equation~\ref{#1}}

\def\1{\bm{1}}

\DeclareMathAlphabet{\mathsfit}{\encodingdefault}{\sfdefault}{m}{sl}
\SetMathAlphabet{\mathsfit}{bold}{\encodingdefault}{\sfdefault}{bx}{n}

\usepackage[round]{natbib}

\usepackage{hyperref}
\usepackage{url}

\usepackage{algorithm}
\usepackage{algorithmic}
\usepackage{amssymb}
\usepackage{booktabs,array,longtable}

\usepackage[utf8]{inputenc} 
\usepackage[T1]{fontenc}    

\usepackage{url}            
\usepackage{amsfonts}       
\usepackage{nicefrac}       
\usepackage{microtype}      
\usepackage{xcolor}         

\usepackage{etoc}

\usepackage{microtype}
\usepackage{graphicx}
\usepackage{subcaption}
\usepackage{booktabs}

\usepackage{hyperref}
\hypersetup{hidelinks}

\usepackage{amsmath}
\usepackage{amssymb}
\usepackage{mathtools}

\allowdisplaybreaks
\usepackage{wrapfig}
\usepackage{array,enumitem}
\usepackage{thmtools,thm-restate}
\usepackage{amsthm}
\usepackage{multirow}

\theoremstyle{plain}
\newtheorem{theorem}{Theorem}[section]

\newtheorem{lemma}[theorem]{Lemma}
\newtheorem{corollary}[theorem]{Corollary}
\theoremstyle{definition}
\newtheorem{definition}[theorem]{Definition}
\newtheorem{assumption}[theorem]{Assumption}
\theoremstyle{remark}
\newtheorem{remark}[theorem]{Remark}

\newcommand{\Y}{Y}
\newcommand{\Yt}{\tilde{Y}}
\newcommand{\Yb}{\bar{Y}}
\newcommand{\HGamma}{\hat{\Gamma}}
\newcommand{\HUpsilon}{\hat{\Upsilon}}
\newcommand{\tepsilon}{\tilde{\epsilon}}
\newcommand{\hL}{\hat{L}}

\newcommand{\htau}{\hat{\tau}}
\newcommand{\fL}{\mathcal{L}}
\newcommand{\fhL}{\hat{\mathcal{L}}}
\newcommand{\calG}{\mathcal{G}}
\newcommand{\calH}{\mathcal{H}}
\newcommand{\calhatE}{\mathcal{\hat{E}}}

\begin{document}
\runningtitle{ Distributionally Robust Multi-Objective Optimization}

\twocolumn[

\aistatstitle{From Dual Tracking to Clipping: Provably Faster Distributionally Robust Multi-Objective Optimization}

\aistatsauthor{
  Yufeng Yang\textsuperscript{1} \quad
  Fangning Zhuo\textsuperscript{1} \quad
  Ziyi Chen\textsuperscript{2} \quad
  Heng Huang\textsuperscript{3} \quad
  Yi Zhou\textsuperscript{1}
}

\aistatsaddress{
  \textsuperscript{1}Texas A\&M University \quad
  \textsuperscript{2}A*STAR \quad
  \textsuperscript{3}University of Maryland
}
]

\begingroup
\renewcommand{\thefootnote}{\fnsymbol{footnote}}
\footnotetext[1]{Preprint Under Review}
\footnotetext[1]{Email: \texttt{yufeng.yang@tamu.edu}.}
\footnotetext[1]{Email: \texttt{yi.zhou@tamu.edu}.}
\endgroup

\begin{abstract}
  Multi-objective optimization (MOO) has received growing attention in applications that require learning under multiple criteria. However, most existing MOO formulations do not explicitly account for distributional shifts in the data. We introduce distributionally robust multi-objective optimization (DR-MOO), which minimizes multiple objectives under their respective worst-case distributions. We propose Pareto-type solution concepts for DR-MOO and develop multi-gradient descent algorithms (MGDA) with provable guarantees. Leveraging a Lagrangian dual reformulation, we first design a double-loop MGDA that uses an inner loop to estimate dual variables and achieves a total sample complexity \(\mathcal{O}(\epsilon^{-8})\) for reaching an \(\epsilon\)-Pareto-stationary point. To further improve convergence, we combine large-batch sampling with gradient clipping to accommodate generalized smoothness and control bias in stochastic preference updates, eliminating the need for double sampling.  This yields a single-loop double-clip MGDA with substantially improved sample complexity \(\mathcal{O}(\epsilon^{-4})\). Our theory applies to nonconvex problems without requiring uniformly bounded gradients of the dual objectives. Experiments demonstrate that our methods are competitive with state-of-the-art MGDA baselines.
\end{abstract}


\section{Introduction}

Multi-objective optimization (MOO) has attracted increasing interest in applications such as autonomous driving \citep{automativedriving}, AI for science \citep{ai4science}, recommendation systems \citep{recommendation} and LLM Alignment~\citep{RACO}. Mathematically, MOO aims to optimize a set of \(m\) objectives simultaneously:
\begin{align}
    \min_{\theta \in \mathbf{R}^n}\ \Big\{\Phi(\theta):=\big[\phi^1(\theta),\phi^2(\theta),\ldots,\phi^m(\theta)\big]\Big\},\label{eq: MOO}
\end{align}
where each objective $i$ has a stochastic form \(\phi^i(\theta):=\mathbb{E}_{\xi\sim\mathbb{P}^i}[\ell^i(\theta,\xi)]\), and \(\theta\) denotes the model parameters.
A core challenge in solving MOO is that objectives may exhibit highly disparate gradient magnitudes and directions, so naive gradient-based updates can over-emphasize certain objectives and harm others. Consequently, much of the existing MOO literature focuses on Pareto-type solutions (e.g., Pareto stationary), which capture intrinsic trade-offs and admit tractable first-order algorithms \citep{1stempiricalMGDAwork,paretoMTL,stochasticMGDA}.

However, modern multi-task learning (MTL)~\citep{MTL-DL-survey} commonly adopts a shared representation with task-specific prediction heads. This representation-based structure is naturally fragile to task-dependent distribution shifts. Different tasks may suffer from adversarial perturbations, class imbalance or domain mismatch. Moreover, \citet{robustMTL5} showed that adversarially trained shared representations can theoretically improve generalization to new target tasks, suggesting that robust training is critical for improving reliability of MTL. From an optimization viewpoint, such distribution shifts challenge the standard Pareto optimality notion used in MOO.
In particular, a solution that is Pareto-optimal under the nominal task distributions may fail to remain Pareto-optimal when each objective is evaluated under its own shifted distribution, as illustrated by the toy example in Figure~\ref{fig: Toy-example Pareto-visualization} (Appendix~\ref{Appendix: Toy exmple visualization}). Distributionally robust optimization (DRO)~\citep{DRO1st,DRO2nd,DRO3rd} provides a principled framework for handling distribution shift by optimizing performance under the worst-case distribution within a divergence-based ambiguity set. While DRO is well understood for single objective, extending it to MOO is challenging: one must model objective-wise distribution shifts and define meaningful, analytically tractable Pareto-type solutions under worst-case perturbations. This motivates the following questions:
\begin{itemize}[leftmargin = *, topsep = 0pt]
\item \textit{How should we formulate distributionally robust MOO (DR-MOO) to model objective-wise distribution shifts? What Pareto-type notions are both distributionally robust and algorithmically tractable?}
\end{itemize}

On the algorithmic side, the MOO literature~\citep{gradient-balance-MTL} centers on gradient-balancing methods, particularly MGDA~\citep{1stempiricalMGDAwork,MGDA1stwork,stochasticMGDA}, whose guarantees largely concern standard expected-risk objectives under fixed nominal distributions. DR-MOO instead introduces nested worst-case objectives and auxiliary dual variables, fundamentally altering the optimization geometry. Even with uniformly Lipschitz continuous and smooth losses, the resulting dual objectives need not satisfy either property globally. Thus, existing MGDA analyses do not readily apply, and naive combinations of MGDA and DRO subroutines may yield biased gradients and poor convergence. These challenges, together with the effectiveness of normalization and clipping in single-objective nonconvex DRO~\citep{jinnonconvexdro,qirevistfDRO}, motivate a two-level research question:
\begin{itemize}[leftmargin = *, topsep = 0pt]
    \item \textit{How can MGDA be adapted to the dual structure of DR-MOO while retaining non-asymptotic convergence guarantees, and, building on this adaptation, how can normalization and clipping techniques from single-objective DRO be incorporated into multi-objective gradient balancing to further improve efficiency?}
\end{itemize}

We answer this question by developing a double-loop method that accommodates the dual DRO structure and a single-loop double-clip method that exploits its generalized-smooth geometry. To the best of our knowledge, this is the first systematic formulation of nonconvex DR-MOO with non-asymptotic MGDA-type guarantees. Table~\ref{Table: Comparison with existing methods} in Appendix~\ref{Appendix: comparison with other algorithms} provides a detailed comparison with existing MOO methods.

\textbf{Our Contributions}
\begin{enumerate}[leftmargin = *, topsep = 0pt]
    \item We formulate $f$-divergence-regularized DR-MOO and introduce distributionally robust Pareto dominance and optimality. By leveraging strong duality to obtain a finite-dimensional reformulation, we translate these concepts into a form amenable to first-order optimization. This leads to an algorithmically tractable notion of distributionally robust Pareto stationarity, which serves as our convergence criterion.
    
    \item Leveraging this structure, we develop a double-loop MGDA in which an inner SGD tracks the optimal dual variables, controlling the bias in robust-objective gradients. For outer-loop updates, independent data batches remove sampling bias in Gram-matrix estimates, while independently sampled inner-iterate indices enable partial cancellation of optimization bias through a shared conditional mean. With this design, we prove that the algorithm reaches an \(\epsilon\)-distributionally robust Pareto-stationary point with \(\mathcal{O}(\epsilon^{-8})\) sample complexity.
    
    \item To avoid costly dual tracking, we derive a sufficient surrogate stationarity criterion and propose a single-loop double-clip MGDA method. Our key innovation is to incorporate gradient clipping into both the primal--dual updates and the MGDA preference update, controlling magnitudes of stochastic gradient and differences between successive iterates. These controls allow us to handle potentially unbounded gradient variance and local smoothness, yielding high-probability convergence with $\mathcal{O}(\epsilon^{-4})$ sample complexity without inner-loop tracking or double sampling. Experiments support our theoretical findings and show competitive performance against state-of-the-art MGDA baselines.

\end{enumerate}

\section{Related Works}

\textbf{Distributionally robust optimization (DRO).} Most DRO studies consider ambiguity sets defined by \(f\)-divergences~\citep{DRO3rd,duchiDRO,lamfDRO,large-scaleDRO,qiaaai}, KL divergence~\citep{huKLDRO,qiKLDRO}, Wasserstein distance~\citep{blanchetwasserstein,gaoregularizeWDRO,gaoWDRO}, and its regularized variants~\citep{sinkhornDROwang,wangzihansinkhornDRO,yufengSinkhornDRO}. A central challenge is to transform the resulting worst-case optimization over probability distributions into a tractable formulation supporting practical stochastic algorithms. Even in the single-objective setting, existing DRO methods face computational challenges. For example,  \citet{DRO3rd,mini-maxWDRO} explicitly parameterize the adversarial distribution, leading to complexity that scales with the dataset size, while the bilevel formulation \citep{yufengSinkhornDRO} requires a double-loop procedure. Other tractable formulations often rely on restrictive assumptions, \citet{large-scaleDRO,sinkhornDROwang} require globally bounded and convex losses, whereas the nonconvex analysis of \citet{qiaaai} retains global boundedness. Notably, \citet{jinnonconvexdro,qirevistfDRO} exploit strong-duality-based reformulations of \(f\)-divergence DRO~\citep{DROBook,duchiDRO,large-scaleDRO} to provide convergence guarantees for nonconvex, potentially unbounded losses. Extending such tractable formulations and provable algorithms to MOO, however, remains underexplored. Recently, \citet{DRMOOLatorre} introduced Pareto-optimality and stationarity notions for DR-MOO with convex losses, but characterized them through existence arguments on the primal robust objectives, which still involve a generally intractable functional maximization over distributions. Consequently, no practical algorithm or non-asymptotic convergence analysis was provided. Motivated by strong duality, we instead derive a finite-dimensional dual formulation, characterize robust Pareto optimality and stationarity through the dual objectives, and develop practical algorithms with non-asymptotic guarantees for \(f\)-divergence-regularized DR-MOO.

\textbf{Gradient-Balancing Methods in MOO.}
Gradient-balancing methods construct a common descent direction by adapting the model update and preference vector across objectives. Following the formulation of MTL as MOO~\citep{1stempiricalMGDAwork}, many methods build on MGDA~\citep{MGDA1stwork} and fall into two broad lines. The first develops stochastic MGDA variants with convergence guarantees. Stochastic MGDA~\citep{stochasticMGDA} requires increasing batch sizes, whereas MoDo~\citep{MODO} and SDMGrad~\citep{SDMkaiyiJi} use double sampling to control bias, and MoCo~\citep{MOCO} uses moving-average gradient tracking. Their analyses, however, typically assume uniformly bounded objective gradients, which may fail for dual DRO objectives whose gradients grow with the loss. The second line mitigates gradient conflicts through structured manipulations or interpretable optimization formulations, including CAGrad~\citep{CAGrad}, IMTL-Grad~\citep{IMTLGrad}, PCGrad~\citep{PCGrad}, GradDrop~\citep{GradDrop}, RotoGrad~\citep{Rotograd}, and NashMTL~\citep{nashMTL}. In practice, the LibMTL~\citep{LibMTL} implementations of SDMGrad and NashMTL normalize gradient Gram matrices to improve stability. Although effective empirically, these mechanisms remain insufficiently understood theoretically, particularly in stochastic settings with unbounded gradients. This theory-practice gap, together with the analysis of \citet{qimgda} that relaxes uniform boundedness by relating gradient norms to function-value gaps under generalized smoothness, motivates our provable gradient-clipping scheme tailored to dual DR-MOO objectives.

\section{Distributionally Robust MOO}


Distributionally robust optimization (DRO) has been widely adopted in machine learning to ensure robust model performance under distribution shifts. Specifically, consider a supervised learning problem with the nonconvex loss function denoted by $\ell(\theta;\xi)$, where $\theta\in \mathbf{R}^n$ denotes the collection of model parameters and $\xi$ corresponds to a data sample that follows an underlying data distribution $\mathbb{Q}$. We are interested in the following regularized DRO problem
\begin{align}
  &\text{(DRO)}: \quad \min_{\theta \in \mathbf{R}^n} \sup_{\mathbb{Q}} H({\theta},\mathbb{Q}), \nonumber\\
  &\quad
   \text{where} ~ H({\theta},\mathbb{Q}) = \mathbb{E}_{\xi \sim \mathbb{Q}} \bigr [ \ell(\theta;\xi)\bigr] - \lambda D_f(\mathbb{Q} | \mathbb{P}).
    \label{primal1}
\end{align}
Here, $\lambda>0$ is regularization hyperparameter and $D_f$ denotes the $f$-divergence that measures the 
discrepancy between the nominal data distribution $\mathbb{P}$ and the underlying data distribution $\mathbb{Q}$. Such minimax formulation $\min_{\theta}\sup_{\mathbb{Q}}$
seeks model parameters $\theta$ that perform well under the worst-case distribution, thereby improving robustness to distribution shifts away from $\mathbb{P}$.
To model robustness in the presence of multiple, potentially conflicting objectives, we extend DRO to the multi-objective setting. Specifically, we formulate distributionally robust multi-objective optimization as
\begin{align}
   &(\text{DR-MOO}):\nonumber\\
   &\quad~\min_{{\theta} \in \mathbf{R}^n} \sup_{\mathbb{Q}^1,\ldots,\mathbb{Q}^m} 
   \big[ H^1({\theta},\mathbb{Q}^1), \ldots, H^m({\theta},\mathbb{Q}^m)\big],
\label{eq: primal DR-MOO}
\end{align}
where $H^i({\theta}, \mathbb{Q}^i)
=\mathbb{E}_{\xi\sim\mathbb{Q}^i}[\ell^i({\theta},\xi)]
-\lambda D_f(\mathbb{Q}^i\mid \mathbb{P}^i),\forall i\in[m]$.
Unlike scalarized DRO, this formulation preserves vector-valued trade-offs: each task is evaluated under its own worst-case distribution, and robustness is enforced over Pareto trade-offs rather than a fixed weighted aggregation. Building on the MOO framework, DR-MOO naturally adopts Pareto-based solution concepts~\citep{marucsciac1982fritz}. However, distributional uncertainty causes objective values to vary across distributions, motivating an extension of Pareto dominance and optimality to characterize the robust Pareto front.
\begin{definition}[Distributionally robust Pareto-dominance]\label{Def: Pareto-dominance in DR-MOO}
In DR-MOO, for $\theta_1$, $\theta_2\in \mathbf{R}^n$, we say that ${\theta}_1$ distributionally robustly dominates ${\theta}_2$ if and only if $\sup_{\mathbb{Q}^i}H^i({\theta}_1,\mathbb{Q}^i)\leq \sup_{\mathbb{Q}^i}H^i({\theta}_2,\mathbb{Q}^i)$ for \textit{all} $i\in \{1,\dots,m\}$, and $\sup_{\mathbb{Q}^j}H^j({\theta}_1,\mathbb{Q}^j) < \sup_{\mathbb{Q}^j}H^j({\theta}_2,\mathbb{Q}^j)$ for \textit{some} $j\in \{1,\dots,m\}$.
     
\end{definition}
\begin{definition}[Distributionally robust Pareto-optimal]\label{Def: Pareto-optimal under distribution uncertain}
A solution ${\theta}^*$ is distributionally robust Pareto-optimal if no other solution distributionally robustly dominates it.
\end{definition}
Intuitively, these definitions compare model parameters after adversarially perturbing each objective’s data distribution towards its worst-case distribution. When there is no distribution shift, they reduce to the standard Pareto-dominance and optimality condition.  Although extending Pareto notions to the distributionally robust setting is natural, directly optimizing the primal robust objectives
\(\sup_{\mathbb{Q}^i} H^i(\theta,\mathbb{Q}^i)\) remains challenging.
Even when the underlying losses \(\ell^i(\theta;\xi)\) are smooth, the induced robust objectives involve distributional maximization and may be nonsmooth.
These challenges motivate defining Pareto stationarity for DR-MOO through a primal--dual transformation, as detailed in the next subsection.



\subsection{Dual Formulation and Distributionally Robust Pareto-Stationarity}


In this subsection, we leverage Lagrangian duality to reformulate DR-MOO into an equivalent dual form that facilitates gradient-based algorithm design. Throughout, we impose the following standard assumptions on the loss functions, which are widely adopted in the DRO literature \citep{jinnonconvexdro,qirevistfDRO}.
\begin{assumption}\label{assumption 1}
For each $i$, the sample loss $\ell^i(\theta,\xi)$ is finite $\mathbb{P}^i$-almost surely. The loss functions and divergence function in \eqref{eq: primal DR-MOO} further satisfy the following regularity conditions.
\begin{itemize}[leftmargin = *, topsep = 0pt]
    \item There exists $G,L>0$ such that the loss function $\ell^{i}(\cdot, \xi)$ is $G$-Lipschitz continuous and $L$-smooth for all $\xi$ and all $i\in[m]$.
        \item For any ${\theta}$, the variance of $\ell^i(\cdot),~\forall i\in[m]$ with respect to sample $\xi~\sim \mathbb{P}^i$ is bounded by $\kappa^2$. 
    \item The divergence base function $f:[0,+\infty)\rightarrow (-\infty,+\infty]$ is convex and satisfies $f(1)=0, f(0)=\lim_{t\rightarrow 0^{+}}f(t)$. Additionally, its convex conjugate $f^*(\cdot)$ is $M$-smooth.
\end{itemize}    
\end{assumption}
{Several $f$-divergences satisfy the $M$-smoothness assumption, including the $\chi^2$-divergence, smoothed CVaR divergence, etc.} Under Assumption \ref{assumption 1}, \citet{DROBook,DRO3rd,large-scaleDRO} show that the regularized DRO problem in \eqref{primal1} admits the following equivalent dual formulation.
\begin{align}
  &\sup_{\mathbb{Q}} H({\theta},\mathbb{Q}) =
  \phi(\theta),\nonumber\\
  &\quad ~\text{and}~ \phi(\theta)=
  :\min_{\eta\in \mathbf{R}} \lambda \mathbb{E}_{\xi\sim\mathbb{P}}\big[f^*\big(\frac{\ell({\theta};\xi)-\eta}{\lambda}\big)\big]+\eta,\label{eq: single-objective dual DRO}
\end{align}
where $\eta$ denotes the dual variable. Leveraging these facts, primal DR-MOO \eqref{eq: primal DR-MOO} can be rewritten in the following equivalent dual form
\begin{align}
  &\text{(Dual of DR-MOO):}\nonumber\\
  &\quad\min_{\theta} \Big\{ \Phi({\theta}) := \big[\phi^1({\theta}), \phi^2({\theta}), ..., \phi^m({\theta})\big] \Big\},
 \label{eq: dual DR-MOO}
\end{align}
where $\phi^i({\theta}) = \min_{\eta^i\in \mathbf{R}} \lambda \mathbb{E}_{\xi\sim\mathbb{P}^i}[f^*(\tfrac{\ell^i({\theta};\xi)-\eta^i}{\lambda})]+\eta^i$.
In this dual form, DR-MOO can be treated as a standard MOO problem with \(m\) dual objectives. Moreover, under standard regularity conditions, \citet{jinnonconvexdro} showed that each $\phi^i$ is differentiable, with
\begin{equation}
\nabla \phi^i(\theta)=\nabla_\theta \fL^i(\theta,\eta^{i,*}),
\quad
\eta^{i,*}\in\arg\min_{\eta^i\in\mathbb{R}}\fL^i(\theta,\eta^i),
\nonumber
\end{equation}
where
$\fL^i(\theta,\eta^i):=\lambda\mathbb{E}_{\xi\sim\mathbb{P}^i}
[f^*(\tfrac{\ell^i(\theta;\xi)-\eta^i}{\lambda})]+\eta^i$. This dual representation replaces the infinite-dimensional maximization over
$\mathbb{Q}^i$ with a finite-dimensional problem and leads to the following
Pareto-stationarity condition for DR-MOO.
\begin{definition}[$\epsilon$-Distributionally robust Pareto-stationary]\label{def: dr-pstationary}
   Denote $\mathcal{W}$ as the probability simplex. We say that $\theta$ is an $\epsilon$-distributionally robust Pareto-stationary solution of DR-MOO if it is an $\epsilon$-Pareto-stationary solution of the dual formulation, i.e., there exists $w \in \mathcal{W}$ such that$
    \| \sum_{i=1}^{m} w^i \nabla \phi^i(\theta)\| \leq \epsilon$.
\end{definition}
Related notions of Pareto dominance, optimality, and stationarity under
distributional uncertainty were also studied by \citet{DRMOOLatorre}. Their framework
considers convex, possibly nonsmooth losses $\ell^i(\cdot)$ and characterizes
these notions through the existence of worst-case distributions
$\mathbb{Q}^{i,*}$ in the primal formulation. In contrast, we consider
nonconvex, $L$-smooth losses $\ell^i(\cdot)$ and define stationarity through the dual objectives
$\phi^i(\theta)$ induced by the $f$-divergence reformulation. Crucially, their
primal characterization requires identifying the worst-case distributions by
solving infinite-dimensional maximization problems, which limits its use in
gradient-based optimization. Our dual characterization avoids this obstacle and
enables the MGDA-type algorithms developed in the next two sections, together
with provable convergence to $\epsilon$-Distributionally Robust Pareto-stationary solutions.
\section{Double-Loop MGDA}\label{sec: Double-loop algorithm}
For the standard MOO problem, a central challenge is handling conflicts among objective gradients. To address this issue, a variety of multiple-gradient descent algorithms (MGDA) has been proposed \citep{MGDA1stwork,CAGrad,MOCO,MODO,SDMkaiyiJi,stochasticMGDA,qimgda}. The key idea is to compute an update direction \(d\) that maximizes the worst-case improvement across objectives, i.e.,
$
\max_{d\in\mathbf{R}^{{n}}}\frac{1}{\gamma}\min_{i\in[m]}\{\phi^i(\theta)-\phi^i(\theta-\gamma d)\},
$
where \(\gamma\) denotes the learning rate. 
This leads to the following MGDA update rule widely adopted in~\citet{stochasticMGDA,SDMkaiyiJi,qimgda,MODO}
\begin{align}
&\text{(MGDA):} \quad d^* =\nabla\Phi(\theta)w_{\rho}^*,\nonumber \\
&\quad\text{where} ~ w_{\rho}^* =\arg\min_{w\in \mathcal{W}}\frac{1}{2}\|\nabla \Phi(\theta)w\|^2+\frac{\rho}{2}\|w\|^2, \label{eq: MGDA surrogate}
\end{align}
where \(\ell_2\)-regularization term ensures that \(w_\rho^\star\) is uniquely defined. To avoid solving~\eqref{eq: MGDA surrogate} exactly, we use the one-step preference update
\begin{align}
    w_{t+1}=\Pi_\mathcal{W}\big(w_t-\beta(\nabla\Phi(\theta_t)^{\top}\nabla\Phi(\theta_t)w_t+\rho w_t)\big).
    \label{eq: MGDA one-step approximation}
\end{align}
However, applying MGDA~\eqref{eq: MGDA one-step approximation} to the dual formulation~\eqref{eq: dual DR-MOO} faces two challenges. First, computing the exact minimizers \(\eta^{i,*}\) is generally impractical. Second, estimating the Gram-matrix \(\nabla\Phi(\theta)^\top\nabla\Phi(\theta)\) using sample-averaged stochastic gradients can introduce bias into the preference update.
\begin{algorithm}[t]
\hrule
\vspace{0.3em}

\caption{Double-loop MGDA for DR-MOO}
\label{alg1}

\vspace{-0.5em}
\vspace{0.3em}

\footnotesize
\begin{algorithmic}[1]
\STATE \textbf{Initialize} $\theta_0$, $w_0$,
$\{\eta^i_{0,0}\}_{i=1}^m$, $\rho$, and learning rates
$\alpha,\beta,\gamma$.

\FOR{$t=0,\ldots,T-1$}
    \STATE Run inner-SGD repeatedly to obtain independent trajectories.
    \FOR{$d=0,\ldots,D-1$ and $i\in[m]$}
        \STATE Evaluate
        $V^i_d=\nabla_{\eta}L^i(
        \theta_t,\eta^i_{t,d};\xi^i_{t,d})$.
        \STATE
        $\eta^i_{t,d+1}=\eta^i_{t,d}-\gamma V^i_d$.
    \ENDFOR

    \STATE Sample $d,\bar d,\tilde d
\overset{\mathrm{i.i.d.}}{\sim}\operatorname{Unif}\{0,\ldots,D-1\}$.

    \STATE Draw
    $\xi_t,\bar\xi_t,\tilde\xi_t$ independently.

    \STATE
    $Y_t=\nabla_{\theta}\fL(
    \theta_t,\eta_{t,d};\xi_t)$.

    \STATE
    $\bar Y_t=\nabla_{\theta}\fL(
    \theta_t,\eta_{t,\bar d};\bar\xi_t)$.

    \STATE
    $\widetilde Y_t=\nabla_{\theta}\fL(
    \theta_t,\eta_{t,\tilde d};\tilde\xi_t)$.

    \STATE
    $\theta_{t+1}=\theta_t-\alpha Y_t w_t$.

    \STATE
    $\displaystyle
    w_{t+1}=\Pi_{\mathcal W}\!\big(
    w_t-\beta(
    \bar Y_t^\top\widetilde Y_t w_t+\rho w_t
    )\big)$.
    \STATE $\eta_{t+1,0} = \eta_{t,d}$
\ENDFOR
\end{algorithmic}
\vspace{0.3em}
\end{algorithm}
To address these challenges, we propose double-loop MGDA (Algorithm~\ref{alg1}), which combines inner-loop dual tracking with independent sampling of data and dual variable indices. Its two loops play complementary roles in controlling the errors arising from approximate dual minimization and stochastic Gram-matrix estimation, as detailed below.
\begin{itemize}[leftmargin = *, topsep = 0pt]
    \item \textbf{Inner loop (dual tracking).} Since the exact minimizers \(\eta^{i,*}\) are unavailable, we run vanilla SGD \citep{lanSGD} (lines 4-7 in Algorithm~\ref{alg1}) to approximately solve \(\min_{\eta^i} \fL^i(\theta,\eta^i)\). This inner loop tracks the dual minimizers as \(\theta_t\) evolves, controlling the bias between \(\nabla_\theta \fL(\theta_t,\eta_t)\) and the desired \(\nabla \Phi(\theta_t)\).

    \item \textbf{Independent index and batch sampling.} While dual tracking controls optimization bias, reusing the same stochastic gradient estimator with dependent randomness introduces a covariance term into the
    expected Gram-matrix estimate. We therefore adopt a double-sampling strategy \citep{SDMkaiyiJi,MODO}, independently sampling both inner-iterate indices and data batches. Specifically, we draw three independent index--batch pairs to
    construct conditionally independent estimators $Y_t$, $\bar Y_t$, and
    $\tilde Y_t$ such that
    $\mathbb{E}[\bar{Y}_t^{\top}\tilde{Y}_tw_t|\calH_t]=\mathbb E[\bar Y_t\mid\mathcal H_t]^\top
\mathbb E[\tilde Y_t\mid\mathcal H_t]w_t$. Thus, the cross-product $\bar Y_t^\top\tilde Y_t$ eliminates the covariance
    term arising from estimator reuse and provides an unbiased stochastic Gram-matrix estimate, yielding convergence without requiring dynamically growing minibatches~\citep{stochasticMGDA}.
\end{itemize}
\subsection{Convergence Analysis}
We first establish the local regularity of the dual objectives in
Appendix~\ref{Appendix: Relevant Properties of Phi}. Specifically,
Lemmas~\ref{lemma: gradient bound of phi} and~\ref{lemma: smooth of phi} show
that, at primal-dual pair $(\theta,\eta_{\theta}^{i,*})$, $\phi^i$ has locally bounded gradient,
$\|\nabla\phi^i(\theta)\|\leq G$ and locally $L_0$-smooth in terms of $\theta$-gradient.
Consequently, once the inner-loop SGD provides a sufficiently accurate estimate of $\eta_{\theta}^{i,*}$, the analysis builds on techniques for double-sampling stochastic MGDA with nonconvex $L_0$-smooth objectives~\citep{MODO}, with additional control of the bias from inexact dual estimation. This yields the following convergence result. The full statement, hyperparameter choices, and proof are provided in Appendix~\ref{Appendix: proof of convergence algorithm 1}.
\begin{restatable}[Convergence of Algorithm~\ref{alg1}]{theorem}{convgdoubleloopmgda}\label{thm: convergence of algorithm 1}
    Let Assumption \ref{assumption 1} hold. Denote $\Delta_{\theta_0}$ as the initial outer-objective gap and $\overline\Delta_\eta$ an upper bound on the inner-loop warm-start gaps. For any $\epsilon >0$,
    set $\rho=\mathcal{O}( \epsilon^2)$, $\beta =\mathcal{O}(\epsilon^2)$,
    $\alpha = \mathcal{O}(\epsilon^2)$, and $\gamma = \mathcal{O}(\epsilon^2)$ for Algorithm~\ref{alg1}. Then, after $T=\Theta(\max\{\Delta_{\theta_0}\alpha^{-1}\epsilon^{-2},\beta^{-1}\epsilon^{-2}\})$ outer iterations, each associated with $D=\Theta(\overline{\Delta}_{\eta}\gamma^{-1}G^4\epsilon^{-2})$ inner iterations, we have \[\frac{1}{T}\mathbb{E}\big[\sum_{t=0}^{T-1}\|\nabla \Phi(\theta_t)w_t\|^2\big]\leq 94\epsilon^2.\]
\end{restatable}
Theorem~\ref{thm: convergence of algorithm 1} provides a baseline guarantee for applying MGDA directly to the dual DR-MOO formulation~\eqref{eq: dual DR-MOO}. Its outer-loop complexity \(T=\mathcal{O}(\epsilon^{-4})\) matches the rates established for stochastic MGDA variants~\citep{SDMkaiyiJi,MOCO,MODO,qimgda}, whereas its higher sample complexity of \(\mathcal{O}(\epsilon^{-8})\) arises from the inner-loop accuracy needed to control bias due to inexact dual variables. Nevertheless, for fixed \(\theta\), each inner problem \(\min_{\eta^i}\fL^i(\theta,\eta^i)\) is one-dimensional and convex, and thus typically requires few iterations and little memory to reach sufficient accuracy. The ablation study in Figure~\ref{fig: batch-size and inner-step ablation} (Appendix~\ref{Appendix: inner-iteration and wall-clock}) shows that the number of inner steps has little effect on empirical convergence, while the wall-clock comparison in Figure~\ref{fig: Wall-clock} (Appendix~\ref{Appendix: inner-iteration and wall-clock}) indicates that dual-variable updates are lightweight relative to updates of \(\theta_t\) and \(w_t\) in deep learning settings. Together with the constant \(\mathcal{O}(1)\) batch size sufficient for our convergence guarantee, these results support the practicality of double-loop MGDA in peak-memory-constrained settings. Next, we discuss the main technical novelties in our analysis.


\begin{algorithm}[t]
\hrule
\vspace{0.3em}

\caption{Double-Clip MGDA for DR-MOO}
\label{alg1-3}

\vspace{-0.5em}
\hrule
\vspace{0.3em}

\footnotesize
\begin{algorithmic}[1]
\STATE \textbf{Initialize} $\theta_0$, $\eta_0$, $w_0$, $\rho$, $\beta,\gamma$,
\STATE Clipping Rule:
$\alpha_t=\min\!\{c_1,\frac{c_2}{\|X_t w_t\|}\}$,
$\mu_t=\min\!\{f_1,\frac{f_2}{\|Z_t w_t\|}\}$.
\FOR{$t=0,\ldots,T-1$}
    \STATE Evaluate
    $Z_t=\nabla_{\eta}\fhL(\theta_t,\eta_t;\{\xi_t\}_B)$ with $B=N_2$.
    \STATE $\eta_{t+1}=\eta_t-\gamma\mu_t Z_t w_t$.
    \STATE Evaluate
    $X_t=\nabla_{\theta}\fhL(\theta_t,\eta_{t+1};\{\bar{\xi}_t\}_B)$ with $B=N_1$.
    \STATE $\theta_{t+1}=\theta_t-\gamma\alpha_t X_t w_t$.
    \STATE $w_{t+1}=
    \Pi_{\mathcal{W}}\big(w_t-\beta(
    \alpha_t X_t^\top X_t w_t
    +\mu_t Z_t^\top Z_t w_t
    +\rho w_t
    )\big)$.
\ENDFOR
\end{algorithmic}

\vspace{0.3em}
\hrule
\end{algorithm}
\begin{itemize}[leftmargin = *, topsep = 0pt]
\item \textbf{Bias from inexact dual tracking.} While double sampling removes the \emph{stochastic} bias in MGDA, it does not eliminate the \emph{optimization} bias induced by approximate dual solutions, as both \(\nabla_{\theta}\fL(\theta_t,\eta_{t,\bar d})\) and \(\nabla_{\theta}\fL(\theta_t,\eta_{t,\tilde d})\) may differ non-negligibly from \(\nabla\Phi(\theta_t)\). Prior work~\citep{qimgda} assumes access to unbiased stochastic gradients and can exploit martingale structure for terms like $\mathbb{E}[\sum_{t=0}^{\tau-1}(Y_{t}^i-\nabla_{\theta} \fL^i(\theta_t,\eta_{t,d}))w_t^i]$.  In contrast, because the exact dual minimizers $\eta^{i,*}$ are unavailable, our analysis involves bias sequences of the form $\mathbb{E}[\sum_{t=0}^{T-1}(\nabla_{\theta}\fL^i(\theta_t,\eta^i_{t,d})-\nabla\phi^i(\theta_t))w_t^i]$, whose errors may accumulate over time. Sampling \(d,\bar d,\tilde d\) independently  ensures that the corresponding dual-gradient biases have the same conditional mean given \(\calH_t\). This enables partial cancellation between the bias terms arising from the parameter and preference updates after taking conditional expectation (see~\eqref{eq: key insights of algorithm 1}). 
After that, Young's inequality bounds the remaining optimization bias terms by $\mathcal{O}(\|\nabla\Phi(\theta_t)w_t\|^2)+\mathcal{O}(\epsilon^2)$, provided that the dual-tracking error is $\mathcal{O}(\epsilon^2)$, which can be achieved with $\mathcal{O}(\epsilon^{-4})$ inner SGD iterations.
\end{itemize}

\section{Double-Clip MGDA}\label{sec: single-loop algorithm}
The preceding double-loop analysis shows that accurately estimating \(\nabla\Phi(\theta)\) requires tight dual tracking, resulting in conservative total sample complexity. Extending the rescaled surrogate condition of \citet{jinnonconvexdro} to multi-objective setting yields a sufficient Pareto-stationarity criterion, which underlies our single-loop method (Algorithm~\ref{alg1-3}).
\begin{restatable}[Convergence Criterion Reformulation]{lemma}{reformedcondition}\label{lemma: optimality condition}
    Let Assumption \ref{assumption 1} hold. Given $\epsilon$, if one can obtain $({\theta},{\eta})$ and preference vector $w\in \mathcal{W}$ such that
\begin{align}
    G\!\sum_{i=1}^{m}\!|w^i\nabla_{\eta^i} \fL^i({\theta}, \eta^i) | + \| \!\sum_{i=1}^{m} w^i \nabla_{{\theta}}\fL^i({\theta},\eta^i)  \|\leq \epsilon,
    \label{eq: epsilon Optimal condition}
\end{align}
then the corresponding ${\theta}$ satisfies
$\| \nabla\Phi({\theta})w \|\leq\epsilon$. Furthermore, the condition in \eqref{eq: epsilon Optimal condition} can be achieved by optimizing the rescaled function $\fhL(\theta,\eta) = \fL(\theta, G\sqrt{m}\eta)$ such that
$\|\nabla_{{\theta}, {\eta}} \fhL({\theta}, \eta)w \|\leq {\epsilon}/{\sqrt{2}}$.
\end{restatable} 
Compared with the single-objective construction of \citet{jinnonconvexdro}, Lemma~\ref{lemma: optimality condition} introduces an additional factor of \(\sqrt{m}\) to account for aggregating dual gradients across \(m\) objectives. This scaling ensures that stationarity of the weighted surrogate $\fhL(\theta,\eta)$ implies Pareto stationarity of $\Phi(\theta)$. With this stationarity connection established, we exploit the reformulation's unbiased stochastic gradient estimators to jointly update $(\theta,\eta)$, and $w$ in a single loop, without accurately tracking $\eta^{i,*}$. We further incorporate gradient clipping to control stochastic gradient errors without double sampling, leading to Algorithm~\ref{alg1-3}, whose key features are summarized below.
\begin{itemize}[leftmargin = *, topsep = 0pt]
\item \textbf{Single-loop update.} 
Following the analysis of \citet{qirevistfDRO}, the rescaled function $\fhL^i(\theta,\eta^i)$ is generalized $(\hL_0,\hL_1)$-smooth in $\theta$ and $\hL_2$-smooth in $\eta$ (Lemma~\ref{lemma: generalized-smooth} in Appendix~\ref{sec: property of hat L}). Specifically, the effective smoothness of the $\theta$-gradient grows with $|\nabla_{\eta}\fhL^i(\theta,\eta^i)|$. This motivates an additional one-step balanced-gradient update on $\eta$ at each iteration to control the $\eta$-gradient contribution to this smoothness bound and promote convergence of $\nabla_{\eta}\fhL(\theta_t,\eta_t)w_t$ to $\boldsymbol{0}$.
\item 
\textbf{``Double (gradient)-clipping'' when updating $(\theta_t,\eta_t)$ and $w_t$.}
Inspired by normalized methods for heavy-tailed noise and generalized smoothness~\citep{jinnonconvexdro,generalizedsmooth,whyadamisgood}, we combine gradient clipping with large-batch sampling to handle the affine-bounded variance of the stochastic gradients of $\fhL(\theta,\eta)$ (Lemma~\ref{lemma: affine bounded noise} in Appendix~\ref{sec: property of hat L}). The two mechanisms play complementary roles: clipping controls the magnitudes of stochastic updates, while sufficiently large batches control the remaining terms associated with stochastic gradients. Unlike single-objective methods that clip only the gradients used to update the optimization variables, our method also clips the stochastic Gram-matrix contributions to the preference update. Specifically, we clip the balanced gradients $X_tw_t$ and $Z_tw_t$ when updating $\theta_t$ and $\eta_t$, then reuse the clipping factors $\alpha_t$ and $\mu_t$ to control the corresponding contributions $X_t^\top X_tw_t$ and $Z_t^\top Z_tw_t$ to the $w_t$ update. The additional term $\mu_tZ_t^\top Z_tw_t$ incorporates dual-gradient information, allowing $w_t$ to balance both $\nabla_{\theta}\fhL(\theta_t,\eta_{t+1})w_t$ and $\nabla_{\eta}\fhL(\theta_t,\eta_t)w_t$. This design eliminates the need for double sampling.

\end{itemize}
\subsection{Convergence Analysis}
We now establish the convergence of Algorithm~\ref{alg1-3} under the generalized-smooth geometry of \(\fhL(\theta,\eta)\), without assuming uniformly bounded gradients as in prior analyses~\citep{MODO,MOCO,SDMkaiyiJi}. Instead, we use the relationship between gradient norms and function-value gaps~\citep{Haochuangensmooth} to jointly control stochastic gradients and function values with high probability. The full result, hyperparameter choices, and proof appear in Appendix~\ref{Appendix: Convergence of Algorithm 3}.
\begin{restatable}[Convergence of Algorithm~\ref{alg1-3}]{theorem}{convgalgorithm3}\label{thm: convergence of alg3}
     Let Assumption~\ref{assumption 1} hold. Denote $\hat{\Delta}_{\theta_0,\eta_0}$ as the initial objective gap for $\fhL$. Fix $\delta\in(0,1)$, for any $\epsilon>0$ satisfying $\delta\epsilon\leq\min\{1, \mathcal{O}({m^{-1}\Delta_{\theta_0,\eta_0}^{-1/2}}),\mathcal{O}(m^{-1}\gamma^{-1/2}\beta^{1/2})\}$, set the hyperparameters in Algorithm~\ref{alg1-3} as $c_1,f_1=\tfrac{1}{2}$, $c_2,f_2=\delta\epsilon$, 
    $\rho=\mathcal{O}(\delta^2\epsilon^2)$, $\beta,\gamma = \mathcal{O}(1)$. Choose batch sizes $N_1,N_2=\mathcal{O}(\delta^{-3}\epsilon^{-2})$. Then, after
    $T
    =\Theta(\max\{ \hat{\Delta}_{\theta_0,\eta_0}\gamma^{-1}\delta^{-2}\epsilon^{-2},\beta^{-1}\delta^{-2}\epsilon^{-2}\})$ iterations, with probability at least $1-\delta$, we have 
    \[\frac{1}{T}\sum_{t=0}^{T-1}\|\nabla_{\theta,\eta}\fhL(\theta_t, \eta_{t+1})w_t\|\leq 34\epsilon.\]
    \vspace{-1.5em}
\end{restatable}
Theorem~\ref{thm: convergence of alg3} establishes an \(\mathcal{O}(\epsilon^{-2})\) iteration complexity and an \(\mathcal{O}(\epsilon^{-4})\) total sample complexity for Algorithm~\ref{alg1-3}, matching the complexity of smooth nonconvex stochastic optimization~\citep{lowerbound}. Under our current convergence analysis framework targeting
\begin{align}
    \frac{1}{T}\mathbb{E}[\sum_{t=0}^{T-1}\gamma\|\nabla_{\theta,\eta}\fhL(\theta_t,\eta_{t+1})w_t\|^2]\leq\mathcal{O}(\gamma\epsilon^2),
    \label{eq: general analysis framework}
\end{align}
and the adopted learning-rate design, batch sizes \(N_1,N_2=\mathcal{O}(\epsilon^{-2})\) are needed to control the stochastic errors and preserve the stated iteration complexity. Despite this theoretical bottleneck, the batch-size ablation in Figure~\ref{fig: batch-size and inner-step ablation} (Appendix~\ref{Appendix: ablation on batch size and inner-steps}) shows that \(N_1,N_2=256\) provides stable performance across all experiments. We next explain the key technical novelties underlying our convergence guarantee.
\begin{itemize}[leftmargin = *, topsep = 0pt] 
\item \textbf{Bias control via double clipping and large-batch sampling.} To establish the target bound in~\eqref{eq: general analysis framework}, the
terms arising from the preference update, including
$\gamma\beta\mathbb{E}[\|\mu_tZ_t^\top Z_tw_t\|^2]$ and
$\gamma\beta\mathbb{E}[\|\alpha_tX_t^\top X_tw_t\|^2]$, must be controlled at
the $\mathcal{O}(\gamma\epsilon^2)$ level. Using the clipped factor $\mu_t$, we obtain
\begin{align}
\gamma\beta\mathbb{E}[\|\mu_tZ_t^\top Z_tw_t\|^2]
&\leq
2\gamma\beta\delta^2\epsilon^2
\mathbb{E}[\|Z_t-\nabla_{\eta}\fhL(\theta_t,\eta_t)\|_F^2]\nonumber\\
&\quad+
2\gamma\beta\delta^2\epsilon^2
\mathbb{E}[\|\nabla_{\eta}\fhL(\theta_t,\eta_t)\|_F^2].\nonumber
\end{align}
Here, the factor $\delta^2\epsilon^2$ follows from
$\mu_t\|Z_tw_t\|\leq\delta\epsilon$. An analogous bound holds for
$\gamma\beta\mathbb{E}[\|\alpha_tX_t^\top X_tw_t\|^2]$. Meanwhile, the descent
analysis contains the stochastic cross terms
$\gamma\mu_t\langle(Z_t-\nabla_{\eta}\fhL(\theta_t,\eta_t))w,Z_tw_t\rangle$
and
$\gamma\alpha_t\langle(X_t-\nabla_{\theta}\fhL(\theta_t,\eta_{t+1}))w,
X_tw_t\rangle$. After applying Cauchy--Schwarz and the clipping rules, these
terms are proportional to the stochastic estimation errors. Batch sizes
$N_1,N_2=\mathcal{O}(\epsilon^{-2})$ ensure that
$\mathbb{E}[\|(X_t-\nabla_{\theta}\fhL(\theta_t,\eta_{t+1}))w\|]$ and
$\mathbb{E}[\|(Z_t-\nabla_{\eta}\fhL(\theta_t,\eta_t))w\|]$ are both
$\mathcal{O}(\epsilon)$, keeping the corresponding contributions at the
$\mathcal{O}(\gamma\epsilon^2)$ level. Together, double clipping and large-batch
sampling remove the need to set $\beta,\gamma=\mathcal{O}(\epsilon^2)$ and
thereby improve the iteration complexity.
\item \textbf{Objective-independent clipping factor.}
Unlike single-objective methods~\citep{qirevistfDRO,revisitclip,generalizedsmooth}, whose clipping factors depend on problem-specific smoothness and variance constants, our method uses a universal constant factor of \(c_1,f_1=\tfrac{1}{2}\) for \((\theta_t,\eta_t)\) and \(w_t\), independent of these parameters. This is essential because $X_t$, $Z_t$, and
their Gram-matrix contributions generally have different smoothness and
variance scales. Using parameter-dependent factor would prevent a uniform $\mathcal{O}(\gamma\epsilon^2)$ bound on the corresponding error terms on the right-hand side of~\eqref{eq: general analysis framework}.
\item \textbf{Controlling potentially unbounded gradients.}
At the exact primal--dual pair
$(\theta,\eta^{i,*})$, we have $\|\nabla\phi^i(\theta)\|\leq G$. However, the gradients of
$\fhL(\theta_t,\eta_t)$ and $\fhL(\theta_t,\eta_{t+1})$ need not be uniformly
bounded at arbitrary dual iterates. Nevertheless,
$\|\nabla_{\theta}\fhL^i(\theta_t,\eta^i_{t+1})\|$ is coupled with
$\|\nabla_{\eta}\fhL^i(\theta_t,\eta^i_{t+1})\|$. We exploit this coupling through the
gradient--function-value relationship~\eqref{eq: boundgrad} and a stopping-time
argument. Specifically, we define $\htau$ to track abnormally large stochastic
gradient errors and function-value gaps and show that
$\mathbb{P}(\htau<T)\leq\tfrac{\delta}{2}$. We then rearrange the stopped descent inequality
\begin{align}
&\mathbb{E}[\fhL(\theta_{\htau},\eta_{\htau})w]-\fhL^*w\nonumber\\
&\quad\leq \frac{\bar F\delta}{8m}
-\frac{\gamma}{2}\mathbb{E}[\sum_{t=0}^{\htau-1}
\alpha_t\|X_tw_t\|^2]
-\frac{\gamma}{2}\mathbb{E}[\sum_{t=0}^{\htau-1}
\mu_t\|Z_tw_t\|^2],\nonumber
\end{align}
from Lemma~\ref{lemma: descent lemma of alg1-3} and combine it with the variance
bounds for $X_t^i$ and $Z_t^i$ and the one-step relationship between
$\nabla_{\eta^i}\fhL^i(\theta_t,\eta^i_{t+1})$ and
$\nabla_{\eta^i}\fhL^i(\theta_t,\eta^i_t)$ to obtain the desired result.
 
\end{itemize}
\begin{table*}[!t]
\caption{Test Accuracy (\%) under FGSM attack}
\label{Table: Multi-mnist FGSM}
\centering
{
\scriptsize
\renewcommand{\arraystretch}{1.08}
\setlength{\tabcolsep}{1.5pt}

\begin{tabular*}{\textwidth}{
    @{\extracolsep{\fill}}
    lccccc|ccccc
    @{}
}
\toprule
\multirow{2}{*}{Method / Perturbation budget}
& \multicolumn{5}{c|}{Multi-MNIST 2-digits ($70$-epochs training)}
& \multicolumn{5}{c}{Multi-MNIST 3-digits ($100$-epochs training)} \\
\cmidrule(lr){2-6}\cmidrule(lr){7-11}
& $0.00$ & $0.01$ & $0.03$ & $0.05$ & $0.08$
& $0.00$ & $0.01$ & $0.03$ & $0.05$ & $0.08$ \\
\midrule
Double-Clip MGDA
& $\boldsymbol{95.66}$ & $\boldsymbol{83.48}$ & $\boldsymbol{65.95}$ & $\boldsymbol{60.40}$ & $\boldsymbol{57.13}$
& $\boldsymbol{98.76}$ & $\boldsymbol{97.59}$ & $\boldsymbol{94.40}$ & $\boldsymbol{91.05}$ & $\boldsymbol{86.65}$ \\

Double-loop MGDA
& 92.80 & 72.81 & 57.71 & 54.49 & $51.63$
& 97.49 & 95.38 & 89.99 & 85.25 & 79.88 \\

MoCo~\citep{MOCO}
& 94.49 & 77.69 & 61.74 & 58.63 & 56.43
& 98.27 & 96.62 & 92.75 & 88.75 & 83.43 \\

NashMTL~\citep{nashMTL}
& 91.21 & 62.58 & 51.67 & 49.54 & 47.09
& 96.17 & 92.31 & 84.53 & 78.91 & 73.48 \\

FAMO~\citep{FAMO}
& 89.05 & 61.04 & 50.82 & 48.66 & 46.48
& 95.90 & 91.86 & 84.29 & 78.94 & 73.52 \\

SDMGrad~\citep{SDMkaiyiJi}
& 89.59 & 64.02 & 52.00 & 49.91 & 47.31
& 96.46 & 92.89 & 85.06 & 79.37 & 73.50 \\

MoDo~\citep{MODO}
& 91.10 & 64.08 & 51.95 & 49.73 & 47.49
& 96.50 & 93.15 & 86.03 & 80.56 & 74.81 \\

MGDA~\citep{MGDA1stwork,qimgda}
& 89.44 & 62.61 & 52.19 & 50.81 & 48.30
& 96.34 & 92.85 & 85.42 & 80.26 & 75.18 \\
\bottomrule
\end{tabular*}}
\end{table*}
\begin{table*}[!t]
\centering
\caption{Test Performance (\%) against Label Imbalance}
\label{table:accuracy-celeba}
{
\scriptsize
\renewcommand{\arraystretch}{1.08}
\setlength{\tabcolsep}{1.5pt}

\begin{tabular*}{\textwidth}{
    @{\extracolsep{\fill}}
    lccccc|ccccc
    @{}
}
\toprule
Methods/Evaluation metric (in \%) & Averaged Accuracy & Balanced Accuracy & AUC \\
\midrule
Double-Clip MGDA & $\boldsymbol{88.41}$ & $\boldsymbol{89.55}$ & $\boldsymbol{94.63}$ \\
Double-loop MGDA & 86.61 & 87.97 & 93.28 \\
MoCo~\citep{MOCO} & 87.27 & 88.50 & 93.31 \\
NashMTL~\citep{nashMTL} & 85.58 & 87.00 & 92.41 \\
FAMO~\citep{FAMO} & 85.32 & 86.03 & 91.41 \\
SDMGrad~\citep{SDMkaiyiJi} & 85.89 & 87.06 & 92.45 \\
MoDo~\citep{MODO} & 85.66 & 87.08 & 92.42 \\
MGDA~\citep{MGDA1stwork,qimgda} & 85.53 & 86.87 & 92.58 \\
\bottomrule
\end{tabular*}%
}
\vspace{-1em}
\end{table*}
\section{Experiments}
In this section, we evaluate Algorithms~\ref{alg1} and~\ref{alg1-3} for solving the dual formulation of DR-MOO in~\eqref{eq: dual DR-MOO} in deep learning settings. To illustrate the practical relevance of DR-MOO, we focus on three representative sources of distributional heterogeneity: adversarial perturbations on Multi-MNIST\citep{mnist}, label imbalance on CelebA~\citep{celebA}, and shifts across domains in Office-31 dataset~\citep{officedata31}. In all experiments, we use a ResNet-18 encoder~\citep{ResNet} with multiple MLP classification heads, set $\ell(\cdot)$ to the cross-entropy loss, and choose $f^*(\cdot)$ as the convex conjugate associated with the $\chi^2$-divergence. We compare our algorithms against MGDA~\citep{MGDA1stwork,stochasticMGDA}, MoDo~\citep{MODO}, MoCo~\citep{MOCO}, SDMGrad~\citep{SDMkaiyiJi}, NashMTL~\citep{nashMTL}, and FAMO~\citep{FAMO}. For all baselines, we treat \((\theta,\eta)\) as optimization variables, perform joint updates, and fine-tune all hyperparameters. Due to space limitations, we defer the Office-31, linear and logistic regression experiments and ablation studies to Appendices~\ref{Appendix: Synthetic Experiments} and~\ref{Appendix: Ablation study}.

\subsection{Robustness against Adversarial Attacks}
We conduct experiments on Multi-MNIST 2-digit and 3-digit datasets~\citep{mnist}, evaluating robustness of proposed formulation and algorithms under adversarial perturbations. For training configurations and hyperparameters, we refer to Tables~\ref{Table: backbone-training-config} and ~\ref{Table: Learning Rate Hyper-parameter} in Appendix~\ref{Appendix: Hyper-parameter of main experiments.} for more details. At test time, we load model parameters, compute task gradient summation and apply FGSM~\citep{FGSM} attack on test data. The corresponding test accuracies for all methods under different attack strengths are reported in Table~\ref{Table: Multi-mnist FGSM}. These results show that models trained using existing MGDA solvers are less robust under adversarial perturbations, whereas the proposed methods better preserve task performance under increasing attack strength.

\subsection{Robustness against Label Imbalance}
We further scale our methods to CelebA dataset to evaluate robust performance of proposed formulation and algorithms under label imbalance. We adopt the same neural network backbone as Multi-MNIST experiments. Detailed hyper-parameters are provided in Table~\ref{Table: Learning Rate Hyper-parameter CelebA}. To address label imbalance, we employ an importance-sampling strategy when approximating the dual of DR-MOO objective~\eqref{eq: dual DR-MOO}. At test time, we report averaged accuracy, balanced accuracy and averaged AUC across tasks, with results summarized in Table~\ref{table:accuracy-celeba}. These results indicate that the proposed algorithms improve model robustness under task-wise label imbalance.

\section{Conclusion}
In this paper, we introduced DR-MOO as a framework for modeling objective-wise distribution shifts in multi-objective optimization. We characterized distributionally robust Pareto solution concepts and proposed a dual-based definition of robust Pareto stationarity that enables convergence analysis. Building on this foundation, we developed two MGDA-type algorithms: a double-loop method based on dual tracking and a Double-Clip method that provably improves sample complexity through large-batch sampling and coordinated clipping in the parameter and preference updates. Our results establish convergence guarantees under distributional uncertainty and highlight the importance of controlling relative update magnitudes in stochastic MOO. We hope this work motivates further theoretical investigation of gradient normalization in existing MTL methods and inspires extensions of optimizers such as Adam~\citep{kingma2014adam} and Muon~\citep{jordan2024muon} to multi-objective settings.



\subsection{Acknowledgements}
Yufeng Yang, on behalf of all co-authors, thanks Mr. Kang An for early discussions of this project in the autumn of 2025. These discussions highlighted a gap in rigorous extensions of preconditioned, normalized, and clipped first-order optimizers to multi-objective settings, and directly inspired Double-Clip MGDA. For an up-to-date discussion of applying Double-Clip MGDA to general multi-task learning, see the first author's \href{https://ynyang94.github.io/blog/double-clip-mgda/}{blog post}.

\subsection*{AI Use Statement}
In this work, we used generative AI tools to assist with experimental design and to check the consistency and mathematical correctness of theoretical arguments. We did not use generative AI tools to conceive or develop the proposed algorithms, generate mathematical claims or proofs, implement the methods, generate synthetic data, process datasets, or interpret experimental results. Tasks involving surveys, interviews, transcription, and qualitative or thematic analysis were not applicable. Additionally, we used generative AI tools to improve the clarity, grammar, concision, and organization of author-written text.

All AI-assisted suggestions were critically reviewed by the authors. Theoretical feedback was independently checked against the stated assumptions and complete derivations, while experimental suggestions were evaluated, implemented, and validated by the authors. All algorithms, proofs, experiments, results, and conclusions were developed or verified by the authors. We take responsibility for the final content of this work, including all text, claims, and artifacts produced with the aid of generative AI.

\bibliographystyle{apalike}
\bibliography{main}

\clearpage
\appendix
\onecolumn

\part*{Appendix}
\addcontentsline{toc}{section}{Appendix}

\etocdepthtag.toc{mtappendix}
\etocsettagdepth{mtchapter}{none}      
\etocsettagdepth{mtappendix}{subsection} 
\etocsettocstyle
  {}
  {\vspace{1em}}
\tableofcontents
\clearpage

\allowdisplaybreaks
\section{Limitations and Future Work}
First, our high-probability guarantees for Double-Clip MGDA require batch sizes that depend strongly on the confidence parameter \(\delta\). This dependence arises from controlling stochastic-gradient errors and function-value gaps throughout the optimization trajectory and may partly reflect limitations of the current analytical framework~\eqref{eq: general analysis framework}. Future work could develop sharper concentration arguments through alternative Lyapunov analyses and explore algorithmic extensions, such as momentum or variance reduction, to reduce this dependence. Second, our methods seek Pareto-stationary solutions without explicitly enforcing user-specified trade-offs. An important extension is therefore to incorporate preference guidance, potentially through proximal or mirror-descent updates with preference-aware regularization.

\section{A Toy Example on Motivation of Study}\label{Appendix: Toy exmple visualization}
We visualize the effect of distribution shift on the objective geometry and Pareto frontier using a simple bi-objective toy problem. 
Given a parameter $\theta$, we evaluate two objectives
\[
f_1(\theta) = (\theta-x_1)^2+b_1,
\qquad
f_2(\theta) = (\theta-x_2)^2+b_2.
\]
We simulate the nominal and shifted data distributions by perturbing $x_1=0,x_2=2$ and $b_1,b_2=0$ with Gaussian noise. 
We approximate the Pareto frontier by retaining only the non-dominated sampled points. 
Figure~\ref{fig: Toy-example Pareto-visualization} visualizes the objective geometry (Left) and Pareto frontier (Right) before and after perturbation. 
As shown, the perturbation changes not only the objective geometry but also the induced Pareto frontier. 
In particular, we found the frontier is substantially distorted, and some parameter values that are Pareto-optimal under the nominal distribution become non-optimal after the perturbation.
\begin{figure}[ht]
    \centering
    \begin{subfigure}{0.23\linewidth}
        \includegraphics[width=\linewidth]{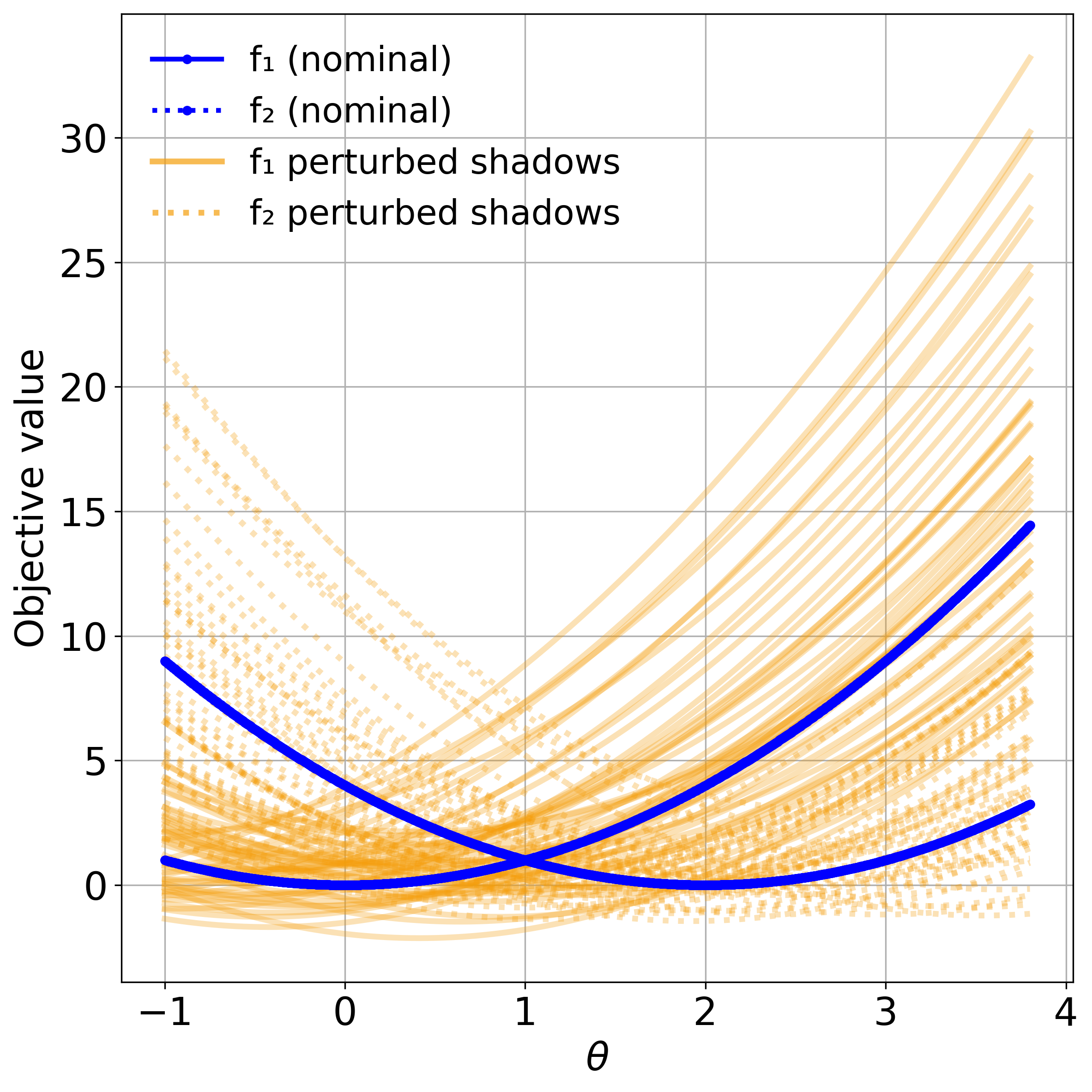}
    \end{subfigure}
    \hspace{0.02\linewidth}
    \begin{subfigure}{0.23\linewidth}
        \includegraphics[width=\linewidth]{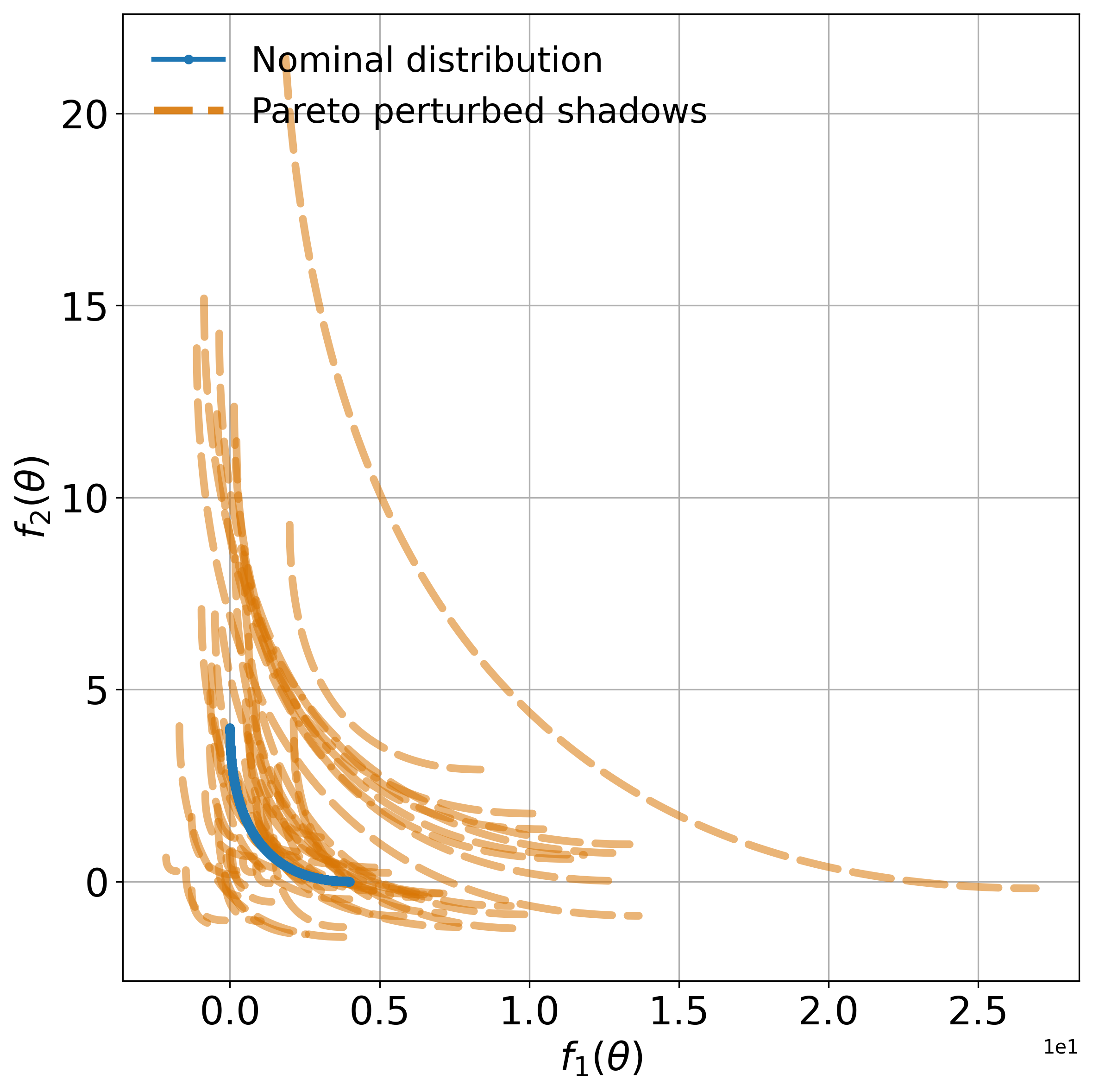}  
    \end{subfigure}
    \caption{Effects of data distribution shift on function geometry (Left) and Pareto-frontier (Right).}
    \label{fig: Toy-example Pareto-visualization}
\end{figure}
\section{Summary of Comparison with Closely Related Works}\label{Appendix: comparison with other algorithms}
Table~\ref{Table: Comparison with existing methods} compares our algorithms' theoretical results with those of other related algorithms. All complexity results are compared in terms of convergence to $\epsilon$-Pareto stationary point.
\begin{table}[ht!]
\centering
\small
\caption{Comparison of convergence results on variants of gradient-balancing methods. (\textbf{Note: For the proposed algorithms, the smoothness and sample noise are defined with respect to the dual of DR-MOO formulation~\eqref{eq: dual DR-MOO}}. The corresponding properties are provided in Lemma~\ref{lemma: smooth of phi}, Corollary~\ref{Corollary: estimation error of true gradient}, Lemma~\ref{lemma: generalized-smooth}, and Lemma~\ref{lemma: affine bounded noise}, respectively.) Explanation on the upper footmarks: $1:$ Bounded Noise indicates the stochastic gradient variance is bounded by a constant. Affine Bounded Noise indicates stochastic gradient variance is affine bounded;
$\times$ indicates ``Does Not Apply" since NashMTL is analyzed under access to full-gradient; Bounded Output Noise indicates stochastic oracles output a bounded noisy gradient.
$2:$ (semi-LS) indicates function geometry is $L$-smooth along certain trajectory;
(LS) indicates $L$-smooth assumption on objectives while (GS) indicates objectives satisfy $(L_0,L_1)$ generalized-smooth~\citep{generalizedsmooth,Haochuangensmooth}; $3:$ (BF) denotes the bounded function value assumption and (BG) denotes the bounded gradient assumption; $4:$ Sample complexity is measured in terms of achieving $\epsilon$-accurate Pareto stationary point under stochastic settings. N/A denotes no additional assumptions.
}
\resizebox{\columnwidth}{!}{
\begin{tabular}{lccccc}
\toprule
Methods & Sample Noise\textsuperscript{1} & Smoothness\textsuperscript{2} & Other
Assumption\textsuperscript{3} & Batch Size & Sample Complexity\textsuperscript{4}  \\
\midrule
Double-Loop MGDA & Bounded Noise & semi-LS & N/A & $O(1)$ & $O(\epsilon^{-8})$ \\
Double-Clip MGDA & Affine Bounded Noise& GS & N/A & ${O}(\epsilon^{-2})$ & ${O}(\epsilon^{-4})$ \\
\midrule
MoCo~\citep{MOCO} & Bounded Noise&LS & BG and BF & ${O}(1)$ & $O(\epsilon^{-4})$ \\
NashMTL~\citep{nashMTL} & $\times$ & LS & N/A & N/A & asymptotic convergence\\
SDMGrad~\citep{SDMkaiyiJi} & Bounded Noise & LS & BG &$O(1)$& $O(\epsilon^{-4})$.  \\
MoDo~\citep{MODO} & Bounded Output Noise & LS & BG & $O(1)$&  $O(\epsilon^{-4})$  \\
MGDA~\citep{qimgda} & Bounded Noise & GS & N/A & ${O}(1)$ & ${O}(\epsilon^{-4})$ \\
\bottomrule
\end{tabular}}
\label{Table: Comparison with existing methods}
\end{table}

\textbf{Robust multi-task learning.}
A line of work in robust MTL addresses noisy or outlier tasks through statistical modeling, structured parameterization, or decomposition. Early approaches~\citep{robustmtl1} use Gaussian-process variants to limit the influence of outlier tasks, while \citet{robustmtl2} decompose the task-parameter matrix into shared and task-specific components to identify such tasks explicitly. More recently, \citet{robustMTL3} develop adaptive frameworks with statistical guarantees that exploit task similarity while allowing task-specific deviations. Other studies consider robustness to label noise~\citep{robustMTL4} and adversarial perturbations~\citep{robustMTL5}. Despite these advances, the machine-learning literature on robust MTL has largely focused on robustness through task modeling, noise mitigation, or adversarial training. Task losses are typically aggregated or adaptively reweighted into a scalar objective, rather than modeled using objective-wise ambiguity sets over data distributions. Consequently, these studies do not formulate robust MTL as a vector-valued DRO problem or study distributionally robust Pareto solutions. In contrast, DR-MOO preserves the vector-valued robust objective and establishes convergence to distributionally robust Pareto-stationary solutions.

\textbf{On first-order algorithmic solutions for solving DRO}
Several works study single-objective DRO formulations amenable to first-order optimization. For the $f$-divergence--regularized risk $\min_{\theta}\phi(\theta)$, where $\phi(\theta):=\min_{\eta\in\mathbb{R}}\{\lambda\mathbb{E}_{\xi\sim\mathbb{P}}[f^*(\tfrac{\ell(\theta;\xi)-\eta}{\lambda})]+\eta\}$, \citet{jinnonconvexdro} develop normalized SGD with momentum and establish an $\mathcal{O}(\epsilon^{-4})$ gradient complexity. Under assumptions comparable to ours, \citet{qirevistfDRO} propose D-SGD-C and its variance-reduced counterpart D-Spider-C, achieving gradient complexities of $\mathcal{O}(\epsilon^{-4})$ and $\mathcal{O}(\epsilon^{-3})$, respectively. These results demonstrate that normalization and clipping effectively exploit the generalized-smooth geometry and affine stochastic noise of dual DRO objectives. However, they address only single-objective parameter updates and do not account for the gradient balancing and stochastic preference updates required in MOO. This motivates us to adopt clipping as a building block and extend it jointly to the optimization and preference updates in DR-MOO. Earlier, \citet{large-scaleDRO} developed stochastic methods for convex CVaR and $\chi^2$-DRO using biased batch-gradient and multilevel Monte Carlo estimators under stronger boundedness and regularity assumptions.

Beyond regularized $f$-divergence DRO, \citet{qiaaai} study constrained Cressie--Read DRO by combining stochastic-gradient and Frank--Wolfe updates, while \citet{qiKLDRO} develop constant-batch stochastic algorithms for constrained KL-DRO with complexity independent of the dataset size. For Sinkhorn DRO, \citet{sinkhornDROwang} employ stochastic mirror descent with biased gradient estimators and obtain near-optimal sample complexity under convexity and boundedness assumptions. Other work studies empirical Group-DRO: \citet{droacceleration} develop ball-oracle acceleration with variance-reduced inner solvers, whereas \citet{yu2024efficient} propose the ALEG stochastic primal--dual method by exploiting its two-level finite-sum structure. These methods optimize scalar robust risks under problem-specific assumptions but do not address vector-valued robust objectives or adaptive preference balancing in DR-MOO.

\section{Hyper-parameter Configuration of Main Experiments}\label{Appendix: Hyper-parameter of main experiments.}
We conduct our deep learning experiments using two primary configurations: an Intel-based workstation featuring a single NVIDIA RTX 4090 GPU and 32 GB RAM, and a server equipped with four NVIDIA RTX A6000 GPUs.

\subsection{Neural Network Backbone, Common hyper-parameters and Implementation Details}
Table~\ref{Table: backbone-training-config} summarizes the network backbone and common training settings for multi-task classification on Multi-MNIST, CelebA, \footnote{
The UCI Wine Quality dataset is licensed under CC BY 4.0. CelebA is available for non-commercial research use, and its agreement prohibits redistribution of the dataset.} and Office-31. For the ResNet18 encoder, we freeze the initial convolutional and batch normalization layers, together with the first residual stage, throughout training.

For CelebA, we select eight attributes: \textbf{Bald, Wearing Hat, Eyeglasses, Receding Hairline, Narrow Eyes, Blond Hair, Bags Under Eyes, and Big Nose}. For each attribute, we estimate the empirical label proportions within each batch, denoted by $\tilde{p}$ and $1-\tilde{p}$. When evaluating the dual DR-MOO objective~\eqref{eq: dual DR-MOO}, we assign samples from the two classes weights of $\frac{1}{2\tilde{p}}$ and $\frac{1}{2(1-\tilde{p})}$, respectively. This batchwise class reweighting approximates a nominal distribution $\mathbb{P}$ with balanced class probabilities, i.e., $p=\frac{1}{2}$.

For computational efficiency, our implementation of Double-Loop MGDA samples three iterates, $\eta_d$, $\eta_{\bar d}$, and $\eta_{\tilde d}$, from a single inner-loop SGD trajectory rather than from independent runs. We use these iterates to construct $Y_t$, $\bar Y_t$, and $\tilde Y_t$ for the model and preference updates, avoiding additional inner-solver runs.

\begin{table}[ht]
\centering
\caption{Neural network backbone and Common hyper-parameter Settings}
\label{Table: backbone-training-config}
{
\scriptsize
\renewcommand{\arraystretch}{1.08}
\setlength{\tabcolsep}{1.5pt}
\begin{tabular*}{\textwidth}{
    @{\extracolsep{\fill}}
    ll
    @{}
}
\toprule
\textbf{Component} & \textbf{Configuration} \\
\midrule
Shared encoder & ResNet18 pre-trained on ImageNet~\citep{imageNet} \\
Classification heads & MLP:  Two-layer linear head with intermediate
ReLU activation and dropout \\
Objective & Dual of DR-MOO~\eqref{eq: dual DR-MOO} for all methods  with $\ell(\cdot)$ to be cross-entropy loss\\
$f^*(\cdot)$ & Convex conjugate dual of $\chi^2$ divergence, i.e., $f^*(t)=\frac{1}{4}(t+2)_{+}^2-1$ \\
Head initialization & Kaiming uniform initialization for each layer \\
Optimizer for heads & SGD \\
Batch size & $256$ \\
Learning rate for heads & $1e^{-2}$ for all methods \\
$\rho$ & $1e^{-5}$ for multi-mnist; $5e^{-5}$ for CelebA\\
$\lambda$ in DR-MOO~\eqref{eq: dual DR-MOO} & $0.6$ for multi-mnist and CelebA\\
Inner iterations & $5$ for Double-Loop MGDA and SDMGrad; \\
\midrule
2-digit Multi-MNIST input size & $128 \times 128$ \\
2-digit Multi-MNIST training epochs & $70$ for all methods \\
3-digit Multi-MNIST input size & $96 \times 96$ \\
3-digit Multi-MNIST training epochs & $100$ for all methods \\
CelebA input size & $128\times128$ \\
CelebA training epochs & $100$ for all methods.\\
\bottomrule
\end{tabular*}}
\end{table}

\begin{table}[ht]
\centering
\caption{Hyperparameter settings for Multi-MNIST experiments.}
\label{Table: Learning Rate Hyper-parameter}
{
\scriptsize
\renewcommand{\arraystretch}{1.08}
\setlength{\tabcolsep}{1.5pt}
\begin{tabular*}{\textwidth}{
    @{\extracolsep{\fill}}
    lcccccccc
    @{}
}
\toprule
\textbf{Method} 
& \textbf{LR for $(\theta,\eta)$} 
& \textbf{LR for preference ($\beta$)} 
& $\boldsymbol{f_1,f_2}$ 
& $\boldsymbol{c_1,c_2}$ \\
\midrule
Double-Loop MGDA 
& $\alpha=5e^{-4},\gamma=3e^{-3}$ 
& $1e^{-5}$ 

& -- 
& -- \\

 Double-Clip MGDA 
& $\gamma=5e^{-3}$
& $1e^{-2}$ 
& $1.0$,\,$0.5$ 
& $1.0,\,0.5$ \\

MGDA~\citep{qimgda,MGDA1stwork} 
& $1e^{-4}$ 
& $1e^{-6}$ 
& -- 
& -- \\

MoDo~\citep{MODO} 
& $1e^{-4}$ 
& $1e^{-6}$ 
& -- 

& -- \\

MoCo~\citep{MOCO}
& $5e^{-4}$ 
& $1e^{-4}$ 
& -- 
& -- \\

SDMGrad~\citep{SDMkaiyiJi}
& $1e^{-4}$ 
& $1e^{-6}$ 
& -- 
& -- \\

FAMO~\citep{FAMO}
& $1e^{-4}$ 
& $2.5e^{-3}$ 
& -- 
& -- \\

NashMTL~\citep{nashMTL}
& $1e^{-4}$ 
& -- 
& -- 
& -- \\
\bottomrule
\end{tabular*}}
\end{table}

\begin{table}[ht]
\centering
\small
\caption{Hyperparameter settings for CelebA Experiments.}
\label{Table: Learning Rate Hyper-parameter CelebA}
{
\scriptsize
\renewcommand{\arraystretch}{1.08}
\setlength{\tabcolsep}{1.5pt}
\begin{tabular*}{\textwidth}{
    @{\extracolsep{\fill}}
    lcccccccc
    @{}
}
\toprule
\textbf{Method} 
& \textbf{LR for $(\theta,\eta)$} 
& \textbf{LR for preference ($\beta$)} 
& $\boldsymbol{f_1,f_2}$ 
& $\boldsymbol{c_1,c_2}$ \\
\midrule
Double-Loop MGDA 
& $\alpha=5e^{-4},\gamma=3e^{-4}$ 
& $5e^{-5}$ 

& -- 
& -- \\

Double-Clip MGDA 
& $\gamma=5e^{-3}$
& $2e^{-2}$ 
& $1.0$,\,$0.5$ 
& $1.0,\,0.5$ \\

MGDA~\citep{qimgda,MGDA1stwork} 
& $1e^{-4}$ 
& $5e^{-5}$ 
& -- 
& -- \\

MoDo~\citep{MODO} 
& $1e^{-4}$ 
& $5e^{-5}$ 
& -- 

& -- \\

MoCo~\citep{MOCO}
& $5e^{-4}$ 
& $5e^{-5}$ 
& -- 
& -- \\

SDMGrad~\citep{SDMkaiyiJi}
& $1e^{-4}$ 
& $1e^{-6}$ 
& -- 
& -- \\

FAMO~\citep{FAMO}
& $1e^{-4}$ 
& $5e^{-6}$ 
& -- 
& -- \\

NashMTL~\citep{nashMTL}
& $1e^{-4}$ 
& -- 
& -- 
& -- \\
\bottomrule
\end{tabular*}}
\end{table}

\begin{table}[t]
\centering
\caption{Hyperparameter settings for the Office-31 experiments.}
\label{tab:office31_hyperparameters}
{
\scriptsize
\renewcommand{\arraystretch}{1.08}
\setlength{\tabcolsep}{1.5pt}
\begin{tabular*}{\textwidth}{
    @{\extracolsep{\fill}}
    lcc
    @{}
}
\toprule
Method & LR for $(\theta,\eta)$ & LR for preference $(w)$ \\
\midrule
Double-Loop MGDA
& $\alpha=5e^{-4},\ \gamma=1e^{-3}$
& $5e^{-4}$ \\

Double-Clip MGDA
& $\gamma=1e^{-3}$
& $2e^{-3}$ \\

MGDA~\citep{MGDA1stwork}
& $1e^{-4}$
& $5e^{-5}$ \\

MoDo~\citep{MODO}
& $1e^{-4}$
& $5e^{-5}$ \\

MoCo~\citep{MOCO}
& $5e^{-4}$
& $5e^{-5}$ \\

SDMGrad~\citep{SDMkaiyiJi}
& $1e^{-4}$
& $5e^{-6}$ \\

FAMO~\citep{FAMO}
& $5e^{-5}$
& $1e^{-2}$ \\

NashMTL~\citep{nashMTL}
& $1e^{-4}$
& -- \\
\bottomrule
\end{tabular*}}
\end{table}

\subsection{Learning Rate Choices}
Table~\ref{Table: Learning Rate Hyper-parameter} and Table~\ref{Table: Learning Rate Hyper-parameter CelebA} summarize the learning rate choices for all methods, which are used when training the network backbone on the Multi-MNIST and CelebA datasets. For method-specific hyperparameters, we set $\gamma = 1e^{-2}$ for FAMO~\citep{FAMO}, and the exponential moving average parameter $\text{ema}=0.95$ for MoCo~\citep{MOCO}, we leverage mirror-descent solver to solve nonlinear system introduced in NashMTL~\citep{nashMTL} and run $100$ iterations each time step.

\section{Additional Experiments}\label{Appendix: Synthetic Experiments}
\subsection{Multi-Task Linear Regression over Synthetic Data}
We evaluate Algorithm~\ref{alg1} and \ref{alg1-3} on the dual of DR-MOO problem~\eqref{eq: dual DR-MOO} with the mean-squared error loss. We consider a synthetic linear regression setup with $m=3$ objectives, data dimension $n=10$, and a total of $6000$ samples. Inputs are drawn i.i.d. from $X \sim \mathcal{N}(0, I)$.
To induce partial conflicts across objectives, we generate ground-truth parameters as $
\theta^{1,*} \sim \mathcal{N}(0, I), \quad
\theta^{2,*} \sim \mathcal{N}(-0.2\theta_1^*, 0.04 I), \quad
\theta^{3,*} \sim \mathcal{N}( 0.5\theta_1^*, 0.25 I)
$. We then generate labels according to $y^{i}=X\theta^{*,i}+\varepsilon^i$, where the noise terms have different variances $\varepsilon^1 \sim \mathcal{N}(0, 0.04)$, $\varepsilon^2 \sim \mathcal{N}(0, 0.36)$, $
\varepsilon^3 \sim \mathcal{N}(0, 0.25)$. 

We compare all methods for $T={600}$ iterations with tuned hyperparameters. For all methods, we use a batch size of $256$ to ensure numerical stability.
For the double-loop MGDA (Algorithm~\ref{alg1}), we set the inner-loop length $D=20$, learning rates $\gamma=5e^{-3}$, $\alpha=\beta=5e^{-5}$, and regularization parameter $\rho=1e^{-5}$.
For the single-loop double-clip MGDA (Algorithm~\ref{alg1-3}), we set learning rate $\gamma=3e^{-2}$, $\beta = 1e^{-2}$, clipping thresholds $c_1=f_1=0.5$, $c_2=f_2=0.1$ and $\rho=1e^{-5}$. For stochastic MGDA, we set learning rate for both $(\theta_t,\eta_t)$ and $w_t$ to $1e^{-5}$, with $\rho=0$. For MoDo and MoCo, we adopt the same settings as stochastic MGDA, but set $\rho=1e^{-5}$. For SDMGrad, we set $D=10$, learning rates $1e^{-4}$ for $(\theta_t,\eta_t)$ and $5e^{-4}$ for $w_t$, and $\rho=1e^{-5}$.
Figure~\ref{fig: synthetic experiments} (Left) reports the balanced stochastic gradient norm versus iteration, together with total sample consumption. 
Both our proposed methods achieve competitive performance relative to all baselines, highlighting the benefit of explicitly exploiting the structure of DR-MOO. While SDMGrad attains a similar iteration-wise convergence trend to Double-Loop MGDA, it requires substantially more samples and relies on additional heuristic normalization for numerical stability.
\begin{figure}[ht]
    \centering
    \begin{subfigure}{0.23\textwidth}
        \includegraphics[width=\linewidth]{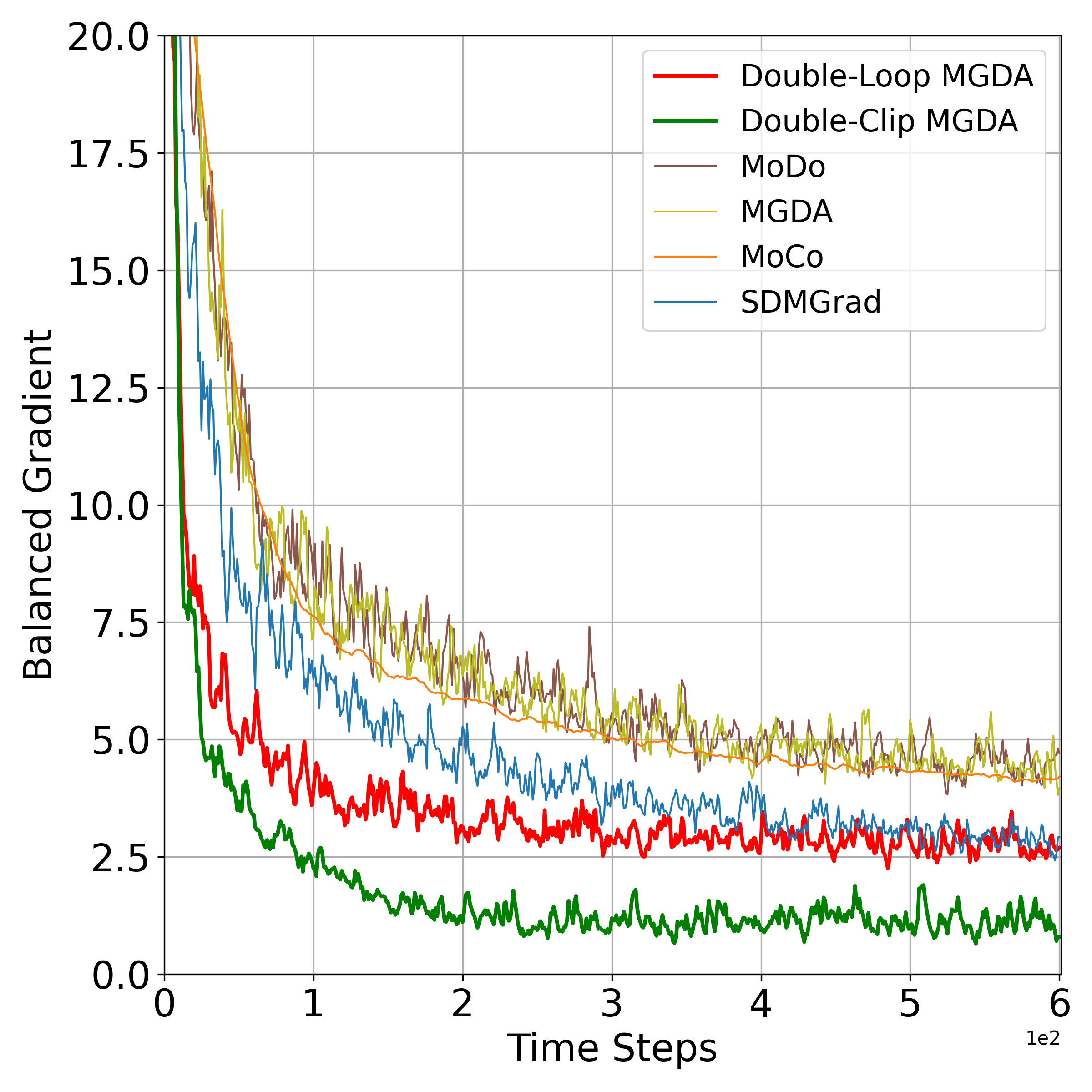}
    \end{subfigure}
    \hspace{0.01\linewidth}
    \begin{subfigure}{0.23\textwidth}
        \includegraphics[width=\linewidth]{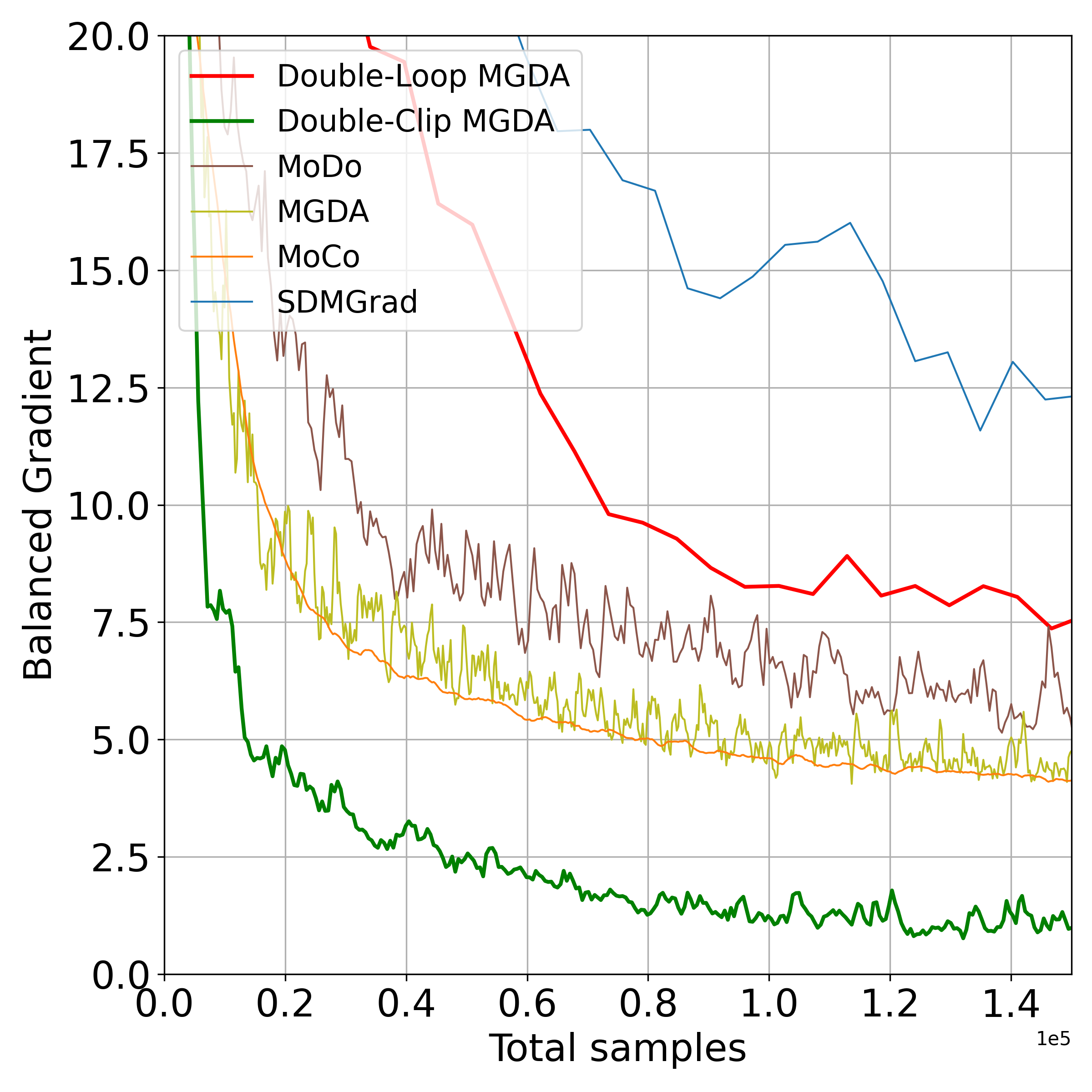}
    \end{subfigure}  
    \hspace{0.01\linewidth}
      \begin{subfigure}{0.23\textwidth}
        \includegraphics[width=\linewidth]{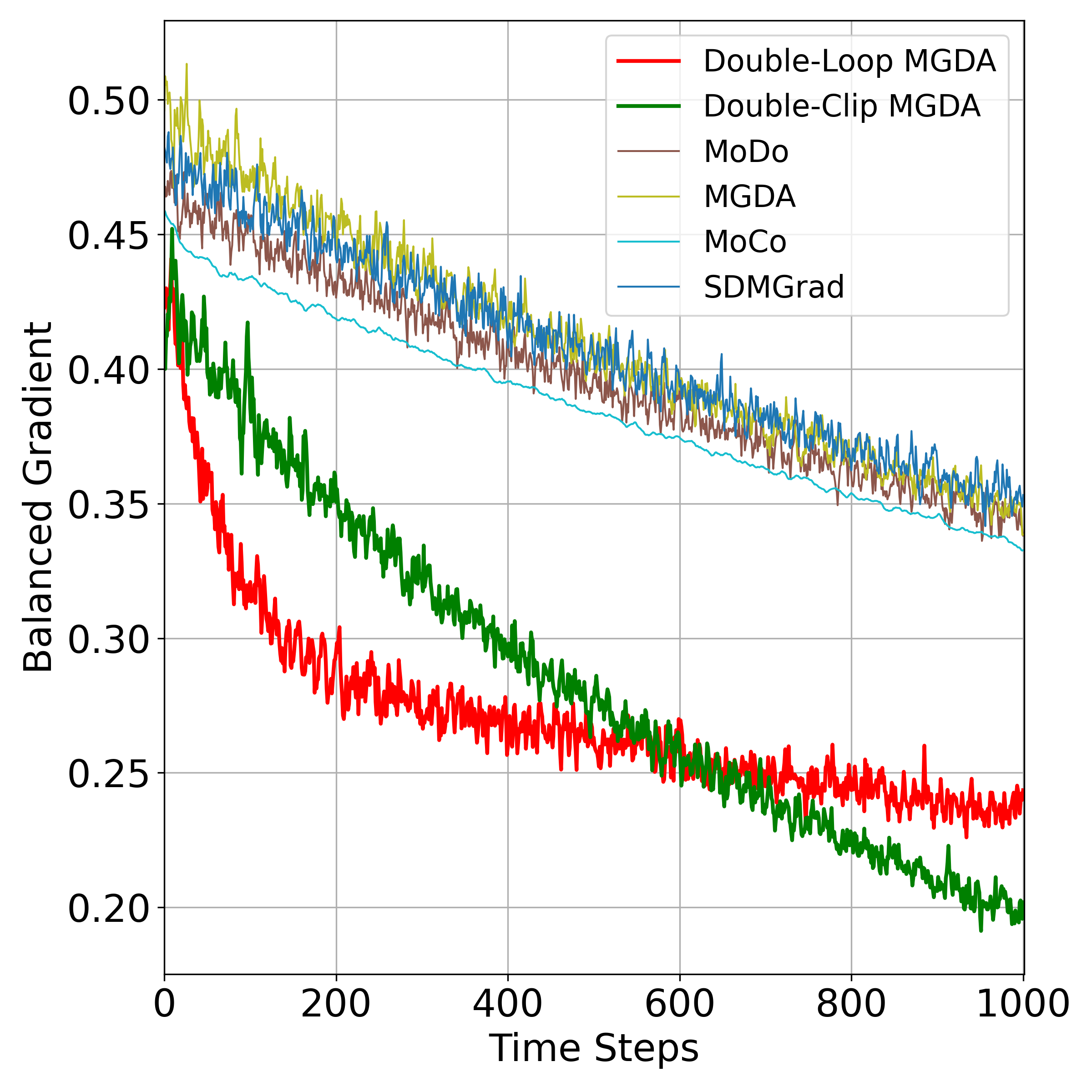}
    \end{subfigure}
    \hspace{0.01\linewidth}
    \begin{subfigure}{0.23\textwidth}
        \includegraphics[width=\linewidth]{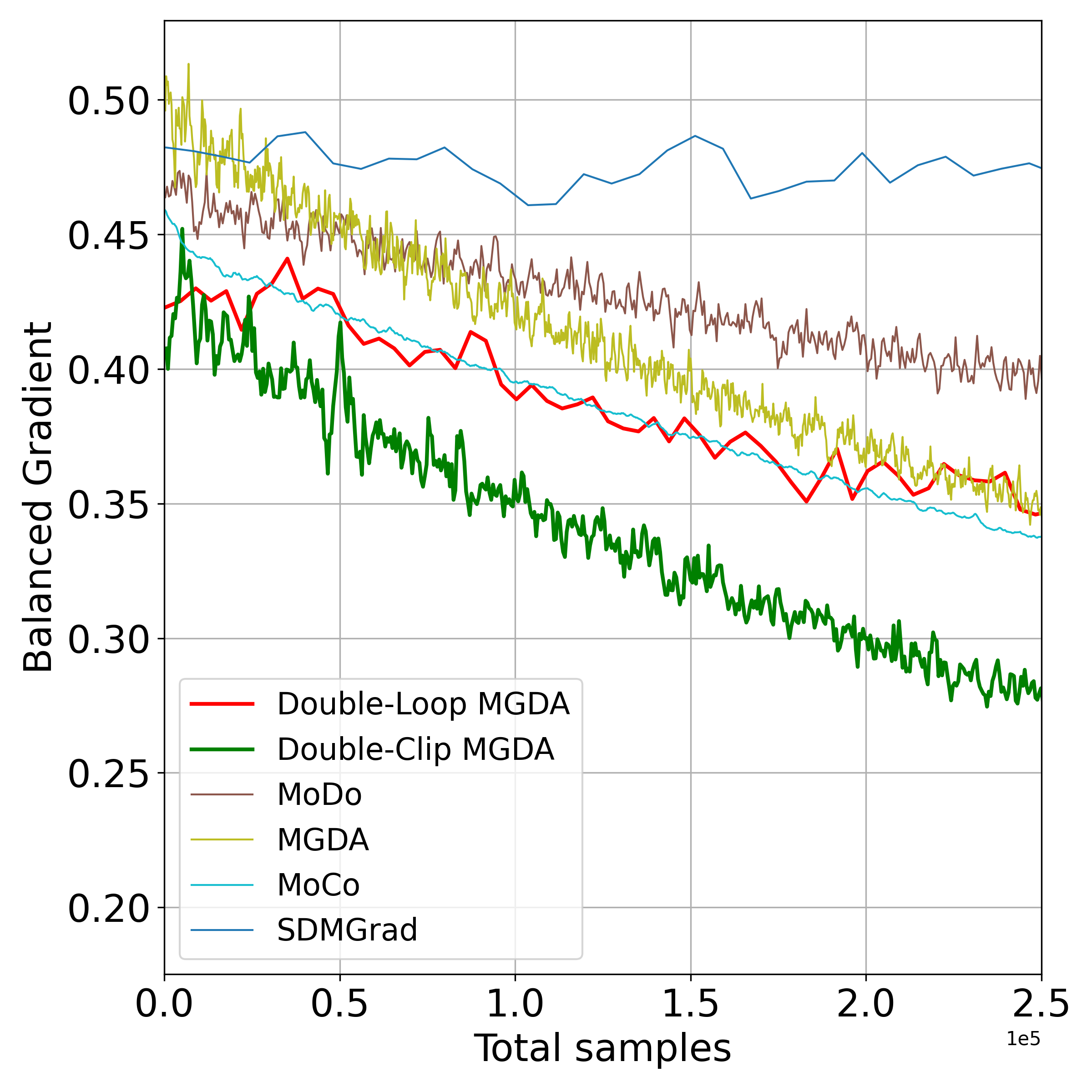}
    \end{subfigure}
    \caption{Balanced Gradient vs.Time Steps and Sample Consumption for Linear (Left) and Logistic Regression (Right)}
    \label{fig: synthetic experiments}
\end{figure}
\begin{table}[t]
    \centering
    
    \setlength{\tabcolsep}{5pt}
    \caption{Balanced-gradient norm over 100 random seeds. For each seed, the
    metric is averaged over the final 20 iterations. Lower values are better.
    The intervals are two-sided 90\% Student-$t$ confidence intervals.}
    \label{tab:regression-gradient-results}
    \resizebox{0.75\columnwidth}{!}{
    \begin{tabular}{llcc}
        \toprule
        \textbf{Experiment} & \textbf{Method}
        & \textbf{Mean final gradient} & \textbf{90\% CI} \\
        \midrule
        Synthetic linear regression
        & Double-Loop MGDA
        & 2.530248
        & $[2.514342,\;2.546153]$ \\

        Synthetic linear regression
        & Double-Clip MGDA
        & \textbf{1.224745}
        & $[1.205730,\;1.243761]$ \\

        UCI white-wine logistic regression
        & Double-Loop MGDA
        & 0.239733
        & $[0.237789,\;0.241677]$ \\

        UCI white-wine logistic regression
        & Double-Clip MGDA
        & \textbf{0.175293}
        & $[0.174285,\;0.176301]$ \\
        \bottomrule
    \end{tabular}}
\end{table}
\subsection{Multi-Task Logistic Regression over UCI White Wine Dataset}
We evaluate Algorithm~\ref{alg1} and \ref{alg1-3} on the dual of DR-MOO problem in~\eqref{eq: dual DR-MOO} with the cross-entropy loss. Specifically, we train $m=3$ logistic regression models on the UCI white wine dataset~\citep{uci-dataset}, corresponding to binary classification tasks based on wine quality, alcohol content, and residual sugars. To induce class imbalance, we generate binary labels using task-specific quantile thresholds $s^i$, i.e., $y^i=\mathbf{1}(y\geq s^i),\forall i\in \{1,2,3\}$. We set thresholds for quality, residual sugars and alcohol as $0.5$, $0.8$, $0.1$ respectively. 

We compare all methods for $T={1000}$ iterations. For all methods, we use a batch size of $256$. For the double-loop MGDA (Algorithm~\ref{alg1}), we set the inner-loop length $D=15$, learning rates $\gamma=5e^{-3}$, $\alpha=1.5e^{-3}$, $\beta = 6e^{-4}$, and $\rho=1e^{-6}$. For the single-loop double-clip MGDA (Algorithm~\ref{alg1-3}), we set learning rate $\gamma = 1e^{-2}$, clipping thresholds $c_1=f_1=0.5$, $c_2=f_2=0.1$ and $\rho=1e^{-5}$. For MGDA, MoDo and MoCo, we set learning rates for parameter $(\theta,\eta)$ as $1e^{-3}$, and learning rate for preference vector $w_t$ as $6e^{-4}$. For regularization parameter, we set $\rho=1e^{-6}$ for MoCo and MoDo, $\rho=0$ for MGDA, respectively. For SDMGrad, we set $D=15$, learning rates for parameter $(\theta,\eta)$ as $1e^{-3}$, for preference vector $w_t$ as $5e^{-4}$, and regularization parameter $\rho=1e^{-4}$.  Figure~\ref{fig: synthetic experiments} reports the comparison results, from which one can see that our proposed methods
achieve competitive performance relative to the baselines.
\subsection{Multi-Seed Experiments over Linear and Logistic Regression}
We repeat the linear and logistic regression experiments over 100 random seeds and report the balanced-gradient norm averaged over the final 20 iterations. Table~\ref{tab:regression-gradient-results} presents the mean and 90\% confidence interval across seeds. Double-Clip MGDA consistently achieves a lower final gradient norm than Double-Loop MGDA on both tasks, while the narrow confidence intervals indicate stable performance across different random initializations and stochastic samples.
\subsection{Multi-Domain Classification on Office-31}\label{Appendix: office-31 dataset experiment}
We further evaluate our methods on Office-31~\citep{officedata31}, which contains $31$ object classes from the Amazon, DSLR, and Webcam domains.  We treat each domain as a separate task and objective, reuse the neural network backbone from the preceding experiments and common hyper-parameters. Table~\ref{tab:office31_hyperparameters} presents fine-tuned learning rates of each method. To avoid repeatedly sampling the smaller domains solely to match the size of Amazon, we independently shuffle each domain and partition every domain loader into $31$ minibatches per pass. This results in batch sizes of $63$--$64$ for Amazon, $10$--$11$ for DSLR, and $17$--$18$ for Webcam. Thus, each training image is used once during a complete pass through its domain loader. At each training step, we concatenate one mini-batch from each domain for the forward pass and mask labels that do not belong to the domain whose objective is being computed.

For each domain, we optimize the same dual DR-MOO objective~\eqref{eq: dual DR-MOO}. We train each method for $70$ epochs. Table~\ref{table:accuracy-office31} reports classification accuracy on test set. Under this setting, Double-Clip MGDA achieves the highest observed worst-domain and mean-domain test accuracy. Double-Loop MGDA ranks behind MoCo but outperforms the remaining baselines on the worst-domain metric. These results demonstrate that the proposed algorithms are competitive when applied to the common DR-MOO objective on this multi-domain classification task.
\begin{table}[ht]
\centering
\small
\caption{Office-31 test accuracy (\%). ``Worst'' is the minimum accuracy across the three domains, and ``Mean'' weights the domains equally.}
\label{table:accuracy-office31}

\scriptsize
\renewcommand{\arraystretch}{1.08}
\setlength{\tabcolsep}{1.5pt}
{
\begin{tabular*}{\textwidth}{
    @{\extracolsep{\fill}}
    lccccc
    @{}
}
\toprule
\textbf{Method} & \textbf{Amazon} & \textbf{DSLR} & \textbf{Webcam} & \textbf{Worst} & \textbf{Mean} \\
\midrule
Double-Clip MGDA & $\boldsymbol{79.06}$ & $\boldsymbol{98.06}$ & $\boldsymbol{98.61}$ & $\boldsymbol{79.06}$ & $\boldsymbol{91.91}$ \\
Double-Loop MGDA & 78.62 & $\boldsymbol{98.06}$ & 97.22 & 78.62 & 91.30 \\
MoCo~\citep{MOCO} & 78.84 & $\boldsymbol{98.06}$ & 97.92 & 78.84 & 91.61 \\
SDMGrad~\citep{SDMkaiyiJi} & 78.17 & 97.09 & 97.92 & 78.17 & 91.06 \\
MoDo~\citep{MODO} & 78.17 & 97.09 & 97.92 & 78.17 & 91.06 \\
MGDA~\citep{MGDA1stwork,qimgda} & 78.17 & 97.09 & 97.92 & 78.17 & 91.06 \\
NashMTL~\citep{nashMTL} & 77.95 & 97.09 & 97.92 & 77.95 & 90.99 \\
FAMO~\citep{FAMO} & 75.72 & 96.12 & 97.22 & 75.72 & 89.69 \\
\bottomrule
\end{tabular*}}
\end{table}

\section{Ablation Studies}\label{Appendix: Ablation study}
We examine sensitivity to the regularization parameter $\lambda$, learning rates, batch size, and the number of inner iterations, and compare the per-iteration computational cost of our methods with the baselines.

\subsection{Ablation on $\lambda$}
According to classical results on regularization hyperparameters~\citep{PRML}, from the perspective of the primal DR-MOO formulation~\eqref{eq: primal DR-MOO}, a smaller $\lambda$ permits larger deviations in the shifted distribution $\mathbb{Q}$, while a larger $\lambda$ enforces tighter control. However, our theoretical analysis reveals that several problem-dependent quantities scale with $\lambda^{-1}$. In particular, the smoothness parameters satisfy $\hL_0, \hL_2 = \mathcal{O}(\lambda^{-1})$, and the stochastic gradient variance satisfies $\hat{K}_0,\hat{K}_2 = \mathcal{O}(\lambda^{-2})$, both of which critically influence convergence behavior (See Appendix~\ref{Appendix: Relevant Properties of Phi} and~\ref{sec: property of hat L} for more details).

To study the sensitivity to $\lambda$, we conduct an ablation experiment. We fix the optimizer to the proposed single-loop Double-Clip MGDA, adopt the same hyperparameters as in the Multi-MNIST experiments (see Table~\ref{Table: backbone-training-config} and~\ref{Table: Learning Rate Hyper-parameter}), and vary $\lambda \in \{0.4, 0.6, 1.0, 1.6\}$. At test time, we evaluate robustness using FGSM attacks, and report classification accuracy under different attack levels in Table~\ref{Table: Lambda Ablation}. The results show that $\lambda=0.4$ yields slightly worse robustness compared to $\lambda=0.6,1.0$ and $1.6$, suggesting that small $\lambda$ may degrade optimization performance.
\begin{table}[t]
\caption{Ablation Comparison (Accuracy in \%) on $\lambda$}
\label{Table: Lambda Ablation}
\centering
{
\scriptsize
\renewcommand{\arraystretch}{1.08}
\setlength{\tabcolsep}{1.5pt}
\begin{tabular*}{\textwidth}{
    @{\extracolsep{\fill}}
    lccccc
    @{}
}
\toprule
Method/Attack Level & Clean & $0.010$ & $0.030$ & $0.050$ & $0.080$ \\
\midrule
$\lambda=0.4$ & 95.54 & 80.16 & 63.13 & 58.94 & 55.78 \\
$\lambda=0.6$ & 95.67 & 83.50 & 65.97 & 60.39 & 57.12 \\
$\lambda=1.0$ & 95.99 & 83.56 & $66.17$ & $61.33$ & $57.97$ \\
$\lambda=1.6$ & ${96.01}$ & ${83.57}$ & 65.96 & 60.77 & 57.38 \\
\bottomrule
\end{tabular*}}
\end{table}

\subsection{Ablation on Learning Rates}
During our experiments, we observed that the balanced gradient norm plays a critical role in training stability. Its magnitude also induces different feasible regions for effective learning rates between unclipped methods and Double-Clip MGDA. To select feasible learning rates and ensure convergence across control groups, we fix the batch size to $B=256$, $\rho = 1e^{-5}$, and $\gamma = 1e^{-3}$ for both linear and logistic regression. We then vary the learning rates for updating the model parameters $\theta$ and the preference vector $w$ over predefined candidate sets.

For linear regression, we perform a greedy search over the set $\{1e^{-5}, 5e^{-5}, 5e^{-4}, 1e^{-3}, 5e^{-3}, 1e^{-2}\}$ for all methods except Double-Clip MGDA. For Double-Clip MGDA, we restrict the search to $\{1e^{-3}, 5e^{-3}, 1e^{-2}, 5e^{-2}, 1e^{-1}\}$. For logistic regression, we use the candidate set $\{1e^{-5}, 5e^{-5}, 5e^{-4}, 1e^{-3}, 5e^{-3}, 1e^{-2}\}$ for all methods except Double-Clip MGDA, and $\{5e^{-3}, 1e^{-3}, 1e^{-2}, 5e^{-2}, 1e^{-1}\}$ for Double-Clip MGDA.
\begin{figure}[ht]
    \centering
    \begin{subfigure}{0.23\linewidth}
        \includegraphics[width=\linewidth]{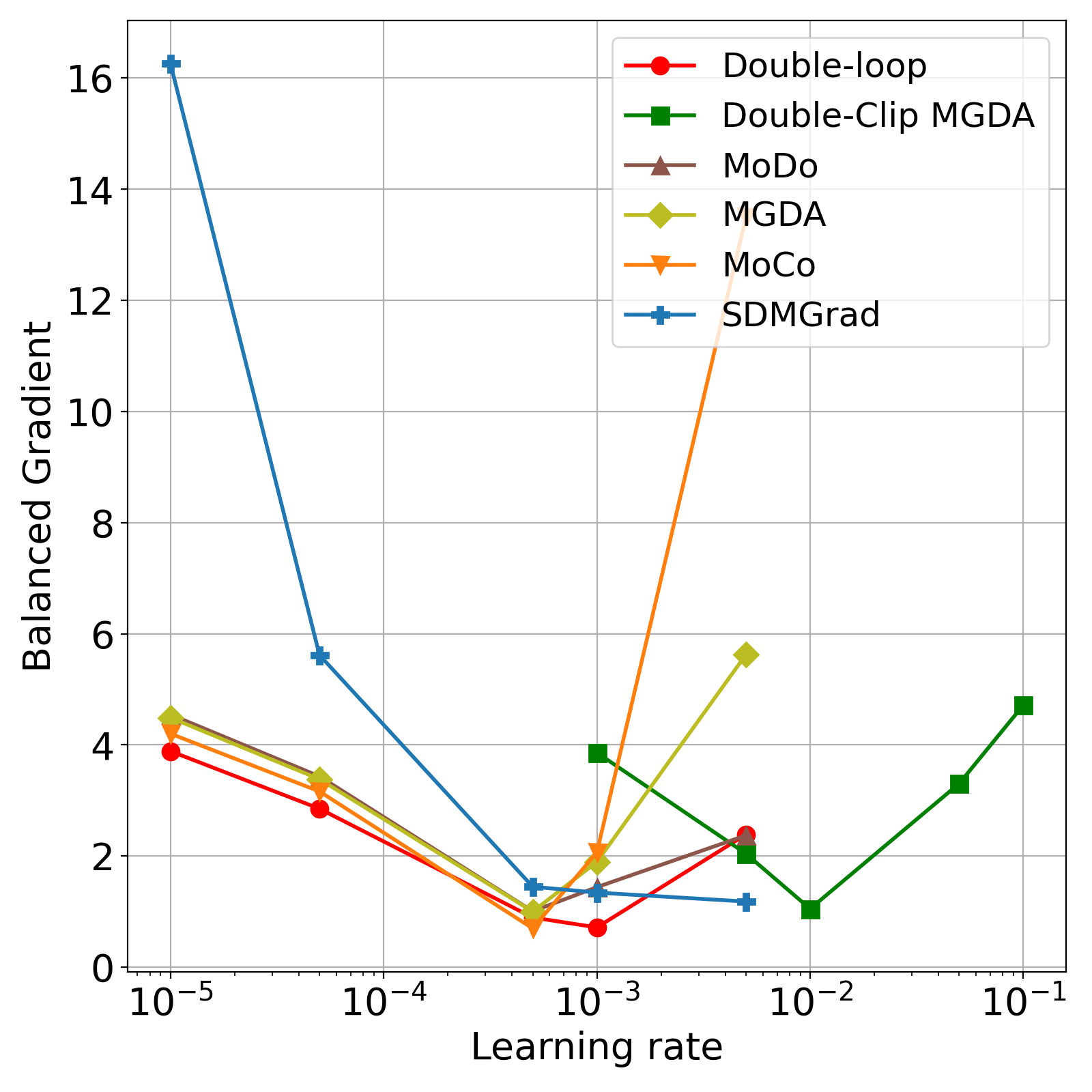}
    \end{subfigure}
    \hspace{0.02\linewidth}
    \begin{subfigure}{0.23\linewidth}
        \includegraphics[width=\linewidth]{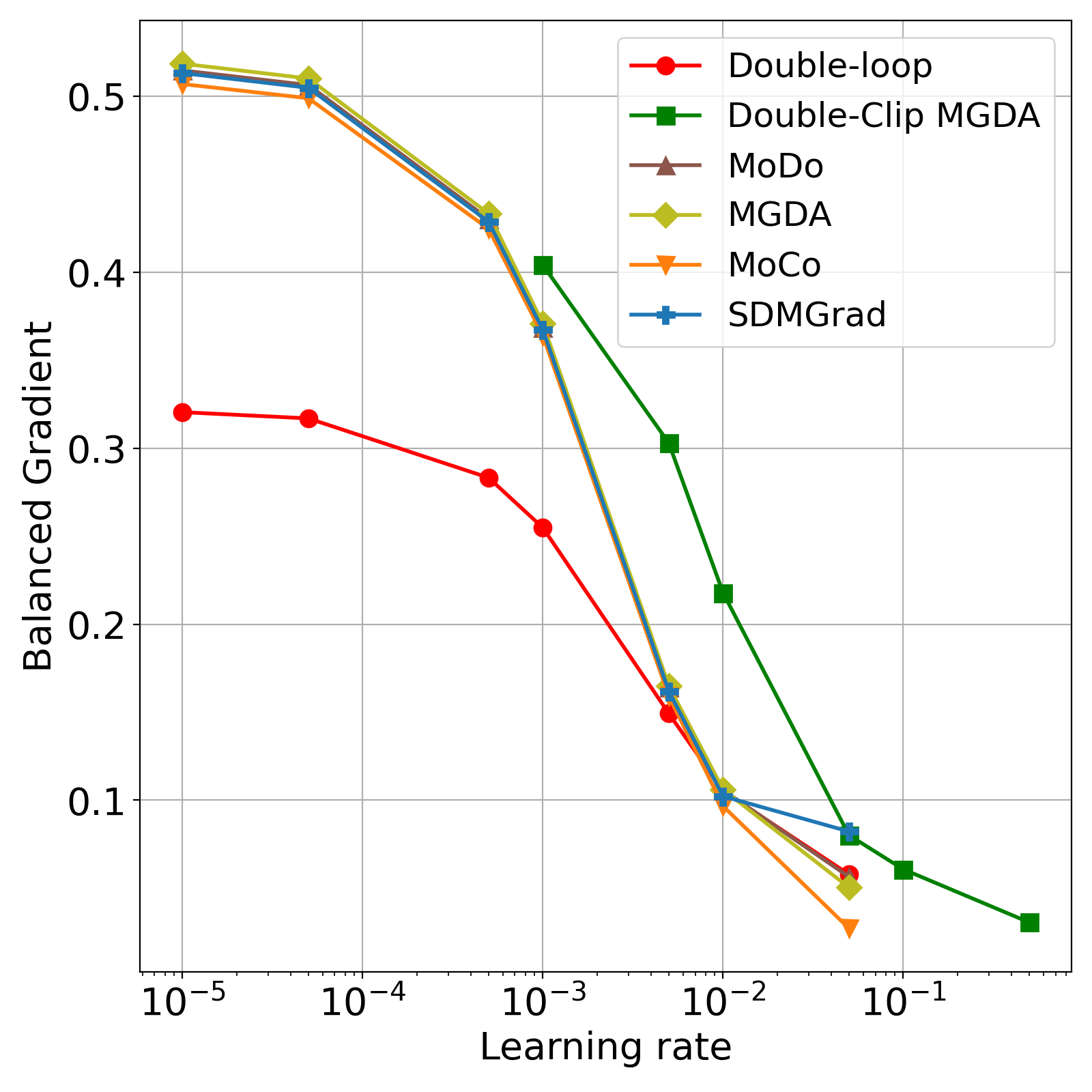}  
    \end{subfigure}
    \caption{Ablation on effective learning rates for Linear(Left) and Logistic Regressions(Right)}
    \label{fig: ablation on LR}
\end{figure}

Figure~\ref{fig: ablation on LR} illustrates the relationship between the learning rate and the averaged balanced gradient norm over the last $20$ iterations. As shown, unclipped methods exhibit a different feasible region compared with Double-Clip MGDA. To achieve comparable convergence error, Double-Clip MGDA (green line) requires a larger learning rate, which is consistent with Theorem~\ref{thm: convergence of alg3} statements on choices of $\beta, \gamma$. 

\subsection{Ablation on Batch Size }\label{Appendix: ablation on batch size and inner-steps}
\begin{figure}[ht]
    \centering
    \begin{subfigure}{0.23\linewidth}
        \includegraphics[width=\linewidth]{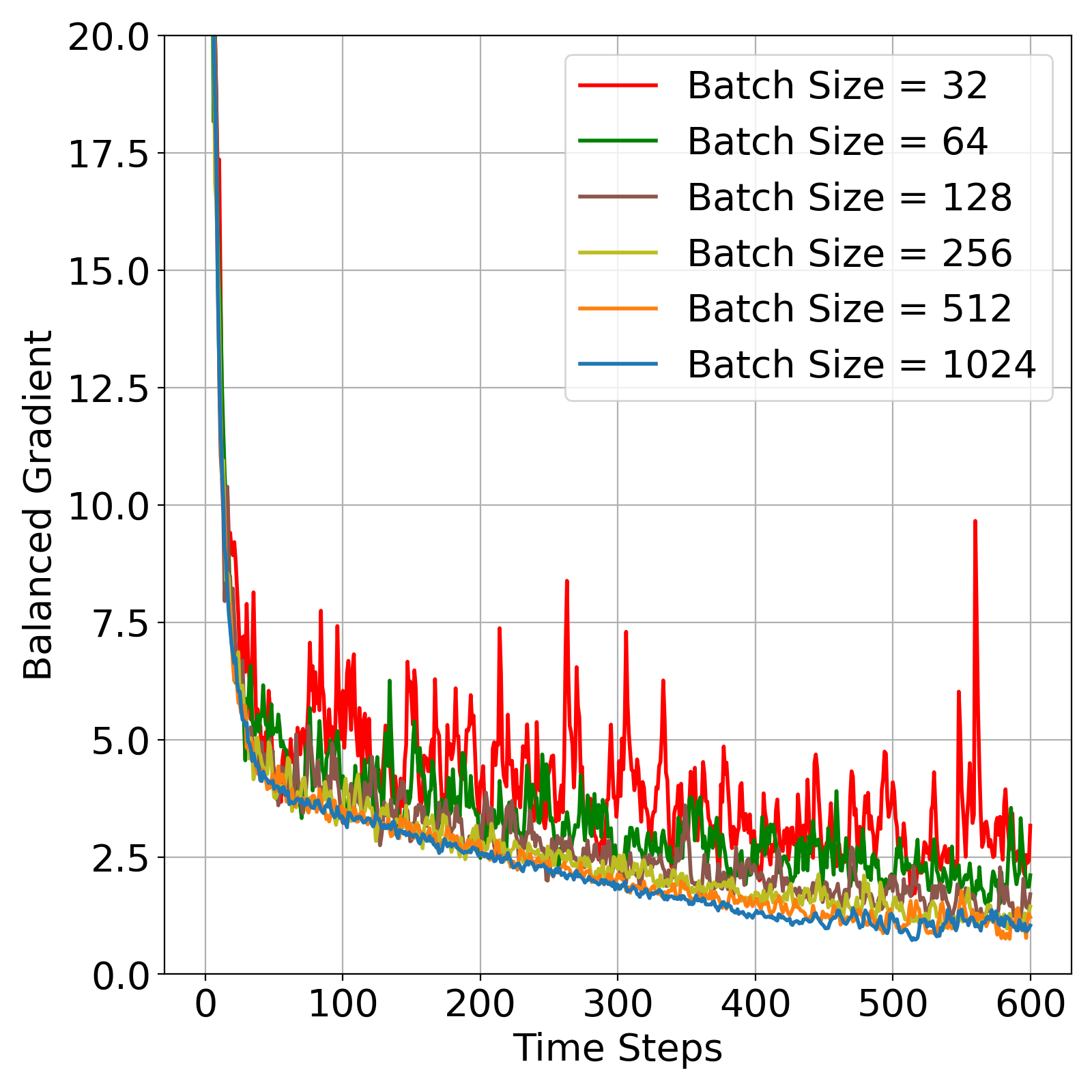}
        \caption{Linear}
    \end{subfigure}
    \hspace{0.02\linewidth}
    \begin{subfigure}{0.23\linewidth}
        \includegraphics[width=\linewidth]{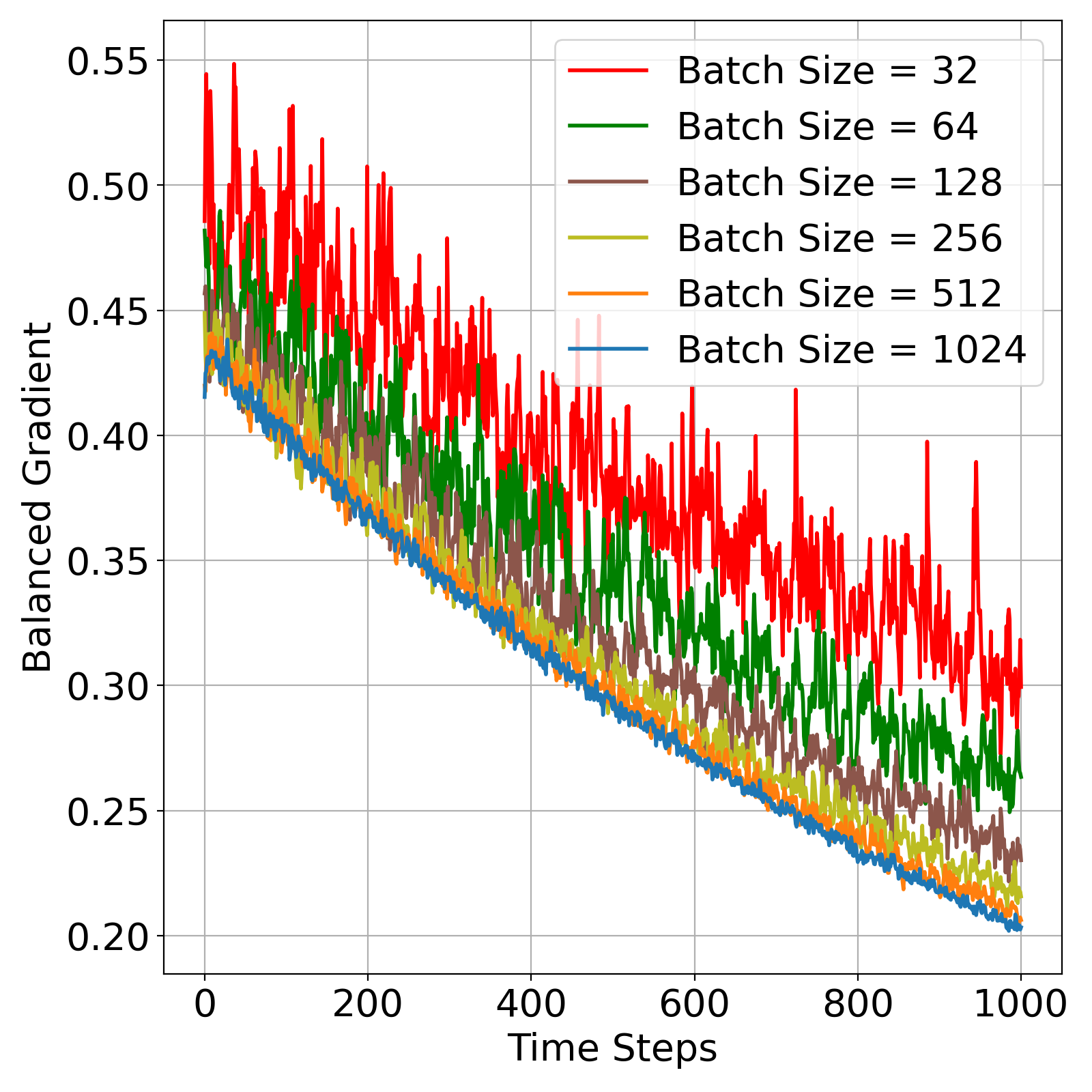}
        \caption{Logistic}
    \end{subfigure} 
        \begin{subfigure}{0.23\linewidth}
        \includegraphics[width=\linewidth]{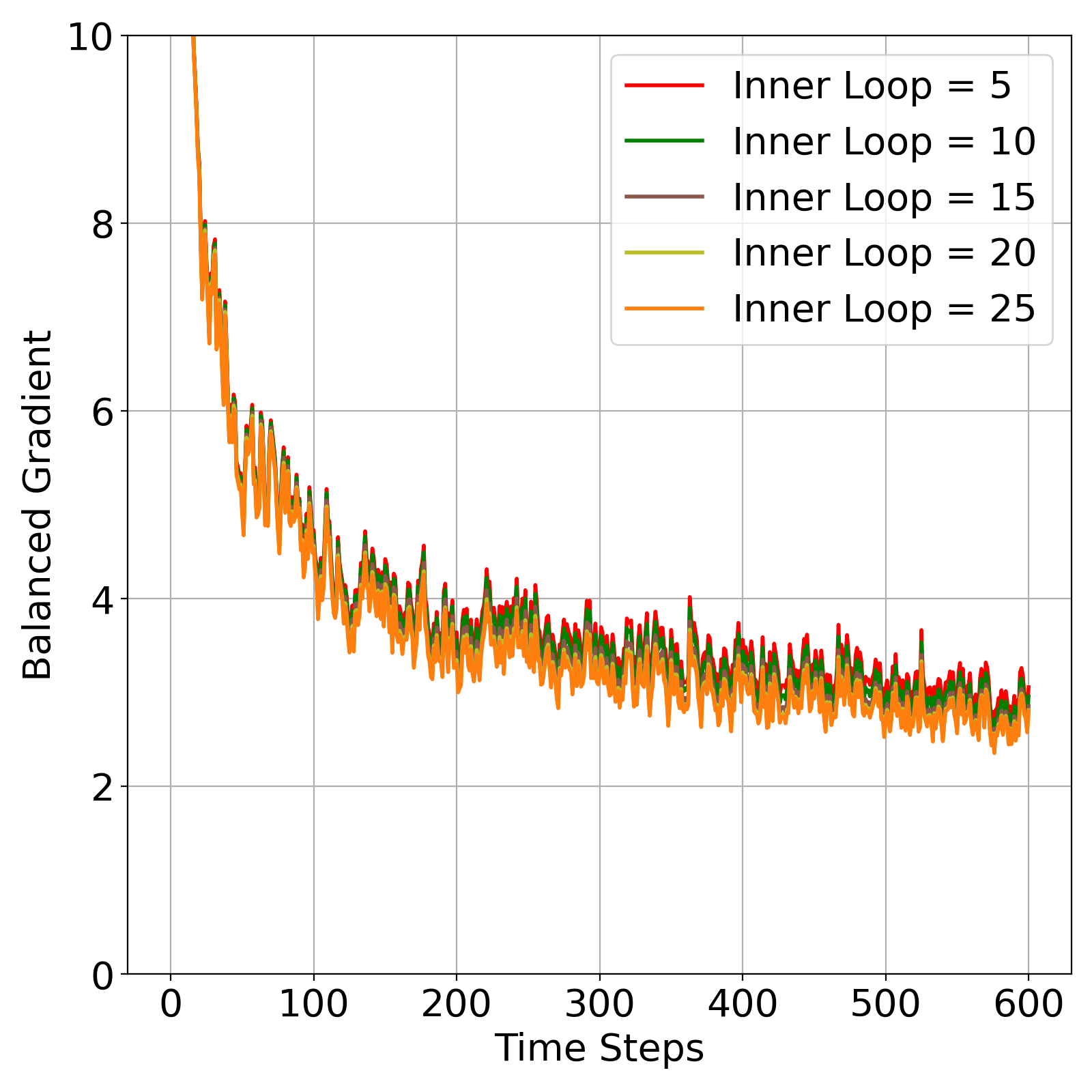}
        \caption{Linear}
    \end{subfigure}
    \hspace{0.02\linewidth}
    \begin{subfigure}{0.23\linewidth}
        \includegraphics[width=\linewidth]{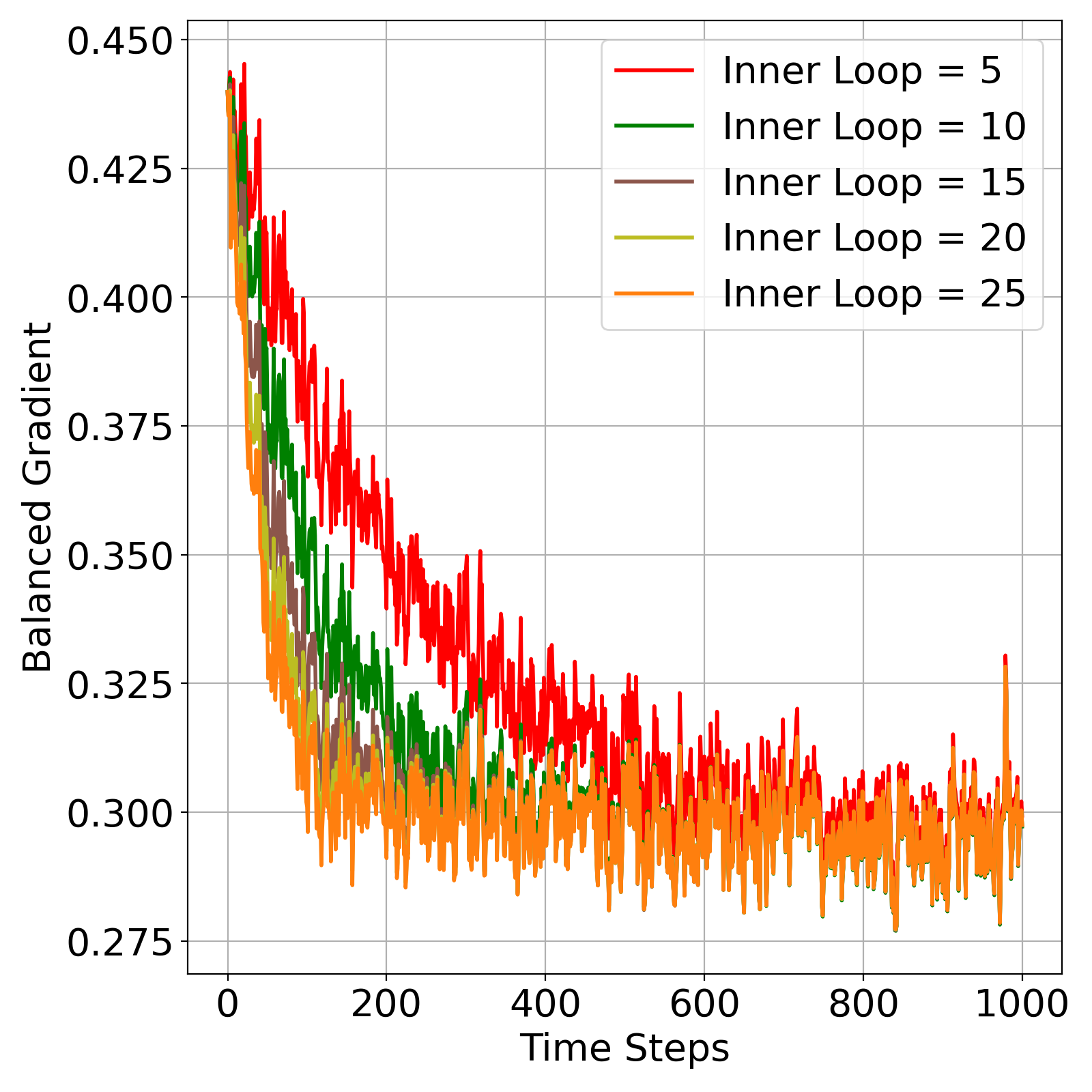}
        \caption{Logistic}
    \end{subfigure}
    \caption{Ablation on Effects of Batch Size (Left) and Inner-Step (Right)}
    
    \label{fig: batch-size and inner-step ablation}
\end{figure}

As suggested by Theorem~\ref{thm: convergence of alg3}, Double-Clip MGDA requires a batch size $N_1, N_2=\mathcal{O}(\epsilon^{-2})$ to guarantee convergence. In this section, we study the practical choice of batch size that balances computational efficiency and optimization stability. Specifically, for linear and logistic regression, we unify the learning rates across all control groups as $\beta, \gamma = 1e^{-2}$, and set $c_1=f_1=1.0$, $c_2=f_2=0.5$, and $\rho = 1e^{-5}$. We then vary the batch sizes $N_1, N_2$ over $\{32, 64, 128, 256, 512, 1024\}$.

The convergence behavior, measured by the balanced gradient, is shown in Figure~\ref{fig: batch-size and inner-step ablation} (a) and (b). As observed, when the batch size is small (e.g., $32$ or $64$), the gradients exhibit significant fluctuations during training, indicating instability. As the batch size increases beyond $256$, this instability becomes negligible. These results justify our practical choice of batch size $256$ as a sufficient trade-off between stability and efficiency.

\subsection{Effect of the Number of Inner Iterations}\label{Appendix: inner-iteration and wall-clock}
According to Theorem~\ref{thm: convergence of algorithm 1}, the primary convergence bottleneck of Double-Loop MGDA arises from the high complexity of the inner loop, which requires $D=\mathcal{O}(\epsilon^{-4})$ inner iterations per step to guarantee convergence. In this section, we study the practical choice of the number of inner iterations, aiming to balance overall computational cost and convergence performance.

Specifically, to ensure convergence across all control groups, for linear regression we unify the hyperparameters as $\alpha=\beta = 5e^{-5}$, $\gamma = 1e^{-3}$, and $\rho = 1e^{-5}$. For logistic regression, we set $\alpha, \gamma = 1e^{-3}$, $\beta = 6e^{-4}$, and $\rho = 1e^{-6}$ across all groups. We then vary the number of inner iterations over $\{5, 10, 15, 20, 25\}$. The convergence behavior, measured by the balanced gradient, is shown in Figure~\ref{fig: batch-size and inner-step ablation} (c) and (d). As observed, for linear regression, the convergence behavior is relatively insensitive to the number of inner steps under these hyperparameter settings. For logistic regression, although using $25$ inner steps leads to faster convergence, the final convergence results across all control groups show no significant differences. This suggests that, since the inner subproblem is convex, it is comparatively easier to solve than the outer loop responsible for updating model parameters and the preference vector.

In deep learning experiments, to reduce computational overhead, we further set the number of inner iterations to $5$ for Double-Loop MGDA. We compare its iteration-wise wall-clock time with Double-Clip MGDA and other baselines, with results shown in Figure~\ref{fig: Wall-clock}. Combined with the accuracy results reported in Table~\ref{Table: Multi-mnist FGSM} and Table~\ref{table:accuracy-celeba}, these findings indicate that setting the number of inner iterations to $5$ is sufficient to achieve performance comparable to other baselines. Moreover, with only $5$ inner iterations, the additional computational overhead becomes negligible relative to the cost of the network backbone and dataset size.


\begin{figure}[ht]
    \centering
    \includegraphics[width=0.6\linewidth]{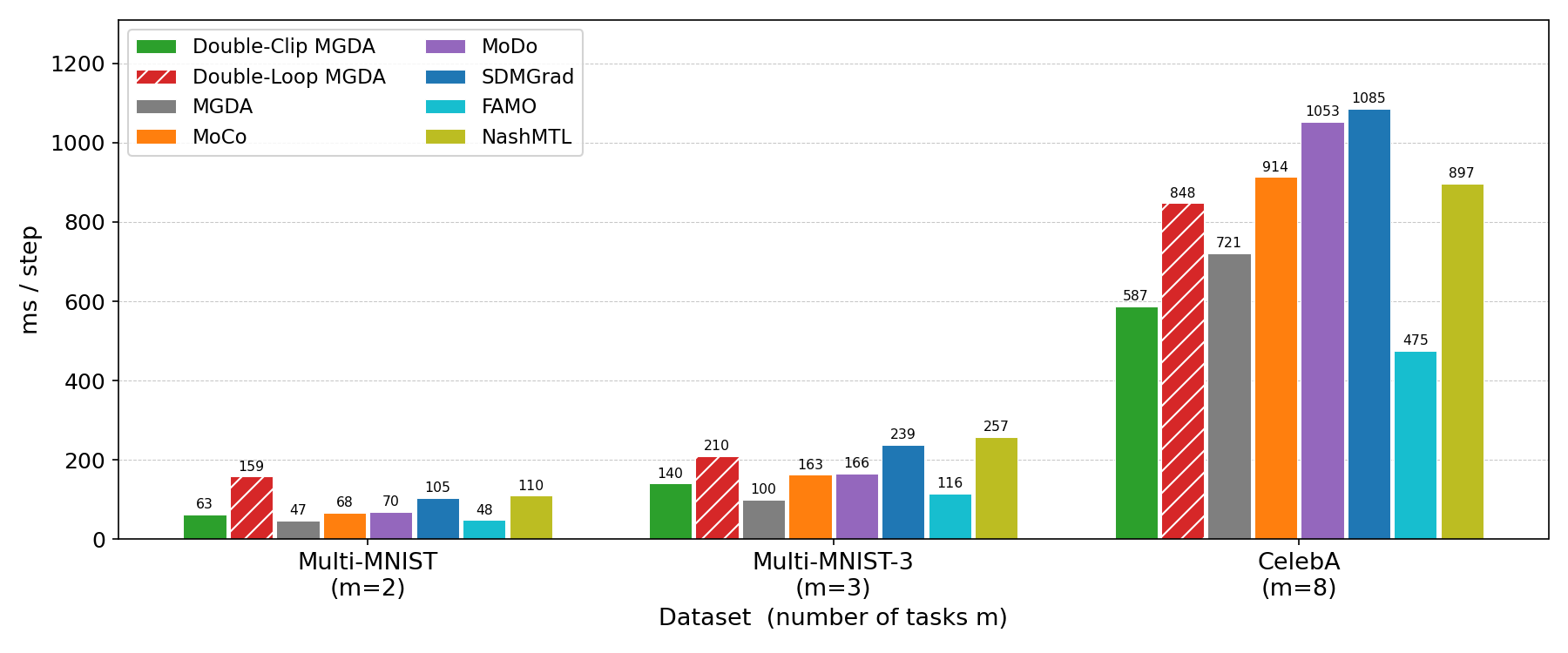}
    \caption{Iteration-wise Wall-clock Time Comparison for Each Method.}
    \label{fig: Wall-clock}
\end{figure}

\section{Notation for Convergence Analysis}

Table~\ref{tab:convergence-notation} summarizes the notation used
in the convergence analysis of Double-loop MGDA
(Algorithm~\ref{alg1}) and Single-loop Double-Clip
(Algorithm~\ref{alg1-3}).
All parameter spaces are real space.
Throughout the paper, we assume that our studied objective $\phi^i(\theta)$, $\fL^i(\theta, \eta),\fhL^i(\theta, \eta),\forall i\in [m]$ are well-defined, i.e., function value evaluated at any $(\theta, \eta)$ is finite. Its minimal value $\phi^{i,*}$, $\fL^{i,*}$ and $\fhL^{i,*}$ is finite. The corresponding minimizer $\theta^*, \eta_{\theta}^{i,*}$ exists for all $i\in [m]$.
\begingroup
\small
\setlength{\tabcolsep}{6pt}
\renewcommand{\arraystretch}{1.18}
\setlength{\LTleft}{0pt}
\setlength{\LTright}{0pt}

\begin{longtable}{
    @{}
    >{\raggedright\arraybackslash}p{0.27\textwidth}
    >{\raggedright\arraybackslash}
      p{\dimexpr0.73\textwidth-2\tabcolsep\relax}
    @{}
}
\caption{Notation used in the convergence analysis.}
\label{tab:convergence-notation}\\
\toprule
\textbf{Notation} & \textbf{Definition / description} \\
\midrule
\endfirsthead

\multicolumn{2}{@{}l}{%
    \textit{Table~\thetable\ (continued)}}\\
\toprule
\textbf{Notation} & \textbf{Definition / description} \\
\midrule
\endhead

\midrule
\multicolumn{2}{r@{}}{\textit{Continued on next page}}\\
\endfoot

\bottomrule
\endlastfoot

\multicolumn{2}{@{}l}{\textbf{Basic notation and gradient matrices}}\\*
\addlinespace[2pt]

$n,\ m$
&
Dimension of $\theta$ and number of objectives, respectively;
$[m]=\{1,\ldots,m\}$.
\\

$\theta,\ \eta$
&
$\theta\in\mathbb{R}^{n}$ and
$\eta=(\eta^1,\ldots,\eta^m)\in\mathbb{R}^{m}$.
\\

$\eta_{\theta}^{i,*}$
&
Optimal inner variable for objective $i$ at $\theta$.
\\

$\nabla\Phi(\theta)$
&
Exact multi-objective gradient matrix in $\mathbb{R}^{n\times m}$.
Its $i$th column is
$\nabla\phi^i(\theta)
 =\nabla_{\theta}\fL^i(\theta,\eta_{\theta}^{i,*})$.
\\

$\nabla_{\theta}\fL(\theta,\eta)$
&
Gradient matrix in $\mathbb{R}^{n\times m}$ with $i$th column
$\nabla_{\theta}\fL^i(\theta,\eta^i)$.
\\

$\nabla_{\eta}\fL(\theta,\eta)$
&
\textbf{Diagonal} matrix in $\mathbb{R}^{m\times m}$ with $i$th
diagonal entry $\nabla_{\eta^i}\fL^i(\theta,\eta^i)$.
\\

$\nabla_{\theta,\eta}\fL(\theta,\eta)$
&
Joint gradient matrix:
$\begin{bmatrix}
\nabla_{\theta}\fL(\theta,\eta)\\
\nabla_{\eta}\fL(\theta,\eta)
\end{bmatrix}
\in\mathbb{R}^{(n+m)\times m}$.
\\

$\|\cdot\|,\ |\cdot|$
&
Euclidean $\ell_2$ norm and $\ell_1$ norm, respectively.
\\

\midrule
\multicolumn{2}{@{}l}{\textbf{Double-loop MGDA}}\\*
\addlinespace[2pt]

$t,\ d$
&
Outer-loop and inner-loop iteration indices, respectively.
\\

$V_d$
&
Stochastic estimator of
$\nabla_{\eta}\fL(\theta_t,\eta_{t,d})$;
$V_d\in\mathbb{R}^{m\times m}$.
\\

$Y_t,\ \bar{Y}_t,\ \tilde{Y}_t$
&
Mutually independent stochastic estimators of the
$\theta$-gradient matrices evaluated at
$\eta_{t,d}$, $\eta_{t,\bar d}$, and
$\eta_{t,\tilde d}$, respectively, with $\theta=\theta_t$.
Each belongs to $\mathbb{R}^{n\times m}$.
\\

$\Upsilon_{t,d}$
&
Inner-gradient estimation error:
$\Upsilon_{t,d}
 =V_d-\nabla_{\eta}\fL(\theta_t,\eta_{t,d})
 \in\mathbb{R}^{m\times m}$.
\\

$\Gamma_{t,2}$
&
Stochastic gradient error:
$\Gamma_{t,2}
 =Y_t-\nabla_{\theta}\fL(\theta_t,\eta_{t,d})$.
\\

$\Gamma_{t,3}$
&
Stochastic gradient error:
$\Gamma_{t,3}
 =\bar{Y}_t
  -\nabla_{\theta}\fL(\theta_t,\eta_{t,\bar d})$.
\\

$\Gamma_{t,4}$
&
Stochastic gradient error:
$\Gamma_{t,4}
 =\tilde{Y}_t
  -\nabla_{\theta}\fL(\theta_t,\eta_{t,\tilde d})$.
\\

$A_{t,2},\ A_{t,3},\ A_{t,4}$
&
Inner-optimization errors:
\newline
$A_{t,2}
 =\nabla_{\theta}\fL(\theta_t,\eta_{t,d})
  -\nabla_{\theta}\Phi(\theta_t)$,
\newline
$A_{t,3}
 =\nabla_{\theta}\fL(\theta_t,\eta_{t,\bar d})
  -\nabla_{\theta}\Phi(\theta_t)$,
\newline
$A_{t,4}
 =\nabla_{\theta}\fL(\theta_t,\eta_{t,\tilde d})
  -\nabla_{\theta}\Phi(\theta_t)$.
\\

$E_{t,2},\ E_{t,3},\ E_{t,4}$
&
Total gradient errors, with
$E_{t,j}=\Gamma_{t,j}+A_{t,j}$ for $j\in\{2,3,4\}$:
\newline
$E_{t,2}=Y_t-\nabla_{\theta}\Phi(\theta_t)$,
\newline
$E_{t,3}=\bar{Y}_t-\nabla_{\theta}\Phi(\theta_t)$,
\newline
$E_{t,4}=\tilde{Y}_t-\nabla_{\theta}\Phi(\theta_t)$.
All $\Gamma_{t,j}$, $A_{t,j}$, and $E_{t,j}$
belong to $\mathbb{R}^{n\times m}$.
\\

$L_0,\ L_2$
&
Smoothness parameters for $\fL(\theta,\eta)$.
\\

$K_0,\ K_1,\ K_2$
&
Constants controlling the variance bounds for the
columnwise stochastic errors
$\Gamma_{t,2}^i$, $\Gamma_{t,3}^i$,
$\Gamma_{t,4}^i$, and $\Upsilon_{t,d}^i$,
$i\in[m]$.
\\
$\calH_t, \calG_t$
& $\sigma$-algebra collects history randomness including initialization and all random draws in outer iterations before $t$; $\calH_t \vee \{\text{$\sigma$-algebra of all inner-solver randomness and output selections at iteration t.} \} $
\\
\midrule
\multicolumn{2}{@{}l}{\textbf{Single-loop Double-Clip MGDA}}\\*
\addlinespace[2pt]

$\fhL(\theta,\eta)$
&
Rescaled vector-valued loss:
$\fhL(\theta,\eta)=\fL(\theta,G\sqrt{m}\,\eta)$.
Its $i$th component is
\[
\begin{aligned}
\fhL^i(\theta,\eta)
={}&\lambda\mathbb{E}_{\xi\sim\mathbb{P}}
\big[
f^*\!\left(
\frac{\ell^i(\theta;\xi)-G\sqrt{m}\,\eta^i}
     {\lambda}
\right)
\big]+G\sqrt{m}\,\eta^i.
\end{aligned}
\]
\\

$\nabla_{\theta}\fhL,\ \nabla_{\eta}\fhL$
&
Gradient matrices of the rescaled loss with respect to
$\theta$ and $\eta$, belonging to
$\mathbb{R}^{n\times m}$ and
$\mathbb{R}^{m\times m}$, respectively.
\\

$Z_t$
&
Stochastic estimator of
$\nabla_{\eta}\fhL(\theta_t,\eta_t)$;
$Z_t\in\mathbb{R}^{m\times m}$.
\\

$X_t$
&
Stochastic estimator of
$\nabla_{\theta}\fhL(\theta_t,\eta_{t+1})$;
$X_t\in\mathbb{R}^{n\times m}$.
\\

$\HUpsilon_t$
&
Error relative to the rescaled $\eta$-gradient:
$\HUpsilon_t
 =Z_t-\nabla_{\eta}\fhL(\theta_t,\eta_t)
 \in\mathbb{R}^{m\times m}$.
\\

$\HGamma_t$
&
Error relative to the rescaled $\theta$-gradient:
$\HGamma_t
 =X_t-\nabla_{\theta}\fhL(\theta_t,\eta_{t+1})
 \in\mathbb{R}^{n\times m}$.
\\

$\hL_0,\ \hL_1,\ \hL_2$
&
Generalized smoothness parameters of $\fhL$
under the $(L_0,L_1,L_2)$-smoothness framework
of \citet{qirevistfDRO}.
\\

$\hat K_0,\ \hat K_1,\ \hat K_2$
&
Constants controlling the variance bounds for the
columnwise stochastic errors in the single-loop analysis.
\\
$\mathcal{F}_t$
& 
$\sigma$-algebra collects history randomness including initialization and all random draws before $t$.
\\
$\htau,\htau_1,\htau_2$
& Stopping time
\\
\midrule
\multicolumn{2}{@{}l}{\textbf{Bounds and convergence parameters}}\\*
\addlinespace[2pt]

$\bar F$
&
Upper bounds on the maximal function-value gaps
constructed in Theorem~\ref{thm: convergence of alg3}, respectively.
\\

$\Lambda$
&
Upper bounds on
$\|\nabla_{\eta^i}\fhL^i(\theta_t,\eta_t^i)\|$, obtained using gradient-function value relationship \eqref{eq: boundgrad}.
\\

$\delta\in(0,1)$
&
Upper bound on the tail-event probability.
\\

$\epsilon>0$
&
Target accuracy.
\\

$\tepsilon$
&
Rescaled target accuracy for the inner-loop analysis:
$\tepsilon
 =\min\{1,G^{-1}\}\epsilon$.
\\

\end{longtable}
\endgroup

We establish the function-value-gap gradient bound relationship following Lemma~5.1 from \citet{Haochuangensmooth}; 
For function satisfying $\hat{L}_2$-smooth, i.e., $
\|\nabla_{\eta} \fhL^i({\theta},\eta)-\nabla_{\eta} \fhL^i({\theta},\bar{\eta}^{i
}) \|\leq \hat{L}_2\|\eta^i-\bar{\eta}^i \|$
In particular, letting $\bar{\eta}^{i}=\eta^i-\frac{\nabla_{\eta} \fhL^i(\theta,\eta^i)}{\hat{L}_2(\|\nabla_{\eta} \fhL^i(\theta,\eta^i)\|+1)^{1/2}}$ and applying the descent lemma gives 
\begin{align}  
\|\nabla_{\eta} \fhL^i(\theta,\eta^{i})\|^2 \leq 2\hat{L}_2(\fhL^i(\theta,\eta^i)-\inf_{\eta^i}\fhL^i(\theta, \eta^{i}))(\|\nabla_{\eta} \fhL^i(\theta,\eta^{i})\|+ 1). \label{eq: boundgrad}
\end{align}
Therefore, if one can find a constant $\bar{F}$ such that $\fhL^i(\theta,\eta^i)-\fhL(\theta, \eta^{i,*})\leq \bar{F}$, defining $\Lambda := \sup\{u\geq 0| u^2\leq {2\hat{L}_2}\bar{F}(u+1) \}$, we obtain $\|\nabla_{\eta} \fhL^i(\theta_t,\eta^i)\|\leq \Lambda$.

\section{Relevant Properties of $\phi^i(\theta)$}\label{Appendix: Relevant Properties of Phi}
\begin{restatable}[$G$-Lipschitz of $\phi^i(\theta)$]{lemma}{Glipschitzphi}\label{lemma: gradient bound of phi}
    Let Assumption \ref{assumption 1} hold. Denote $\eta_{{\theta}}^{i,*}\in \arg\min_{{\eta}\in \mathbf{R}}\fL^i({\theta},\eta^i)$. For any $\theta,\theta'$, we have$
        \|\phi^i({\theta})-\phi^i(\theta') \|\leq G\|{\theta}-{\theta}' \|$
holds for $i\in[m]$. Consequently, at $(\theta, \eta_{\theta}^{i,*})$, we have $\|\nabla \phi^i(\theta)\|\leq G$.
\end{restatable}
\begin{proof}
    Strong duality between $H^i(\theta,\mathbb{Q}^*)$ and $\phi^i(\theta)$ implies given a fixed $\theta$, there is a corresponding worst-case distribution $\mathbb{Q}^*$ such that $H^i(\theta,\mathbb{Q}^*) = \phi^i(\theta)$. Moreover, at $\theta'$ $\mathbb{Q}^*$ is feasible thus $\phi^i(\theta')\geq H^i(\theta',\mathbb{Q}^*)$. This implies
    \begin{align}
        \phi^i(\theta) - \phi^i(\theta') &\leq H^i(\theta,\mathbb{Q}^*) - H^i(\theta',\mathbb{Q}^*) \nonumber\\
        &\leq \mathbb{E}_{\mathbb{Q}^*}[\ell^i(\theta;\xi)-{\ell}^i(\theta';\xi)]\nonumber\\
        &\leq G\|\theta-\theta'\|,
    \end{align}
    which completes proof. 
\end{proof}

\begin{restatable}[$L$-smooth of $\phi^i({\theta})$ along certain $(\theta, \eta)$-trajectory ]{lemma}{smoothproperty}\label{lemma: smooth of phi}
    Let Assumption \ref{assumption 1} hold. Denote $\eta_{{\theta}}^{i,*}\in \arg\min_{{\eta}\in \mathbf{R}}\fL^i({\theta},\eta^i)$. For any $\theta,\theta'$, we have$
        \|\nabla \phi^i({\theta})-\nabla_{{\theta}} \fL^i({\theta'},\eta^{i,*}_{{\theta}}) \|\leq L_0\|{\theta}-{\theta}' \|$
holds for $i\in[m]$, where $L_0=G^2M\lambda^{-1}+L$.
\end{restatable}
\begin{proof}
Define {matrices} $A$ and $B$ as follows
\begin{align}
    A&= \mathbb{E}_{\xi}\Big[f^{*'}(\frac{\ell^i({\theta};\xi)-\eta_{{\theta}}^{i,*}}{\lambda})\nabla \ell^i({\theta};\xi) - f^{*'}(\frac{\ell^i({\theta};\xi)-\eta_{{\theta}}^{i,*}}{\lambda})\nabla \ell^i({\theta}';\xi) \Big]\nonumber\\
    B&=\mathbb{E}_{\xi}\Big[f^{*'}(\frac{\ell^i({\theta};\xi)-\eta_{{\theta}}^{i,*}}{\lambda})\nabla \ell^i({\theta}';\xi)-f^{*'}(\frac{\ell^i({\theta'};\xi)-\eta_{{\theta}}^{i,*}}{\lambda})\nabla \ell^i({\theta'};\xi) \Big]. \nonumber
\end{align}
It's obvious that $\|A+B\| =\| \nabla \phi_i({\theta})-\nabla_{{\theta}} \fL^i({\theta'},\eta^{i,*}_{{\theta}})\|\leq \| A\|+\|B\|$. Next, we bound $\| A\|$ and $\| B\|$ separately.
For $\|A\|$, we have
\begin{align}
    &\Big\|\mathbb{E}_{\xi}\Big[f^{*'}(\frac{\ell^i({\theta};\xi)-\eta_{{\theta}}^{i,*}}{\lambda})\nabla \ell^i({\theta};\xi)-f^{*'}(\frac{\ell^i({\theta};\xi)-\eta_{{\theta}}^{i,*}}{\lambda})\nabla \ell^i({\theta}';\xi) \Big]\Big\|\nonumber\\
    &\leq \mathbb{E}_{\xi} \big\|f^{*'}(\frac{\ell^i({\theta};\xi)-\eta_{{\theta}}^{i,*}}{\lambda})\nabla \ell^i({\theta};\xi)-f^{*'}(\frac{\ell^i({\theta};\xi)-\eta_{{\theta}}^{i,*}}{\lambda})\nabla \ell^i({\theta}';\xi) \big\| \nonumber\\
    &\leq \mathbb{E}_{\xi}\Big[|{{f^*}'}(\frac{\ell^i({\theta};\xi)-\eta^{i,*}_{{\theta}'}}{\lambda})|\cdot \|\big(\nabla \ell^i({\theta};\xi)-\nabla \ell^i({\theta}';\xi) \big) \big\|\Big]\nonumber\\
    &\leq L \mathbb{E}_{\xi}|{{f^*}'}(\frac{\ell^i({\theta};\xi)-\eta^{i,*}_{{\theta}}}{\lambda}) |\|{\theta}-{\theta}'\|\nonumber\\
    &=L\|{\theta}-{\theta}'\|,
\end{align}
where the first inequality utilizes Jensen's inequality; the second inequality utilizes Cauchy-Schwarz inequality; the third inequality utilizes $L$-smooth of $\ell(\cdot)$ and the last equality utilizes maximization arguments such that $(f^{*})'(\cdot)\geq 0$ (recall that the primal variable is $r =\mathrm{d}\mathbb{Q}/\mathrm{d}\textit{Unif}$, which is Radon–Nikodym derivative over p.d.f.) and $\nabla_{\eta_i}\fL^i(\theta,\eta_{\theta}^{i,*})=1-\mathbb{E}_\xi[f^{*'}(\frac{\ell^i({\theta};\xi)-\eta^{i,*}_{{\theta}}}{\lambda})]=0$.

Similarly, for $\|B\|$, we have
\begin{align}
    &\Big\|\mathbb{E}_{\xi}\Big[f^{*'}(\frac{\ell^i({\theta};\xi)-\eta_{{\theta}}^{i,*}}{\lambda})\nabla \ell^i({\theta}';\xi) - f^{*'}(\frac{\ell^i({\theta}';\xi)-\eta_{{\theta}}^{i,*}}{\lambda})\nabla \ell^i({\theta}';\xi) \Big] \Big\| \nonumber\\
    &\leq \mathbb{E}_{\xi}\Big\|f^{*'}(\frac{\ell^i({\theta};\xi)-\eta_{{\theta}}^{i,*}}{\lambda})\nabla \ell^i({\theta}';\xi) - f^{*'}(\frac{\ell^i({\theta}';\xi)-\eta_{{\theta}}^{i,*}}{\lambda})\nabla \ell^i({\theta}';\xi) \Big\| \nonumber \\
    &\leq \mathbb{E}_{\xi}\Big[\|f^{*'}(\frac{\ell^i({\theta};\xi)-\eta_{{\theta}}^{i,*}}{\lambda}) -f^{*'}(\frac{\ell^i({\theta}';\xi)-\eta_{{\theta}}^{i,*}}{\lambda}) \|\|\nabla \ell^i({\theta}';\xi) \|\Big] \nonumber\\
    &\leq G\mathbb{E}_{\xi}\Big[ \|f^{*'}(\frac{\ell^i({\theta};\xi)-\eta_{{\theta}}^{i,*}}{\lambda}) -f^{*'}(\frac{\ell^i({\theta}';\xi)-\eta_{{\theta}}^{i,*}}{\lambda}) \|\Big] \nonumber\\
    &\leq G^2M\lambda^{-1}\|{\theta}-{\theta}' \|,
\end{align}
where the first inequality applies Jensen's inequality; the second inequality utilizes Cauchy-Schwarz inequality; the third inequality utilizes $G$-Lipschitz of $\ell(\cdot)$; the {fourth} inequality utilizes $M$-smooth of $f^*(\cdot)$, $G$-lipschitz of $\ell(\cdot)$. Combine above two inequality gives desired result.
\end{proof}

\subsection{Inner-loop Convergence of Algorithm \ref{alg1}}\label{Appendix: Proof of inner loop}
\begin{lemma}[Lemma 1 \citep{qirevistfDRO}]
Let assumptions \ref{assumption 1} hold. Fix $\theta$, gradient $\nabla_{\eta^i}L^i(\theta,\eta^i)$ satisfy $L_2$-smooth property, i.e., for any $\eta^i,\bar{\eta}^{i}\in \mathbf{R}$, we have
\begin{align}
    \|\nabla_{\eta^i}\fL^i(\theta,\eta^i) - \nabla_{\eta^i}\fL^i(\theta,\bar{\eta}^i) \|\leq L_2\|\eta^i-\bar{\eta}^i\|,
\end{align}
where $L_2 = M\lambda^{-1}$.
\end{lemma}
\begin{lemma}[Lemma 3 \citep{qirevistfDRO}]
Let assumption \ref{assumption 1} hold. For each $\fL(\theta_t,\eta^i),i\in[m]$, variance of $\Upsilon_t^i, \Gamma^i_{t,2},\Gamma^i_{t,3},\Gamma^i_{t,4}$ can be upper bounded as
\begin{align}
    \mathbb{E}_{{\xi}^{i}_{t}\sim \mathbb{P}^i}[\|\Gamma^{i}_{t,2} \|^2], \mathbb{E}_{\bar{\xi}^{i}_{t}\sim \mathbb{P}^i}[\|\Gamma^{i}_{t,3} \|^2], &\mathbb{E}_{\tilde{\xi}^{i}_{t}\sim \mathbb{P}^i}[\|\Gamma^{i}_{t,4} \|^2] \leq K_0+K_1|\nabla_{\eta^i}\fL({\theta}_t,\eta^i)|^2, \nonumber \\
    &\text{ and }\mathbb{E}_{\xi^{i}_{t,d}\sim \mathbb{P}^i}[\|\Upsilon_{t,d}^{i} \|^2 ]\leq K_2,
    \label{eq: variance bound}
\end{align}
where $K_0=8G^2+10G^2M^2\lambda^{-2}\kappa^2$, $K_1 = 8G^2$, $K_2 = M^2 \lambda^{-2} \kappa^2$.
\end{lemma}
\begin{proof}
    The proof exactly follows \citet{qirevistfDRO}'s proof by changing $\fL(\theta,G\eta)$ into $\fL(\theta,\eta)$.
\end{proof}
\begin{restatable}[Inner-loop Convergence of Algorithm~\ref{alg1}]{corollary}{innerloopconvergence}
Denote $\calH_t$ as $\sigma$-algebra generated by outer-loop and initialization up to time $t$.
Denote the rescaled accuracy $\tepsilon = \min\{1,G^{-1}\}\epsilon$. For warm-start gaps admit a finite upper bound $\overline{\Delta}_{\eta}$ independent of $\epsilon$, i.e.,
$\max_{0\leq t\leq T,\,i\in[m]}\{\fL^i(\theta_t,\eta^i_{t,0})-\phi^i(\theta_t)\}\leq \overline{\Delta}_{\eta}$, choosing $\gamma= \min\{\tfrac{\tepsilon^2}{2L_2K_2G^2},\tfrac{1}{L_2} \}$, 
after $D=\Theta(\overline{\Delta}_{\eta}K_2L_2^2G^4\tepsilon^{-4})$ iterations, at uniformly sampled index 
$d^i$, we have
\begin{align}
\mathbb{E}_{\eta^i_{t,d}}\big[\|\nabla_{\eta^i} \fL^i({\theta}_t, \eta^i_{t,d}) \|^2|\calH_t\big]\leq \frac{\tepsilon^2}{G^2},
\label{eq: inner loop stationary}
\end{align}
which further implies
$\mathbb{E}_{\eta_{t,d}}[\|\nabla_{\eta} \fL({\theta}_t, {\eta}_{t,{d}}) \|_F^2|\calH_t]\leq {\tfrac{m\tepsilon^2}{{G^2}}}$. Same result also holds for $\mathbb{E}_{\eta^i_{t,\bar{d}}}\big[\|\nabla_{\eta^i} \fL^i({\theta}_t, \eta^i_{t,\bar{d}}) \|^2|\calH_t\big]$ and $\mathbb{E}_{\eta^i_{t,\tilde{d}}}\big[\|\nabla_{\eta^i} \fL^i({\theta}_t, \eta^i_{t,\tilde{d}}) \|^2|\calH_t\big]$.
\end{restatable}
\begin{proof}
    The proof is similar as vanilla SGD analysis~\citep{lanSGD} with minor changes, we present the full steps here for completeness. Denote $\calH_{t,d} = \calH_t \vee \sigma\{\xi_{t,k}^i, 0\leq k<d,i\in[m] \}$
    
    For each $\fL^i({\theta}_t, \eta^{i})$, it satisfies $L_2$-smoothness w.r.t $\eta^i$, following the descent lemma, we have
    \begin{align}
        \fL^i({\theta}_t, \eta^i_{t,d+1})\leq \fL^i({\theta}_t, \eta^i_{t,d}) + \langle\nabla_{\eta^i}\fL^i({\theta}_t, \eta^i_{t,d}),\eta^i_{t,d+1}-\eta^i_{t,d} \rangle + \frac{L_2}{2} \|\eta^i_{t,d+1}-\eta^i_{t,d} \|^2.\nonumber
    \end{align}
    Put the update rule, $\eta^i_{t,d+1}=\eta^i_{t,d} - \gamma V^i_d$ in, for each $i\in[m]$, we then have
    \begin{align}
        \fL^i({\theta}_t, \eta^i_{t,d+1})\leq \fL^i({\theta}_t, \eta^i_{t,d}) - \gamma\langle\nabla_{\eta^i}\fL^i({\theta}_t, \eta^i_{t,d}),V^i_{d} \rangle + \frac{L_2\gamma^2}{2} \|V^i_{d} \|^2.\nonumber
    \end{align}
Taking conditional expectation over $\xi^{i}_{t,d}\sim \mathbb{P}^i$ and leverage variance upper bound of $\Upsilon_{t,d}^i$~\eqref{eq: variance bound}, we have
\begin{align}
    \mathbb{E}_{\xi^{i}_{t,d}}\big[\fL^i({\theta}_t, \eta^i_{t,d+1})|\calH_{t,d}\big]\leq \mathbb{E}_{\xi^{i}_{t,d}}\big[ \fL^i({\theta}_t, \eta^i_{t,d})|\calH_{t,d}\big] - \gamma(1-\frac{\gamma L_2}{2}) \|\nabla_{\eta^i} \fL^i({\theta}_t, \eta^i_{t,d}) \|^2 + \frac{L_2\gamma^2K_2}{2}. \nonumber
\end{align}
For learning rate $\gamma$ satisfies $\gamma \leq \frac{1}{L_2}$, we have 
$\gamma(1-\frac{\gamma L_2}{2})\geq \frac{\gamma}{2}$,
taking conditional expectation over $\calH_t$, reorganizing above inequality implies
\begin{align}
   \frac{\gamma}{2} \mathbb{E}[\|\nabla_{\eta^i}\fL^i({\theta}_t, \eta^i_{t,d}) \|^2|\calH_t] \leq \mathbb{E}_{\xi^i_{t,d}}\big[\fL({\theta}_t, \eta^i_{t,d}) - \fL^i({\theta}_t, \eta^i_{t,d+1}) |\calH_t\big] + \frac{L_2\gamma^2K_2}{2}.\nonumber
\end{align}
Summing above inequality through $0$ to $D-1$, multiplying both sides with $\frac{2}{D\gamma}$, we have
\begin{align}
    \frac{1}{D} \sum_{d=0}^{D-1}\mathbb{E}[\|\nabla_{\eta^i} \fL^i({\theta}_t, \eta^i_{t,d})\|^2|\calH_t] \leq \frac{2\overline{\Delta}_{\eta}}{D\gamma}+L_2\gamma K_2.\nonumber
\end{align}
Choosing $\gamma = \min\{\frac{1}{L_2},\frac{\tilde\epsilon^2}{2L_2K_2G^2}\}$, $D = \overline{\Delta}_{\eta}\max\{\frac{8L_2K_2G^4}{\tepsilon^{4}},\frac{4L_2G^2}{\tepsilon^2} \} = \Theta(\overline{\Delta}_{\eta}L_2K_2G^4\tepsilon^{-4})$, RHS of above inequality satisfies 
\begin{align}
\mathbb{E}_{\eta_{t,d}}[\|\nabla_{\eta^i} \fL^i({\theta}_t, {\eta}^i_{t,{d}}) \|^2|\calH_t]\leq {\frac{\tepsilon^2}{G^2}},
\end{align}
which completes proof.
\end{proof}

\begin{corollary}[Estimation bias of $\nabla\phi^i(\theta_t)$]\label{Corollary: estimation bias of inner-loop}
    Let Assumption \ref{assumption 1} hold, for $\eta_{t,d}$ generated from inner-loop of Algorithm~\ref{alg1}, we have     $\|\nabla_{\theta}\fL^i(\theta_t, \eta^i_{t,d}) - \nabla \phi^i(\theta_t) \|^2\leq G^2\|\nabla_{\eta^i}\fL(\theta_t,\eta_{t,d}^i)\|^2$.
    Furthermore, its expectation satisfies
    \begin{align}
        \mathbb{E}_{\eta_{t,d}}[\| A_{t,2}^i\|^2|\calH_t]=\mathbb{E}_{\eta_{t,d}}[\|\nabla_{\theta}\fL^i(\theta_t, \eta^i_{t,d}) - \nabla \phi^i(\theta_t) \|^2|\calH_t]\leq \tepsilon^2.
        \label{eq: approximation error bound }
    \end{align}
    Correspondingly, we also have $ \mathbb{E}_{\eta_{t,d}}[\|\nabla_{\theta}\fL(\theta_t, \eta_{t,d}) - \nabla \phi(\theta_t) \|_F^2|\calH_t]\leq m\tepsilon^2$. Similar bounds also hold for $A_{t,3}$, $A_{t,4}$.
\end{corollary}
\begin{proof}
   Expanding the expression of $\nabla_{\theta}\fL(\theta_t, \eta_{t,d})$ and $\nabla \phi(\theta_t)$, we have
    \begin{align}
        & \|\mathbb{E}_{{\xi}^i_t}\big[f^{*'}(\frac{\ell^i(\theta_t;{\xi}_t^i)-\eta^i_{t,d}}{\lambda})\nabla \ell^i(\theta_t;{\xi}_t^i)-f^{*'}(\frac{\ell^i(\theta_t;{\xi}_t^i)-\eta^{i*}}{\lambda})\nabla\ell^i(\theta_t;{\xi}_t^i)|\calH_t\big] \| \nonumber \\
        &\leq \mathbb{E}_{{\xi}^i_{t}} \big[\|f^{*'}(\frac{\ell^i(\theta_t;{\xi}^i_{t})-\eta^i_{t,d}}{\lambda})\nabla \ell^i(\theta_t;{\xi}^i_t)-f^{*'}(\frac{\ell^i(\theta_t;{\xi}^i_t)-\eta^{i*}}{\lambda})\nabla \ell^i(\theta_t;{\xi}^i_t) \||\calH_t\big] \nonumber \\
        &\leq G\|\mathbb{E}_{{\xi}_{t}^i}[f^{*'}(\frac{\ell^i(\theta_t;{\xi}^i_t)-\eta^i_{t,d}}{\lambda}) - f^{*'}(\frac{\ell(\theta_t;{\xi}_t^i)-\eta^{i*}}{\lambda})|\calH_t]  \| \nonumber\\
        &=G\|\nabla_{\eta^i} \fL^i(\theta_t,\eta^i_{t,d})\|.\nonumber
    \end{align}
    where the first inequality utilizes Jensen's inequality; the second inequality utilizes Cauchy-schwarz inequality, and $G$-Lipschitz assumption of $\ell^i(\cdot)$ holding for $\xi_t^i$ and Cauchy-Schwarz inequality; the last equality utilizes the fact $f^*(\cdot)$ is a monotone non-decreasing convex function, where the sign of
    $|f^{*'}(\frac{\ell^i(\theta_t;{\xi}^i_t)-\eta^i_{t,d}}{\lambda}) - f^{*'}(\frac{\ell^i(\theta_t;{\xi}_t)-\eta^{i*}}{\lambda})|$ merely depends on relative position between $\eta^i_{t,d}$ and $\eta^{i*}$, which is homogeneous among different ${\xi}^i_t$. Thus $|\mathbb{E}_{{\xi}_t^i}[f^{*'}(\frac{\ell^i(\theta_t;{\xi}^i_t)-\eta^i_{t,d}}{\lambda}) - f^{*'}(\frac{\ell^i(\theta_t;{\xi}^i_t)-\eta^{i*}}{\lambda})|\calH_t]|=|\nabla_{\eta^i}\fL^i(\theta_t,\eta^i_{t,d})-\nabla_{\eta^i}\fL^i(\theta,\eta^{i,*})| = |\nabla_{\eta^i}\fL^i(\theta_t,\eta_{t,d}^i)|$. 
    
    Taking square on both sides, and taking conditional expectation over $\mathbb{E}_{\eta_{t,d}^i}$ we have
    \begin{align}
        &\mathbb{E}_{\eta_{t,d}^i}\Big[\|\mathbb{E}_{{\xi}^i_t}\big[f^{*'}(\frac{\ell^i(\theta_t;{\xi}_t^i)-\eta^i_{t,d}}{\lambda})\nabla \ell^i(\theta_t;{\xi}_t^i)-f^{*'}(\frac{\ell^i(\theta_t;{\xi}_t^i)-\eta^{i*}}{\lambda})\nabla\ell^i(\theta_t;{\xi}_t^i)|\calH_t\big] \|^2|\calH_t\Big]\nonumber\\
        &\leq G^2\mathbb{E}_{\eta^i_{t,d}}[\|\nabla_{\eta^i} \fL^i(\theta_t,\eta^i_{t,d})\|^2|\calH_t]\leq \tepsilon^2,\nonumber
    \end{align}
    where the last inequality applies $\mathbb{E}_{\eta^i_{t,d}}[\|\nabla_{\eta^i} \fL^i(\theta_t,\eta^i_{t,d})\|^2|\calH_t]\leq \tfrac{\tepsilon^2}{G^2}$.
    
    For multi-objective loss $\nabla_{\theta}\fL(\theta_{t},\eta_{t,d})$ and $\Phi(\theta_t)$, we have
    \begin{align}
        \mathbb{E}_{\eta_{t,d}}[\|\nabla_{\theta}\fL(\theta_t, \eta_{t,d}) - \nabla \Phi(\theta_t) \|_F^2|\calH_t] = \sum_{i=1}^{m} \mathbb{E}_{\eta^i_{t,d}}[\|\nabla_{\theta}\fL^i(\theta_t,\eta^i_{t,d}) - \nabla \phi^i(\theta_t) \|^2|\calH_t]\leq m\tepsilon^2, \nonumber
    \end{align}
which completes the proof.
\end{proof}
\begin{corollary}[Stochastic estimation error of $\nabla \Phi(\theta_t)$]\label{Corollary: estimation error of true gradient}
    For estimation error $E_{t,2}:=Y_{t}-\nabla\Phi(\theta_t)$, we have
    \begin{align}
        &\mathbb{E}_{\eta_{t,d},{\xi}_t}[\|E_{t,2}w_t\|^2|\calH_t]\leq C_B^2,\text{ and }\mathbb{E}_{\eta_{t,d},{\xi}_t}[\|E_{t,2}\|_F^2|\calH_t]\leq {m C_B^2}.
        \label{eq: bias bound},
    \end{align}
    where $C_B^2:= \sigma^2+2\tepsilon^2$, $\sigma^2 = {2K_0}+{2K_1G^{-2}}\tepsilon^2$. Same upper bound also holds for $E_{t,3}$, $E_{t,4}$. 

\end{corollary}
\begin{proof}
Notice that, for LHS of above inequality, we can upper bound as follows
\begin{align}
    &\mathbb{E}_{\eta_{t,d}^i,{\xi}_t}[\|\Y^i_t-\nabla \phi^i(\theta_t) \|^2|\calH_t]\nonumber\\
    &=\mathbb{E}_{\eta_{t,d},{\xi}^i_t}[\|\Y_t^i- \nabla_{\theta} \fL^i(\theta_t,\eta_{t,d}^i)+\nabla_{\theta} \fL^i(\theta,\eta_{t,d}^i)-\nabla\phi^i(\theta_t)\|^2|\calH_t] \nonumber\\
    &\leq {2}\mathbb{E}_{\eta_{t,d},{\xi}^i_t}[\|\Gamma_{t,2}^i \|^2|\calH_t] + 2\mathbb{E}_{\eta^i_{t,d}}[\|\nabla_{\theta}\fL^i(\theta_t, \eta_{t,d}^i)-\nabla \phi^i(\theta) \|^2|\calH_t]\nonumber\\
    &\leq  {2K_0}+{2K_1\mathbb{E}_{\eta_{t,d}^i}[\|\nabla_{\eta}\fL^i(\theta_t, \eta_{t,d}^i) \|^2}|\calH_t]+2\tepsilon^2 \nonumber\\
    &\leq {2K_0}+{2K_1G^{-2}}\tepsilon^2+2\tepsilon^2 \nonumber\\
    &= C_B^2, \nonumber
\end{align}
where the first inequality, we utilize $(a+b)^2\leq 2a^2+2b^2$, the second inequality utilizes variance bound  and \eqref{eq: approximation error bound }, respectively and the last inequality utilizes \eqref{eq: inner loop stationary} to upper bound $\mathbb{E}_{\eta_{t,d}^i}\|\nabla_{\eta}\fL^i(\theta_t, \eta_{t,d}^i) \|^2$ obtained from inner-loop convergence.

Similarly, for $\mathbb{E}_{\eta_{t,d},\hat{\xi}}[\|(\Y_t-\nabla \Phi(\theta_t))w_t\|^2|\calH_t]$, we conclude
\begin{align}
    \mathbb{E}_{\eta_{t,d},{\xi}_t}[\|(\Y_t-\nabla \Phi(\theta_t))w_t\|^2|\calH_t] &=\mathbb{E}_{\eta_{t,d},{\xi}_t}[\|\sum_{i=1}^{m}(Y_t^i - \nabla\phi^i(\theta_t))w_t^i\|^2|\calH_t]\nonumber\\
    &\leq \mathbb{E}_{\eta^i_{t,d},{\xi}_t^i}\big[\sum_{i=1}^{m}w_t^i \|Y_t^i-\nabla\phi^i(\theta_t)\|^2|\calH_t\big] \nonumber\\
    &= \sum_{i=1}^{m} w_t^i \mathbb{E}_{\eta_{t,d}^i,{\xi}_t^i}\big[\|Y_t^i - \nabla\phi^i(\theta) \|^2|\calH_t \big]\nonumber\\
    &\leq  C_B^2,\nonumber
\end{align}
where the inequality applies Jensen's inequality, and $w_t$ is $\calH_t$-measurable and independent with $\xi_t$ and $\eta_{t,d}$.
\end{proof}
\section{Convergence of Algorithm \ref{alg1}}\label{Appendix: proof of descent lemma algorithm 1}
\subsection{Formal Statement of Theorem \ref{thm: convergence of algorithm 1} and Proof}\label{Appendix: proof of convergence algorithm 1}
Given problem dependent parameter $m$,$G$, $L_0=G^2M\lambda^{-1}+L$, $L_2=M\lambda^{-1}$, $K_0=8G^2+10G^2M^2\lambda^{-2}\kappa^2$, $K_1=8G^2$, $K_2=M^2\lambda^{-2}\kappa^2$, denote $\Delta_{\theta_0} = \max_{i\in[m]}\{\phi^i(\theta_0)-\phi^{i,*}\}<\infty$ as finite initial outer-objective-gap.
Set the hyper-parameters in Algorithm \ref{alg1} as follows
\begingroup
\setlength{\jot}{5pt}
\begin{align}
\sigma^2
&= 2K_0 + 2K_1G^{-2}\min\{1,G^{-2}\}\epsilon^2
 = \mathcal{O}(K_0)
 = \mathcal{O}(G^2M^2\lambda^{-2}\kappa^2)
 \nonumber\\
C_B^2
&= \sigma^2 + 2\min\{1,G^{-2}\}\epsilon^2
 = \mathcal{O}(G^2M^2\lambda^{-2}\kappa^2)
 \nonumber\\
\gamma
&= \min\Bigl\{
    \frac{\epsilon^2}{2L_2K_2G^4},
    \frac{\epsilon^2}{2L_2K_2G^2},
    \frac{1}{L_2}
   \Bigr\}
 = \mathcal{O}\Bigl(
    \frac{\lambda^3\epsilon^2}{M^3G^4\kappa^2}
   \Bigr)
 \nonumber\\
\rho
&= \frac{1}{16}\epsilon^2
 = \mathcal{O}(\epsilon^2)
 \nonumber\\
\alpha
&= \min\Bigl\{
    \frac{1}{2L_0},
    \frac{\epsilon^2}{8L_0C_B^2}
   \Bigr\}
 = \mathcal{O}\Bigl(
    \frac{\lambda^3\epsilon^2}{G^4M^3\kappa^2}
   \Bigr)
 \nonumber\\
\beta
&= \min\Bigl\{
    \frac{\epsilon^2}{32(mG^4+2mG^2C_B^2+mC_B^4)},
    \frac{2}{4\rho}
   \Bigr\} = \mathcal{O}\Bigl(
    \frac{\epsilon^2}
         {m(G^2+G^2M^2\lambda^{-2}\kappa^2)^2}
   \Bigr)
 \nonumber\\
D
&= \overline{\Delta}_{\eta}
   \max\bigl\{
    4G^2\gamma^{-1}\epsilon^{-2},
    4G^4\gamma^{-1}\epsilon^{-2}
   \bigr\}
 = \Theta\bigl(
    \overline{\Delta}_{\eta}G^4\gamma^{-1}\epsilon^{-2}
   \bigr)
 \nonumber\\
T
&= \max\bigl\{
    8\beta^{-1}\epsilon^{-2},
    8\Delta_{\theta_0}\alpha^{-1}\epsilon^{-2}
   \bigr\} = \Theta(\max\bigl\{
    \Delta_{\theta_0}\alpha^{-1}\epsilon^{-2},
    \beta^{-1}\epsilon^{-2}
   \bigr\})
 = \Theta(\epsilon^{-4}).
\end{align}
\endgroup
where $\overline{\Delta}_{\eta}<\infty$ denotes an $\epsilon$-independent finite upper bound on the warm-start gaps over all $i\in[m]$ and $t\in[T]$.
Then, we then have the following convergence statement on Double-Loop Algorithm \ref{alg1}.

\begin{restatable}[Formal Statement of Theorem~\ref{thm: convergence of algorithm 1}]{theorem}{convegencefirstalg}\label{restatethm: convergence of algorithm 1}
    Let Assumption \ref{assumption 1} hold. Denote $\Delta_{\theta_0}$ as the finite initial outer-objective gap and 
    $\overline{\Delta}_{\eta}$ as $\epsilon$-independent upper bound on the warm-start gaps over all $i\in[m]$ and $t\in [T]$. For any $\epsilon>0$,
    set $\rho=\mathcal{O}(\epsilon^2)$, $\beta=\mathcal{O}(\tfrac{\epsilon^2}{m(G^2+G^2M^2\lambda^{-2}\kappa^2)^2})$,
    $\alpha = \mathcal{O}(\tfrac{\lambda^{3}\epsilon^2}{G^4M^3\kappa^2})$, and $\gamma = \mathcal{O}(\tfrac{\lambda^3\epsilon^2}{M^3G^4\kappa^2})$ for Algorithm \ref{alg1}. Then, after $T=\Theta(\max\big\{\Delta_{\theta_0}\alpha^{-1}\epsilon^{-2}, \beta^{-1}\epsilon^{-2}\big\})$ outer iterations, each associated with $D=\Theta(\overline{\Delta}_{\eta}G^4\gamma^{-1}\epsilon^{-2})$ inner-iterations, we have 
    \begin{align}
        \frac{1}{T}\sum_{t=0}^{T-1}\mathbb{E}[\|\nabla \Phi(\theta_t)w_t\|^2]\leq 94\epsilon^2.
    \end{align}
\end{restatable}
\begin{proof}
According to Lemma~\ref{lemma: gradient bound of phi}, we know $\|\nabla \phi^i(\theta_t)\|\leq G$. Put the update rule $\theta_{t+1} = \theta_t - \alpha Y_t{w}_t$ in the descent inequality of $L_0$-smooth function, we have
\begin{align}
    \Phi(\theta_{t+1})w&\leq \fL({\theta}_{t+1}, {\eta}_{\theta_t}^*){w} \leq \Phi(\theta_t){w} -\alpha\langle\nabla \Phi({\theta}_t){w},\Y_{t}{w}_t\rangle  +\frac{L_0}{2} \alpha^2 \|\Y_t{w}_t \|^2\nonumber\\
    &= \Phi(\theta_t){w} -\alpha\langle \nabla \Phi(\theta_t)w,\nabla \Phi(\theta_t) {w}_t\rangle -\alpha \langle \nabla \Phi(\theta_t)w,(\Y_{t}-\nabla\Phi(\theta_t))w_t \rangle  +\frac{L_0}{2} \alpha^2 \|\Y_{t}{w}_t \|^2 \nonumber\\
    &=\Phi(\theta_t){w} - \alpha\langle \nabla \Phi(\theta_t)w,\nabla \Phi(\theta_t) {w}_t\rangle - \alpha{\langle \nabla \Phi(\theta_t)w, (Y_t - \nabla_{\theta} \fL(\theta_t, \eta_{t,d}))w_t\rangle} \nonumber\\
    &\quad -\alpha{\langle \nabla \Phi(\theta_t)w, (\nabla_{\theta} \fL(\theta_t, \eta_{t,d})-\nabla \Phi(\theta_t))w_t\rangle }+\frac{L_0}{2}\alpha^2 \|Y_tw_t \|^2,
    \label{eq: intermediate descent lemma alg1}
\end{align}
where the first equality  $\Y_tw_t = (\Y_t-\nabla \Phi(\theta_t)+\Phi(\theta_t))w_t$ and the second inequality further decompose $Y_t - \nabla \Phi(\theta_t)= Y_t-\nabla_{\theta} \fL(\theta_t, \eta_{t,d})+\nabla_{\theta}\fL(\theta_t, \eta_{t,d})-\nabla \Phi(\theta_t)$.

Next, we upper bound the term $\langle \nabla \Phi({\theta}_t)w, \nabla \Phi({\theta}_t){w}_t \rangle$ from analyzing $\|w_{t+1}-w \|^2$.
\begin{align}
    \|{w}_{t+1} - {w} \|^2
    &= \| \Pi_{\mathcal{W}} \big({w}_t - \beta \big[\Yb_t^{\top}\Yt_t{w}_t+\rho {w}_t \big]\big) - {w} \|^2\leq \|{w}_t - \beta \big[\Yb_t^{\top}\Yt_t{w}_t+\rho {w}_t \big]- {w} \|^2 \nonumber\\
    &= \|{w}_t - {w} \|^2 -2\beta \langle {w}_t-{w}, \Yb_t^{\top}\Yt_t{w}_t+\rho {w}_t \rangle + \beta^2\|\Yb_t^{\top}\Yt_t{w}_t+\rho {w}_t\|^2,
\end{align}
where the first inequality is due to non-expansiveness of projection over probability simplex.

Re-arrange above inequality, we have
\begin{align}
    &2\beta \langle {w}_t-{w},\Yb_t^{\top}\Yt_t{w}_t \rangle \nonumber\\
    &\leq \big(\|{w}_t -{w}\|^2-\|{w}_{t+1} -{w} \|^2 \big) -2\beta\rho\langle {w}_t-{w}, {w}_t\rangle + \beta^2 \|\Yb_t^{\top}\Yt_t{w}_t+\rho {w}_t\|^2\nonumber\\
    &{\leq} \big(\|{w}_t -{w}\|^2-\|{w}_{t+1} -{w} \|^2 \big)+4\beta \rho + 2\beta^2\rho^2+ 2\beta^2\|\Yb_t^{\top}\Yt_t{w}_t \|^2,
    \label{eq: wtsquare}
\end{align}
where the first inequality applies $\langle w_t-w,w_t \rangle= \|w_t\|^2-\langle w, w_t\rangle\leq 2$, and $(a+b)^2\leq 2a^2+2b^2$. 

For $\Yb_t^{\top}\Yt_t$, one can establish the relationship as follows
\begin{align}
    \Yb_t^{\top}\Yt_t &=
    (\underbrace{\Yb_t-\nabla\Phi(\theta_t)}_{E_{t,3}}+\nabla \Phi(\theta_t))^{\top}(\underbrace{\Yt_t-\nabla\Phi(\theta_t)}_{E_{t,4}}+\nabla \Phi(\theta_t)) \nonumber \\
    & = E_{t,3}^{\top}E_{t,4} + E_{t,3}^{\top}\nabla\Phi(\theta_t)+\nabla \Phi(\theta_t)^{\top}E_{t,4}+\nabla \Phi(\theta_t)^{\top}\nabla\Phi(\theta_t),
    \label{eq: Y3Y4 reltion}
\end{align}
Put above equality into \eqref{eq: wtsquare}, we have 
\begin{align}
    &2\beta \langle {w}_t-{w},\Yb_t^{\top}\Yt_t{w}_t \rangle \nonumber\\
    & \overset{(i)}{\leq} \big(\|{w}_t -{w}\|^2-\|{w}_{t+1} -{w} \|^2 \big)+4\beta \rho + 2\beta^2\rho^2\nonumber \nonumber\\
    &\quad + 2\beta^2 \|E_{t,3}^{\top}E_{t,4}w_t + E_{t,3}^{\top}\nabla\Phi(\theta_t)w_t+\nabla \Phi(\theta_t)^{\top}E_{t,4}w_t+\nabla \Phi(\theta_t)^{\top}\nabla \Phi(\theta_t)w_t \|^2\nonumber\\
    &\overset{(ii)}{\leq} \big(\|{w}_t -{w}\|^2-\|{w}_{t+1} -{w} \|^2 \big)+4\beta \rho + 2\beta^2\rho^2\nonumber\\
    & \quad+8\beta^2 \|E_{t,3}^{\top}E_{t,4}{w}_t \|^2 + 8\beta^2\|\nabla \Phi(\theta_t)^{\top}\nabla \Phi(\theta_t){w}_t \|^2 \nonumber\\
    &\quad +8\beta^2\|E_{t,3}^{\top}\nabla \Phi(\theta_t){w}_t \|^2+ 8\beta^2\|\nabla \Phi(\theta_t)^{\top}E_{t,4}{w}_t \|^2,
\end{align}
where (i) utilizes inequality \eqref{eq: Y3Y4 reltion}; (ii) utilizes the fact $(a+b+c+d)^2\leq 4a^2+4b^2+4c^2+4d^2$.

For $\|\nabla \Phi(\theta_t)\|$ and $\|\nabla \Phi(\theta_t)w_t\|$, we have
$\|\nabla \Phi(\theta_t) \|_F\leq \sqrt{m}G$ and $\|\nabla \Phi(\theta_t)w_t \| \leq \sum_{i=1}^{m}w_t^i\| \nabla \phi^i(\theta_t)\|\leq G$ holds for $t\leq T$. Thus, for $\|\nabla\Phi(\theta_t)^{\top}\nabla \Phi(\theta_t){w}_t \|^2$, we can upper bound it as follows
\begin{align}
    &\|\nabla\Phi(\theta_t)^{\top}\nabla \Phi(\theta_t){w}_t \|^2\overset{(i)}{\leq} \|\nabla\Phi(\theta_t)\|_F^2\|\nabla\Phi(\theta_t)w_t \|^2\leq {m G^4},
\end{align}
where (i) applies Cauchy-Schwarz inequality, sub-multiplicative property of Frobenius norm.

Similarly, for $ \|E_{t,3}^{\top}\nabla \Phi(\theta_t){w}_t \|^2$,$\|E_{t,4}^{\top}\nabla \Phi(\theta_t){w}_t \|^2$, we have
\begin{align}
    \| E_{t,3}^{\top}\nabla \Phi(\theta_t){w}_t \|^2
    \leq \|\nabla\Phi(\theta_t){w}_t\|^2 \|E_3 \|_F^2\leq {G^2 \|E_{t,3} \|_F^2}.\nonumber
\end{align}
Same arguments applies to $\|\nabla\Phi(\theta_t)^{\top}E_{t,4}{w}_t \|^2\leq mG^2 \|E_{t,4}{w}_t\|^2$.

Combine above inequalities, we have
\begin{align}
    &2\beta \langle {w}_t-{w},\Yb_t^{\top}\Yt_t{w}_t \rangle\nonumber\\
    &\leq \big(\|{w}_t -{w}\|^2-\|{w}_{t+1} -{w} \|^2 \big)+4\beta \rho + 2\beta^2\rho^2+8\beta^2 \|E_{t,3}^{\top}E_{t,4}{w}_t\|^2\nonumber\\
    &\quad + 8\beta^2 mG^4 + 8\beta^2 G^2(\|E_{t,3}\|_F^2 +m\|E_{t,4}w_t\|_F^2 ).
\end{align}
Put equation $\Yb_t^{\top}\Yt_t=(\nabla\Phi({\theta}_t)+E_{t,3})^{\top}(\nabla \Phi(\theta_t)+E_{t,4})$ in LHS and re-arrange above inequality, we have
\begin{align}
    &2\beta\langle {w}_t-{w},\nabla \Phi({\theta}_t)^{\top}\nabla \Phi({\theta}_t){w}_t \rangle \nonumber\\
    &\leq \big(\|{w}_t -{w}\|^2-\|{w}_{t+1} -{w} \|^2 \big)+4\beta \rho + 2\beta^2\rho^2+8\beta^2 \|E_{t,3}^{\top} E_{t,4}{w}_t\|^2\nonumber\\
    & \quad -2\beta \langle w_t-w, E_{t,3}^{\top}E_{t,4}{w}_t+\nabla \Phi({\theta}_t)^{\top}E_{t,4}w_t+E_{t,3}^{\top}\nabla \Phi({\theta}_t)w_t\rangle\nonumber\\
    & \quad + 8\beta^2 mG^4 + 8\beta^2G^2(\|E_{t,3}\|_F^2 +m\|E_{t,4}w_t\|_2^2 ).
    \label{eq: intermediate step on wt}
\end{align}
Combine   \eqref{eq: intermediate descent lemma alg1} and \eqref{eq: intermediate step on wt}, we have
\begin{align}
    &\Phi(\theta_{t+1})w - \Phi(\theta_t)w\nonumber\\
    &\leq \fL({\theta}_{t+1},{\eta}_{\theta_t}^*){w} - \Phi(\theta_t){w} \nonumber\\
    &\overset{(i)}{\leq} -\alpha \|\nabla \Phi(\theta_t){w}_t \|^2 -\alpha {\langle \nabla\Phi(\theta_t){w},(\Y_t - \nabla_{\theta} \fL(\theta_t,\eta_{t,d})) {w}_t\rangle} -\alpha {\langle \nabla\Phi(\theta_t)w, (\nabla_{\theta} \fL(\theta_t, \eta_{t,d})-\nabla\Phi(\theta_t))w_t\rangle}\nonumber\\
    & \quad + \alpha^2L_0\|\nabla \Phi({\theta}_t){w}_t \|^2 + \alpha^2L_0\|E_{t,2}{w}_t\|^2 \nonumber\\
    &\quad +\frac{\alpha}{2\beta}\Big(\big(\|{w}_t -{w}\|^2-\|{w}_{t+1} -{w} \|^2 \big)+4\beta \rho + 2\beta^2\rho^2+8\beta^2 \|E_3^{\top} E_4{w}_t\|^2\nonumber\\
    & \quad -2\beta \langle {w}_t-{w}, E_{t,3}^{\top}E_{t,4}{w}_t+\nabla \Phi({\theta}_t)^{\top}E_{t,4}w_t+E_{t,3}^{\top}\nabla \Phi({\theta}_t){w}_t\rangle\nonumber\\
    &\quad + 8\beta^2 mG^4 + 8\beta^2G^2(\|E_{t,3} \|_F^2 +m\|E_{t,4} w_t\|^2 ) \Big),
    \label{eq: one step descent inequality}
\end{align}
where (i) utilizes $(a+b)^2\leq 2a^2+2b^2$ to further decompose $\frac{\alpha^2 L_0}{2}\|Y_tw_t \|^2=\frac{\alpha^2L_0}{2}\|\nabla\Phi(\theta_t)w_t+E_{t,2}w_t \|^2$, and upper bound this term by $\alpha^2L_0 \|\nabla \Phi(\theta_t)w_t\|^2+\alpha^2L_0\|E_{t,2}w_t\|^2$ and replace $\langle \nabla \Phi(\theta_t)w, \nabla \Phi(\theta_t)w_t\rangle$ by \eqref{eq: intermediate step on wt}.

Sum above equations from $0$ to $T-1$, taking expectation over all randomness including $\xi_t, \bar{\xi}_t, \tilde{\xi}_t,\theta_t,\eta_{t,d},\eta_{t,\bar{d}},\eta_{t,\tilde{d}}, w_t$. For ${\alpha\leq \frac{1}{2L_0}}$, we upper bound $-\alpha \|\nabla\Phi(\theta_t)w_t\|^2+\alpha^2 L_0\|\nabla\Phi(\theta_t)w_t\|^2$ by $-\frac{\alpha}{2}\|\nabla \Phi(\theta_t)w_t\|^2$ and have following inequality
\begin{align}
    &\mathbb{E}[\Phi(\theta_{T}){w}] - \mathbb{E}[\Phi({\theta}_0){w}] \nonumber\\
    &\leq -\frac{\alpha}{2} \mathbb{E}\big[\sum_{t=0}^{T-1}\|\nabla \Phi({\theta}_t){w}_t \|^2\big] -\alpha {\mathbb{E}\big[ \sum_{t=0}^{T-1}\langle \nabla \Phi({\theta}_t){w},\underbrace{(Y_{t}-\nabla_{\theta}\fL(\theta_t, \eta_{t,d}))}_{\Gamma_{t,2}}{w_t}\rangle \big]}\nonumber\\
    &\quad -\alpha {\mathbb{E}\big[\sum_{t=0}^{T-1}\langle \nabla \Phi(\theta_t)w, \underbrace{(\nabla \fL(\theta_t,\eta_{t,d}) - \nabla \Phi(\theta_t))}_{A_{t,2}} w_t \big]} \nonumber\\
    &\quad -\alpha { \mathbb{E}\big[\sum_{t=0}^{T-1} \langle w_t - w, E_{t,3}^{\top} E_{t,4}w_t \rangle+\langle w_t -w, \nabla \Phi(\theta_t)^{\top}E_{t,4}w_t\rangle+\langle w_t -w, E_{t,3}^{\top}\nabla \Phi(\theta_t)w_t \rangle \big]} \nonumber\\
    &\quad +\frac{\alpha}{2\beta}\big(\|w_{0}-w \|^2-\|w_{\tau}-w\|^2\big)+ 2\alpha \rho T + \alpha\beta\rho^2T + 4\alpha\beta m G^4 T+\alpha^2 L_0 \mathbb{E}\big[ \sum_{t=0}^{T-1} \|E_{t,2}w_t\|^2\big] \nonumber\\
    &\quad +4\alpha\beta G^2 \mathbb{E}\big [\sum_{t=0}^{T-1} \|E_{t,3} \|_F^2+m\|{E}_{t,4}w_t\|^2 \big] + 4\alpha\beta \mathbb{E}\big[\sum_{t=0}^{T-1} \|E_{t,3}^{\top}E_{t,4}w_t\|^2 \big] \nonumber\\
    &\leq -\frac{\alpha}{2} \mathbb{E}\big[\sum_{t=0}^{T-1}\|\nabla \Phi({\theta}_t){w}_t \|^2\big] -\alpha {\mathbb{E}\big[ \sum_{t=0}^{T-1}\langle \nabla \Phi({\theta}_t){w},\Gamma_{t,2}{w_t}\rangle \big]}-\alpha {\mathbb{E}\big[\sum_{t=0}^{T-1}\langle \nabla \Phi(\theta_t)w, A_{t,2} w_t \rangle\big]} \nonumber\\
    &\quad +\alpha {\mathbb{E}\big[\sum_{t=0}^{T-1} \langle w-w_t, \Gamma_{t,3}^{\top} (\Gamma_{t,4}+A_{t,4})w_t \rangle +\langle w-w_t, A_{t,3}^{\top} (\Gamma_{t,4}+ A_{t,4})w_t \rangle \big]} \nonumber\\
    &\quad +\alpha{\mathbb{E}\big[\sum_{t=0}^{T-1}\langle w -w_t, \nabla \Phi(\theta_t)^{\top}\Gamma_{t,4}w_t\rangle + \langle w -w_t, \nabla \Phi(\theta_t)^{\top}A_{t,4}w_t\rangle
    \big]} \nonumber\\
    &\quad +\alpha{\mathbb{E}\big[\sum_{t=0}^{T-1}\langle w -w_t, \Gamma_{t,3}^{\top}\nabla \Phi(\theta_t)w_t \rangle + \langle w -w_t, A_{t,3}^{\top}\nabla \Phi(\theta_t)w_t \rangle\big]} \nonumber\\
    &\quad +\frac{\alpha}{\beta}+ 2\alpha \rho T + \alpha\beta\rho^2T + 4\alpha\beta m G^4 T+\alpha^2 L_0 \mathbb{E}\big[ \sum_{t=0}^{T-1} \|E_{t,2}w_t\|^2\big] \nonumber\\
    &\quad +4\alpha\beta G^2 \mathbb{E}\big [\sum_{t=0}^{T-1} \|E_{t,3} \|_F^2+m\|{E}_{t,4}w_t\|^2 \big] + 4\alpha\beta \mathbb{E}\big[\sum_{t=0}^{T-1} \|E_{t,3}^{\top}E_{t,4}w_t\|^2 \big],
    \label{eq: intermediate descent inequality 2}
\end{align}
where we decompose error $E_{t,j}= \Gamma_{t,j}+A_{t,j}$ for $j\in\{2,3,4 \}$ and rewrite $E_{t,3}^{\top}E_{t,4}=(\Gamma_{t,3}+A_{t,3})^{\top}(\Gamma_{t,4}+A_{t,4})=\Gamma_{t,3}^{\top}(\Gamma_{t,4}+A_{t,4})+A_{t,3}^{\top}(\Gamma_{t,4}+A_{t,4})$, and utilizing the fact $\| w_0-w_{\tau}\|^2\leq 2$ holds for any $w_0, w_{\tau}$.

We first bound $-\alpha \mathbb{E}\big[\sum_{t=0}^{T-1}\langle \nabla \Phi(\theta_t)w, \Gamma_{t,2}w_t\rangle \big]$. 
Recall, $\calG_t$ denotes $\sigma$-algebra collecting outer loop randomness up to $t-1$ and inner-loop randomness at current iteration $t$. Conditioned on $\calG_t$, randomness of $\Gamma_{t,2} = Y_{t,2}-\nabla_{\theta}\fL(\theta_t,\eta_{t,d})$ only depends on $\xi_{t}$, we have
\begin{align}
    \mathbb{E}[\langle \nabla \Phi(\theta_t)w, \Gamma_{2}w_t\rangle|\calG_t] = \langle \nabla \Phi(\theta_t)w,  \mathbb{E}[\Gamma_{t,2}|\calG_t]w_t \rangle =0.
    \label{eq: inner product bound 1}
\end{align}
Taking expectation over $\calG_t$, tower property further implies $\mathbb{E}[\langle \nabla \Phi(\theta_t)w, \Gamma_{2}w_t\rangle]=0$. Thus we have $ -\alpha\mathbb{E}\big[\sum_{t=0}^{T-1}\langle \nabla \Phi(\theta_t)w, \Gamma_{t,2}w_t\rangle \big]=0$.
Similar arguments also apply on $\alpha\mathbb{E}\big[\sum_{t=0}^{T-1}\langle w -w_t, \Gamma_{t,3}^{\top}\nabla \Phi(\theta_t)w_t \rangle\big]$ and $\alpha\mathbb{E}\big[\sum_{t=0}^{T-1}\langle w -w_t, \Gamma_{t,4}^{\top}\nabla \Phi(\theta_t)w_t \rangle\big]$. 

For $\alpha\mathbb{E}[\sum_{t=0}^{T-1} \langle w-w_t, \Gamma_{t,3}^{\top} (\Gamma_{t,4}+A_{t,4})w_t \rangle] $, we decompose it as
\begin{align}
    &\alpha \mathbb{E}\big[\sum_{t=0}^{T-1} \langle w-w_t, \Gamma_{t,3}^{\top} (\Gamma_{t,4}+A_{t,4})w_t \rangle \big] \nonumber\\
    &\qquad = \alpha \mathbb{E}\big [\sum_{t=0}^{T-1}\langle w-w_t, \Gamma_{t,3}^{\top}\Gamma_{t,4}w_t\rangle \big] + \alpha \mathbb{E}\big[\sum_{t=0}^{T-1}\langle w-w_t, \Gamma_{t,3}^{\top} A_{t,4}w_t \rangle \big] \nonumber.
\end{align}
For $\alpha\sum_{t=0}^{T-1}\langle w-w_t, \Gamma_{t,3}^{\top}\Gamma_{t,4}w_t \rangle$, conditioned on $\calG_t$, $\Gamma_{t,3}$ and $\Gamma_{t,4}$ shares independent randomness. Thus, we have
\begin{align}
    \mathbb{E}[\langle w-w_t, \Gamma_{t,3}^{\top}\Gamma_{t,4}w_t \rangle|\calG_t] = \mathbb{E}[ \langle w-w_t, \mathbb{E}[\Gamma_{t,3}|\calG_t]^{\top}\mathbb{E}[\Gamma_{t,4}|\calG_t]w_t \rangle ] = 0.
    \label{eq: inner product bound 2}
\end{align}
Taking expectation over $\calG_t$, applying tower property implies $\mathbb{E}[\langle w-w_t, \Gamma_{t,3}^{\top}\Gamma_{t,4}w_t \rangle] = 0$. Thus $\alpha\mathbb{E}[\sum_{t=0}^{T-1}\langle w-w_t, \Gamma_{t,3}\Gamma_{t,4}w_t\rangle]=0$

Similarly, for $\alpha \sum_{t=0}^{T-1}\langle w-w_t, \Gamma_{t,3}^{\top} A_{t,4}w_t \rangle$, conditioned on $\calG_t$, $\Gamma_{t,3}$ randomness depends on sample $\bar{\xi}_t$, whereas $A_{t,4}$ randomness conditioned on $\calG_t$ is fixed. Thus, we have
\begin{align}
    \mathbb{E}[\langle w-w_t, \Gamma_{t,3}^{\top} A_{t,4}w_t \rangle|\calG_t]= \langle w-w_t, \mathbb{E}[\Gamma_{t,3}^{\top}|\calG_t] A_{t,4}w_t \rangle = 0.
    \label{eq: inner product bound 3}
\end{align}
Taking expectation over $\calG_t$, tower property further implies $\alpha \mathbb{E}[\sum_{t=0}^{T-1}\langle w-w_t, \Gamma_{t,3}^{\top} A_{t,4}w_t \rangle] = 0$.

For $\alpha \mathbb{E}\big[\sum_{t=0}^{T-1}\langle w-w_t, A_{t,3}^{\top} (\Gamma_{t,4}+A_{t,4}) w_t \rangle \big]$, we decompose it as follows
\begin{align}
    &\alpha \mathbb{E}[\sum_{t=0}^{T-1}\langle w-w_t, A_{t,3}^{\top} (\Gamma_{t,4}+A_{t,4}) w_t \rangle] \nonumber\\
    &\quad = \alpha \mathbb{E}\big[\sum_{t=0}^{T-1}\langle w-w_t, A_{t,3}^{\top} \Gamma_{t,4} w_t \rangle\big] + \alpha \mathbb{E}\big[\sum_{t=0}^{T-1}\langle w-w_t, A_{t,3}^{\top}A_{t,4}w_t \rangle \big]. \nonumber
\end{align}
Similarly, conditioned on $\calG_t$, we have 
\begin{align}
    \mathbb{E}[\langle w-w_t, A_{t,3}^{\top}\Gamma_{t,4} w_t \rangle |\calG_t] = \langle w-w_t, A_{t,3}^{\top}\mathbb{E}[\Gamma_{t,4}|\calG_t] w_t \rangle = 0.
    \label{eq: inner product bound 4}
\end{align}
Taking expectation over $\calG_t$, tower property implies $ \alpha \mathbb{E}[\sum_{t=0}^{T-1}\langle w-w_t, A_{t,3}^{\top} \Gamma_{t,4} w_t \rangle] = 0$.

For $\alpha \sum_{t=0}^{T-1}\langle w-w_t, A_{t,3}^{\top} A_{t,4} w_t \rangle$, conditioned on $\calH_t$, we have
\begin{align}
    \alpha\sum_{t=0}^{T-1}  \mathbb{E}[\langle w-w_t, A_{t,3}^{\top} A_{t,4} w_t \rangle|\calH_t] &= \sum_{t=0}^{T-1} \alpha \mathbb{E}[\langle A_{t,3}(w-w_t), A_{t,4} w_t \rangle|\calH_t] \nonumber\\
    &\leq \alpha \sum_{t=0}^{T-1}\sqrt{\mathbb{E}[\|A_{t,3}(w-w_t)\|^2|\calH_t]}\sqrt{\mathbb{E}[\|A_{t,4}w_t\|^2|\calH_t]}\nonumber\\
    &\leq  2\alpha T \tepsilon^2,
    \label{eq: inner product bound 5}
\end{align}
where the first inequality utilizes conditional Cauchy–Schwarz inequality; the second inequality utilizes Jensen's inequality; the last inequality utilizes Corollary~\ref{Corollary: estimation bias of inner-loop}. Taking expectation over $\calH_t$, tower property of expectation gives $\alpha \mathbb{E}[\sum_{t=0}^{T-1}\langle w-w_t, A_{t,3}^{\top} A_{t,4} w_t \rangle]\leq 2\alpha T\tepsilon^2$.

For 
$$\alpha\mathbb{E}[\sum_{t=0}^{T-1}\langle w-w_t, A_{t,3}^{\top}\nabla \Phi(\theta_t)w_t  \rangle]+\alpha \mathbb{E}[\sum_{t=0}^{T-1}\langle w-w_t, \nabla \Phi(\theta_t)^{\top}A_{t,4}w_t \rangle] -\alpha \mathbb{E}\big[\sum_{t=0}^{T-1}\langle \nabla \Phi(\theta_t)w, A_{t,2}w_t \rangle \big],$$ conditioned on $\sigma$-algebra of outer iterations up to $\calH_t$, $A_{t,2}$, $A_{t,3}$ and $A_{t,4}$'s randomness is dependent with $\eta_{t,d}, \eta_{t, \bar{d}}, \eta_{t,\tilde{d}}$ respectively.  Notice that Conditional on $\calH_t$, the estimates $\eta_{t,d}$, $\eta_{t,\bar d}$, and $\eta_{t,\tilde d}$ are independent with each other and generated by the same inner-loop algorithm, using the same $\calH_t$-measurable initialization, learning rate, number of iterations. Each estimate is selected uniformly random. Consequently, these estimates have the same conditional distribution given $\calH_t$. Since $A_{t,2}$, $A_{t,3}$, and $A_{t,4}$ are obtained by applying the same bias mapping to their respective estimates, they also have the same conditional distribution and the same conditional mean. We therefore denote $\mathcal{A}_t=\mathbb{E}_{\eta_{t,d}}[A_{t,2}|\calH_t] = \mathbb{E}_{\eta_{t,\bar{d}}}[A_{t,3}|\calH_t] = \mathbb{E}_{\eta_{t,\tilde{d}}}[A_{t,4}|\calH_t]$. Thus, we further have
\begin{align}
    &\mathbb{E}\big[ \langle w-w_t, A_{t,3}^{\top}\nabla \Phi(\theta_t)w_t  \rangle+ \langle w-w_t, \nabla \Phi(\theta_t)^{\top}A_{t,4}w_t  \rangle \nonumber\\ 
    &\quad - \langle \nabla \Phi(\theta_t)w, A_{t,2}w_t \rangle |\calH_t\big]\nonumber\\
    & = \langle w-w_t, \mathbb{E}[A_{t,3}^{\top}|\calH_t]\nabla \Phi(\theta_t)w_t  \rangle + \langle w-w_t, \nabla \Phi(\theta_t)^{\top} \mathbb{E}[A_{t,4}|\calH_t]w_t \rangle \nonumber \\
    &\quad -\langle\nabla \Phi(\theta_t)w,\mathbb{E}[A_{t,2}|\calH_t]w_t \rangle \nonumber\\
    & = \langle w-w_t, \mathcal{A}_t^{\top}\nabla \Phi(\theta_t)w_t  \rangle + \langle w-w_t, \nabla \Phi(\theta_t)^{\top} \mathcal{A}_tw_t \rangle -\langle\nabla \Phi(\theta_t)w,\mathcal{A}_tw_t \rangle \nonumber\\
    & = \langle \mathcal{A}_t(w-2w_t), \nabla \Phi(\theta_t)w_t \rangle\nonumber\\
    & \leq \frac{1}{4}\|\nabla \Phi(\theta_t)w_t \|^2 + \|\mathcal{A}_t(w-2w_t)\|^2\nonumber\\
    & \leq  \frac{1}{4}\|\nabla \Phi(\theta_t)w_t \|^2+(\|\mathcal{A}_tw\| + 2\|\mathcal{A}_tw_t\|)^2\nonumber\\
    & \leq \frac{1}{4}\|\nabla \Phi(\theta_t)w_t \|^2 + 9\tepsilon^2,
    \label{eq: key insights of algorithm 1}
\end{align}
where the first inequality applies young's inequality, $\langle a,b\rangle\leq \frac{1}{4}\| a\|^2+\|b\|^2$; the second inequality utilizes triangle-inequality $\|a-b\|\leq \|a\|+\|b\|$; the last inequality applies Corollary~\ref{Corollary: estimation bias of inner-loop}. This further implies
\begin{align}
    &\alpha\mathbb{E}[\sum_{t=0}^{T-1}\langle w-w_t, A_{t,3}^{\top}\nabla \Phi(\theta_t)w_t  \rangle]+\alpha \mathbb{E}[\sum_{t=0}^{T-1}\langle w-w_t, \nabla \Phi(\theta_t)^{\top}A_{t,4}w_t \rangle] -\alpha \mathbb{E}\big[\sum_{t=0}^{T-1}\langle \nabla \Phi(\theta_t)w, A_{t,2}w_t \rangle \big] \nonumber\\
    &\leq \frac{\alpha}{4}\sum_{t=0}^{T-1}\mathbb{E}[\|\nabla \Phi(\theta_t)w_t\|^2] + 9\alpha T \tepsilon^2.
    \label{eq: inner product bound 6}
\end{align}
Since $w_t$ is $\calH_t$-measurable, $E_{t, 3}, E_{t, 4}$ are conditionally independent given $\mathcal{H}_t$, thus we have
\begin{align}
\mathbb{E}[\|E_{t, 3}\|_F^2\|E_{t, 4} w_t\|^2]  &=\mathbb{E}[\mathbb{E}[\|E_{t, 3}\|_F^2\|E_{t, 4} w_t\|^2 \mid \mathcal{H}_t]] \nonumber\\
& =\mathbb{E}[\mathbb{E}[\|E_{t, 3}\|_F^2 \mid \mathcal{H}_t] \mathbb{E}[\left\|E_{t, 4} w_t\right\|^2 \mid \mathcal{H}_t]] \leq m C_B^4.
\label{eq: quadratic bound 7}
\end{align}
Combining \eqref{eq: inner product bound 1},\eqref{eq: inner product bound 2},\eqref{eq: inner product bound 3},\eqref{eq: inner product bound 4},\eqref{eq: inner product bound 5},\eqref{eq: inner product bound 6},\eqref{eq: quadratic bound 7} we further have
\begin{align}
    &\mathbb{E}[\Phi(\theta_{T}){w}] - \mathbb{E}[\Phi({\theta}_0){w}] \nonumber\\
    &\leq -\frac{\alpha}{4} \mathbb{E}\big[\sum_{t=0}^{T-1}\|\nabla \Phi({\theta}_t){w}_t \|^2\big] +{11\alpha T \tepsilon^2} +\frac{\alpha}{\beta}+ 2\alpha \rho T + \alpha\beta\rho^2T + 4\alpha\beta mG^4 T\nonumber\\
    &+\alpha^2 L_0 \mathbb{E}\big[ \sum_{t=0}^{T-1} \|E_{t,2}w_t\|^2\big] +4\alpha\beta G^2 \mathbb{E}\big [\sum_{t=0}^{T-1} \|E_{t,3} \|_F^2+m\|{E}_{t,4}w_t\|_2^2 \big] + 4\alpha\beta \mathbb{E}\big[\sum_{t=0}^{T-1} \|E_{t,3}^{\top}E_{t,4}w_t\|^2 \big] \nonumber\\
    &\leq -\frac{\alpha}{4} \mathbb{E}\big[\sum_{t=0}^{T-1}\|\nabla \Phi({\theta}_t){w}_t \|^2\big] +{11\alpha T \tepsilon^2} +\frac{\alpha}{\beta}+ 2\alpha \rho T + \alpha\beta\rho^2T + 4\alpha\beta m G^4 T\nonumber\\
    &+\alpha^2 L_0 \mathbb{E}\big[ \sum_{t=0}^{T-1} \|E_{t,2}w_t\|^2\big] +4\alpha\beta  G^2 \mathbb{E}\big [\sum_{t=0}^{T-1} \|E_{t,3} \|_F^2+m\|E_{t,4}w_t\|_2^2 \big] + 4\alpha\beta \mathbb{E}\big[\sum_{t=0}^{T-1} \|E_{t,3}\|^2_F \|E_{t,4}w_t\|^2 \big] \nonumber\\
    &\leq -\frac{\alpha}{4} \mathbb{E}\big[\sum_{t=0}^{T-1}\|\nabla \Phi({\theta}_t){w}_t \|^2\big] +\frac{\alpha}{\beta}+{11\alpha T \tepsilon^2} + 2\alpha \rho T + \alpha\beta\rho^2T + 4\alpha\beta m G^4 T\nonumber\\
    &+\alpha^2 L_0 C_B^2T+8\alpha\beta m G^2C_B^2T + 4\alpha\beta mT C_B^4 
\end{align}
where the second and third inequality applies sub-multiplicative property of Frobenius-norm and Corollary~\ref{Corollary: estimation error of true gradient}. 
Re-organize above inequality, divide both sides with $\frac{4}{\alpha T}$, we have
\begin{align}
    &\frac{1}{T}\sum_{t=0}^{T-1}\mathbb{E}[\|\nabla \Phi(\theta_t)w_t \|^2] \nonumber\\
    &\leq \frac{4}{\alpha}\cdot \frac{2}{T}\Big(\mathbb{E}[\Phi(\theta_0)w-\Phi^*w]+\frac{\alpha}{\beta}+{11\alpha T \tepsilon^2}+{2\alpha \rho T}+ \alpha\beta\rho^2T+4\alpha\beta m G^4 T\nonumber\\
    &\quad+ {\alpha^2 L_0 C_B^2T}+{8\alpha\beta mG^2C_B^2T + 4\alpha\beta m T C_B^4}\Big) \nonumber\\
    &= \frac{8(\mathbb{E}[\Phi(\theta_0)w-\Phi^*w])}{\alpha T}+\frac{8}{\beta T}+88\tepsilon^2
    +16\rho+8\beta\rho^2+32\beta m G^4+8\alpha L_0C_B^2+64\beta m G^2C_B^2\nonumber\\
    &\quad +32\beta m C_B^4\nonumber\\
    &\leq 94\epsilon^2,
\end{align}
where the last inequality is due to our parameter choices.
\end{proof}

\section{Reformulation of Pareto-stationary condition}
\subsection{Optimality condition Reformulation}\label{Appendix: proof of reformulated optimality condition}
\reformedcondition*
\begin{proof}
Expanding $\| \nabla \Phi({\theta})w \|$, we have
\begin{align}
    &\|\nabla \Phi(\theta)w \| \nonumber\\
    &\overset{(i)}{\leq} 
    \|\nabla \Phi(\theta)w - \nabla_{\theta}\fL(\theta,\eta)w\| + \| \nabla_{\theta} \fL(\theta,\eta)w\|
    \nonumber\\
    &=\big\|\sum_{i=1}^{m} w^i\big(\mathbb{E}_{\xi}[({f^*}'(\tfrac{\ell^i({\theta};\xi)-\eta^i}{\lambda})-{f^*}'(\tfrac{\ell^i({\theta};\xi)-\eta^{i,*}}{\lambda}))\nabla \ell^i({\theta},\xi)]\big)\big \| + \big\|\sum_{i=1}^{m}w^i\nabla_{{\theta}}\fL^i({\theta},\eta^i) \big\| \nonumber \\
    &\overset{(ii)}{\leq} \sum_{i=1}^{m}w^i\big\|\mathbb{E}_{\xi}\big[((f^*)'(\tfrac{\ell^i({\theta};\xi)-\eta^i}{\lambda})-(f^*)'(\tfrac{\ell^i({\theta};\xi)-\eta^{i,*}}{\lambda}))\nabla \ell^i({\theta},\xi)\big] \big\| + \big\| \sum_{i=1}^{m} w^i \nabla_{{\theta}}\fL^i({\theta},\eta^i)  \big\|\nonumber\\
    &\overset{(iii)}{\leq} \sum_{i=1}^{m} w^i\mathbb{E}_{\xi}\big[ \big\| ((f^*)'(\tfrac{\ell^i({\theta};\xi)-\eta^i}{\lambda}) - (f^*)'(\tfrac{\ell^i({\theta};\xi)-\eta^{i,*}}{\lambda})) \nabla \ell^i({\theta},\xi) \big\|\big] +\big\|\sum_{i=1}^{m}w^i\nabla_{{\theta}} \fL^i({\theta}, \eta^i) \big\|\nonumber \\
    & \overset{(iv)}{\leq} \sum_{i=1}^{m} w^i \mathbb{E}_{\xi}\big[\| (f^*)'(\tfrac{\ell^i({\theta};\xi)-\eta^i}{\lambda}) - (f^*)'(\tfrac{\ell^i({\theta};\xi)-\eta^{i,*}}{\lambda}) \| \|\nabla \ell^{i}(\theta,\xi)\| \big]+ \big \|\sum_{i=1}^{m}w^i\nabla_{{\theta}} \fL^i({\theta}, \eta^i) \big \|\nonumber \\
    & \overset{(v)}{\leq} G \sum_{i=1}^{m} w^i \mathbb{E}_{\xi}|(f^*)'(\tfrac{\ell^i({\theta}, \xi)-\eta^i}{\lambda}) - (f^*)'(\tfrac{\ell^i({\theta}, \xi)-\eta^{i,*}}{\lambda}) |+ \|\sum_{i=1}^{m}w^i\nabla_{{\theta}} \fL^i({\theta},\eta^i) \|\nonumber \\
    & = G \sum_{i=1}^{m} w^i \mathbb{E}_{\xi}\|\nabla_{\eta^i}\fL^{i}({\theta}, \eta^i;\xi) - \nabla_{\eta^i}\fL^{i}({\theta}, \eta^{i,*};\xi) \|+\| \sum_{i=1}^{m} w^i \nabla_{{\theta}}\fL^i({\theta}, \eta_i)\| \nonumber \\
    & = G \sum_{i=1}^{m} w^i \big|\mathbb{E}_{\xi}\big[\nabla_{\eta_i} \fL^i({\theta},\eta^i;\xi)- \nabla_{\eta^i}\fL^i({\theta}, \eta^{i,*};\xi) \big] \big| + \|\sum_{i=1}^{m}w^i \nabla_{{\theta}} \fL^i({\theta}, \eta^i) \| \nonumber\\
    &= G  \big|\nabla_{\eta} \fL({\theta}, \eta) w\big|+ \|  \nabla_{{\theta}}\fL({\theta},\eta)w  \|,
\end{align}
where (i) and (ii) applies triangle inequality, $\|a+b \|\leq \| a\|+\| b\|$; (iii) applies Jensen's inequality; (iv) applies Cauchy-Schwarz inequality; (v) applies $G$-Lipschitz continuity assumption for each $\ell^i(\cdot)$. Notice $(f^*)'$ is monotone non-decreasing, the sign of $\nabla_{\eta^i}\fL(\theta,\eta^i;\xi)-\nabla_{\eta^i}\fL(\theta,\eta^{i,*};\xi)$ depends on relative position between $\eta^i$ and $\eta^{i,*}$, this enables to put $\mathbb{E}_{\xi}$ into $|\cdot|$ without increasing its value.

Next we show condition \eqref{eq: epsilon Optimal condition} can be achieved by optimizing rescaled function $\fhL(\theta, \eta)=\fL({\theta}, G{\sqrt{m}}{\eta})$ such that $\|\nabla \fhL(\theta,\eta)w\|\leq \frac{\epsilon}{\sqrt{2}}$. Expanding $\|\nabla_{{\theta}, {\eta}}\fL({\theta}, G{\sqrt{m}}{\eta})w \|^2$, we have
\begin{align}
    \|\nabla_{{\theta}, {\eta}}\fhL({\theta}, {\eta}){w}\|^2=
    &\|\nabla_{{\theta}, {\eta}}\fL({\theta}, G{\sqrt{m}}{\eta}){w} \|^2\nonumber \\
    =&\Big\| \begin{bmatrix}
        \nabla_{{\theta}}\fL({\theta}, G\sqrt{m} {\eta})\\
        G\sqrt{m}\cdot\nabla_{{\eta}} \fL({\theta}, G\sqrt{m} {\eta})\\
    \end{bmatrix} {w} \Big\|^2 \nonumber\\
    =&\|\sum_{i=1}^{m}\nabla_{{\theta}}\fL^i({\theta},G\sqrt{m}\eta^i)w^i\|^2 +G^2m\cdot\sum_{i=1}^{m}|\nabla_{\eta^i}\fL^i({\theta}, G\sqrt{m} \eta^i){w^i}|^2\nonumber \\
    =&\|\sum_{i=1}^{m}\nabla_{{\theta}}\fL^i({\theta},G\sqrt{m}\eta^i)w^i\|^2 +G^2m\cdot\sum_{i=1}^{m}|\nabla_{\eta_i}\fL^i({\theta}, G\sqrt{m} \eta^i){w^i}|^2\nonumber \\
     \overset{(i)}{\geq}& \| \sum_{i=1}^{m}\nabla_{{\theta}}\fL^i({\theta},G\sqrt{m}\eta^i)w^i\|^2+{G^2}\big(\sum_{i=1}^{m}|\nabla_{\eta^i}\fL^i({\theta}, G\sqrt{m}\eta^i)w^i|\big)^2\nonumber\\
    \overset{(ii)}{\geq}& \frac{1}{2}\Big(\| \sum_{i=1}^{m}w^i \nabla_{{\theta}}\fL^i({\theta}, G\sqrt{m}\eta^i)\|+{G}\sum_{i=1}^{m}| \nabla_{\eta^i}\fL^i({\theta},G\sqrt{m}\eta^i)w^i| \Big)^2\nonumber\\
    =&\frac{1}{2}\Big(\| \nabla_{{\theta}}\fL({\theta}, G\sqrt{m}\eta)w\|+{G}| \nabla_{\eta}\fL({\theta}, G\sqrt{m}\eta)w| \Big)^2,
\end{align}
where the first inequality applies $(a_1+....+a_m)^2\leq m\cdot \sum_{i=1}^{m}a_i^2$; the second inequality applies $\|a\|^2+\|b\|^2\geq \frac{1}{2}(a+b)^2$ holds for all $a,b\geq 0$. Taking square-root on both sides, we have
\begin{align}
    \|\nabla_{{\theta}, {\eta}}\fhL({\theta}, {\eta}){w} \|\geq \frac{1}{\sqrt{2}}\Big(\| \nabla_{{\theta}}\fL({\theta}, G\sqrt{m}\eta)w\|+{G} |\nabla_{\eta}\fL({\theta}, G\sqrt{m}\eta)w| \Big),\nonumber
\end{align}
which completes the proof.
\end{proof}

    

\section{Relevant Properties of Rescaled function $\fhL(\theta,\eta)$}\label{sec: property of hat L}
\begin{lemma}[Lemma 1 \citep{qirevistfDRO}]\label{lemma: generalized-smooth}
Let assumption \ref{assumption 1} hold, for each loss $\fhL^i(\theta, \eta^i)$, it satisfies $(\hL_0, \hL_1,\hL_2)$-smooth condition, i.e.,
\begin{align}
&\|\nabla_{\theta} {\fhL}^i({\theta}, \eta^i)-\nabla_{\theta} {\fhL}^i(\bar{\theta}, \eta^i)\|  \leq (\hL_0+\hL_1|\nabla_{\eta^i} {\fhL}^i({\theta}, \eta^i)|)\|{\theta}-\bar{\theta}\|, \nonumber\\
&|\nabla_{\eta^i} {\fhL}^i({\theta}, \eta^i)-\nabla_{\eta^i} {\fhL}^i({\theta}, \bar{\eta}^{i})| \leq \hL_2|\eta^i-\bar{\eta}^{i}|,
\end{align}
where $\hL_0=L+{G^2M}{\lambda}^{-1}, \hL_1={L}{(G\sqrt{m})^{-1}}, \hL_2 = {G^2Mm}{\lambda}^{-1}$
\end{lemma}
\begin{proof}
    The Proof follows \citet{qirevistfDRO} by changing $\fL(\theta, G\eta)$ into $\fhL(\theta,\eta)$. 
\end{proof}
\begin{remark}
The descent lemma with respect to $\theta$ is
\begin{align}
    {\fhL}^i({\theta}', \eta^i) \leq {\fhL}^i({\theta}, \eta^i)+\langle\nabla_{{\theta}} {\fhL}^i({\theta}, \eta^i), {\theta}'-{\theta}\rangle+\frac{\hL_0+\hL_1|\nabla_{\eta^i} {\fhL^i}({\theta}, \eta^i)|}{2}\|{\theta}'-{\theta}\|^2.
    \label{eq: single-objective descent lemma}
\end{align}
\end{remark}
\begin{lemma}[Lemma 3 \citep{qirevistfDRO}]\label{lemma: affine bounded noise}
Let assumption \ref{assumption 1} hold, for each objective $\fhL(\theta_t,\eta_t^i),i\in[m]$, the variance of $\HUpsilon_t^i, \HGamma_t^i$ can be upper bounded as
\begin{align}
    \mathbb{E}_{\bar{\xi}^{i}_{t}\sim \mathbb{P}^i}[\|\HGamma_t^{i} \|^2],\leq \hat{K_0}+\hat{K_1}|\nabla_{\eta^i}\fhL^i({\theta}_t, \eta_{t+1}^i)|^2,
    &\text{ and }\mathbb{E}_{\xi^{i}_{t}\sim \mathbb{P}^i}[\|\HUpsilon_{t}^{i} \|^2 ]\leq \hat{K}_2,
    \label{eq: variance bound for rescaled function}
\end{align}
where $\hat{K_0}=8G^2+10G^2M^2\lambda^{-2}\kappa^2$, $\hat{K_1} = 8/m$, $\hat{K_2} = mG^2 M^2 \lambda^{-2} \kappa^2$.
\end{lemma}
\begin{proof}
    The proof follows \citet{qirevistfDRO} by changing $\fL(\theta,G\eta)$ into $\fhL(\theta,\eta)$.
\end{proof}

\begin{restatable}[Descent Lemma]{lemma}{singleloopdescentlemma}
    Let Assumption \ref{assumption 1} hold. Denote $\hL_0 = L+G^2M\lambda^{-1}, \hL_1 = \fL(G\sqrt{m})^{-1}$. Then, for rescaled multi-objective function $\fhL({\theta}, {\eta})$, given any preference vector $w\in \mathcal{W}$,
    we have the following descent lemma
    \begin{align}
        \fhL({\theta}', {\eta}){w} &\leq \fhL({\theta}, {\eta}){w} + \langle\nabla_{{\theta}}\fhL({\theta}, {\eta}){w}, {\theta}'-{\theta} \rangle +\frac{\hL_0+\hL_1 |\nabla_{\eta}\fhL(\theta,\eta)w|}{2}\|{\theta}'-{\theta} \|^2,
        \label{eq: moo descent lemma theta}
\end{align}
holds for any $\theta,\theta'\in \mathbf{R}^n$ and fixed $\eta$. Additionally, we have
\begin{align}
        \fhL({\theta}, {\eta}'){w}\leq \fhL({\theta}, {\eta}){w} +\langle\nabla_{\eta}\fhL(\theta,\eta)w,\eta'-\eta \rangle + \frac{\hL_2}{2}\|\eta'-\eta\|^2,
        \label{eq: moo descent lemma eta}
\end{align}
holds for any $\eta,\eta'\in\mathbf{R}^m$ given $\theta$.

\end{restatable}
\begin{proof}
Based on the descent lemma \eqref{eq: single-objective descent lemma} for each $\fhL^i(\theta, \eta)$, since $\fhL({\theta}', {\eta}){w} = \sum_{i=1}^{m} \fhL^i({\theta}',{\eta}_i)w^i$, we have
\begin{align}
    &\fhL({\theta}', {\eta}){w} \nonumber\\
    &\overset{(i)}{\leq} \sum_{i=1}^{m}\big( {\fhL}^i({\theta}, \eta^i)+\langle\nabla_{{\theta}} {\fhL}^i({\theta}, \eta^i), {\theta}'-{\theta}\rangle+\frac{\hL_0+\hL_1|\nabla_{\eta^i} {\fhL}({\theta}, \eta^i)|}{2}\|{\theta}-{\theta}'\|^2\big)w^i \nonumber\\
    & \overset{(ii)}{=} \fhL({\theta}, {\eta}){w}+ \langle  \nabla_{\theta}^{\top}\fhL(\theta,\eta)(\theta'-\theta), w \rangle + \frac{\hL_0}{2}\|{\theta}'-{\theta} \|^2+ \frac{\hL_1\sum_{i=1}^{m} | \nabla_{\eta^i} \fhL({\theta}, {\eta}^i)| w^i}{2}\|{\theta}'-{\theta} \|^2 \nonumber\\
    &= \fhL({\theta}, {\eta}){w} + \langle \nabla_{\theta}\fhL({\theta}, {\eta})w,{\theta}'-{\theta} \rangle + \frac{\hL_0}{2}\|{\theta}'-{\theta} \|^2+ \frac{\hL_1 | \nabla_{\eta} \fhL({\theta}, {\eta})w| }{2}\|{\theta}'-{\theta} \|^2,  \nonumber
\end{align}
where (i) applies descent lemma \eqref{eq: single-objective descent lemma} for each $\fhL^i({\theta}, \eta^i)$; (ii) rewrites $\sum_{i=1}^{m}\langle\nabla_{{\theta}} {\fhL}^i({\theta}, \eta^i), {\theta}'-{\theta}\rangle w^i$ as $\langle \nabla_{\theta}^{\top}\fhL(\theta, \eta)(\theta'-\theta),w\rangle$. 

Then, for descent lemma \eqref{eq: moo descent lemma eta}, we have
    \begin{align}
         &\fhL(\theta,\eta')w-\fhL(\theta, \eta)w-\langle \nabla_{\eta}\fhL(\theta,\eta)w, \eta'-\eta \rangle \nonumber \\ 
        =& \sum_{i=1}^{m}\big(\int_{0}^{1}\langle \nabla_{\eta^i} \fhL^i(\theta, \eta^i_{{\iota^i}}), {\eta^{i'}}-\eta^i\rangle \mathrm{d} {\iota}^i - \int_{0}^{1} \langle \nabla_{\eta^i} \fhL^i(\theta, \eta^i), \eta^{i'}-\eta^i\rangle \mathrm{d}{\iota}^i\big)w^i  \nonumber \\
        \overset{(i)}{\leq}& \sum_{i=1}^{m}
        \big(\int_{0}^{1}  \|\nabla_{\eta^i} \fhL^i(\theta, \eta^i_{\iota^i})-\nabla_{\eta^i} \fhL(\theta,\eta^i ) \| \|\eta^{i'}-\eta^i \|\mathrm{d}\iota^i \big) w^i\nonumber \\
        \overset{(ii)}{\leq}& \sum_{i=1}^{m} \big(\int_{0}^{1} \iota^i \hL_2|\eta^{i'}-\eta^i|^2 \mathrm{d}\iota^i \big) w^i\nonumber\\
        =&\frac{\hL_2}{2}\big(\sum_{i=1}^{m}|\eta^{i'}-\eta^i|^2w^i\big)
        \leq \frac{\hL_2}{2}\|\eta'-\eta\|^2,\nonumber
    \end{align}
    where $\iota^i\in[0,1]$, $\eta^i_{\iota^i}=(1-\iota^i)\eta^i+\iota^i\eta^{i'}$; the first inequality utilizes Cauchy-Schwarz inequality in terms of $L_2$-norm ($\nabla_{\eta^i}\fhL^i(\theta, \eta^i), \eta^i\in\mathbf{R}$), and the second inequality is due to $\hL_2$-smooth property for each $\nabla_{\eta^i}\fhL^i(\theta, \eta^i)$ and the last inequality applies $w^i\leq 1,\forall i\in[m]$.
\end{proof}

\begin{lemma}[Relationship between $\nabla_{\theta}\fhL^i(\theta,\eta^i)$ and $\nabla_{\eta}\fhL^i(\theta, \eta^i)$]
    Given $(\theta,\eta^i)$, for $\nabla_{\theta}\fhL^i(\theta,\eta^i)$ and $\nabla_{\eta}\fhL^i(\theta,\eta^i)$, we have the following relationship
    \begin{align}
        \|\nabla_{\theta}\fhL^i(\theta, \eta^i)\|\leq G+|\nabla_{\eta^i}\fhL^i(\theta,\eta^i)|.
        \label{eq: gradient relationship}
    \end{align}
\end{lemma}
\begin{proof}
    Expanding $\nabla_{\theta}\fhL^i(\theta,\eta^i)$, we have
    \begin{align}
        \|\nabla_{\theta}\fhL^i(\theta, \eta^i)\|&=\|\mathbb{E}_{\xi}\big[(f^*)'\big(\frac{\ell^i(\theta;\xi^i)-G\sqrt{m}\eta^i}{\lambda} \big)\nabla\ell^i(\theta;\xi^i)\big]\|\nonumber\\
        &\leq \mathbb{E}_{\xi}\|(f^*)'\big(\frac{\ell^i(\theta;\xi^i)-G\sqrt{m}\eta^i}{\lambda}\big)\nabla \ell^i(\theta;\xi^i)\|\nonumber\\
        &\leq G \mathbb{E}_{\xi}\big[ |(f^*)'\big( \frac{\ell^i(\theta;\xi^i)-G\sqrt{m}\eta^i}{\lambda}\big) |\big]\nonumber\\
        &\overset{(i)}{=} G|\mathbb{E}_{\xi}\big[(f^*)'\big(\frac{\ell^i(\theta;\xi^i)-G\sqrt{m}\eta^i}{\lambda}\big) \big] |\nonumber\\
        &=G |1-\frac{\nabla_{\eta^i}\fhL^i(\theta,\eta^i)}{G\sqrt{m}}|\leq G+\frac{|\nabla_{\eta^i}\fhL^i(\theta,\eta^i)|}{\sqrt{m}}\leq G+|\nabla_{\eta^i}\fhL^i(\theta,\eta^i)|,
    \end{align}
where the first inequality applies Jensen's inequality; the second inequality applies $G$-Lipschitz assumption of $\ell^i(\theta;\xi^i)$; (i) utilizes the maximization arguments of convex conjugate function, the primal variable $r={\mathrm{d}\mathbb{Q}}/{\mathrm{d}\mathbb{P}}\geq 0$ thus implies ${f^*}'(s)\geq0$; the last inequality applies $|a-b|\leq |a|+|b|$ and $m\geq 1$.
\end{proof}

\begin{corollary}[Gradient Upper bound of $\nabla_{\eta}\fhL(\theta_t, \eta_{t+1})$]
    Given update rule $\eta_{t+1} = \eta_t - \gamma\mu_t Z_tw_t$, if there exists $\Lambda$ such that $|\nabla_{\eta^i}\fhL(\theta_t, \eta_{t}^i)|<\Lambda$~\footnote{Notice, $\nabla_{\eta^i}\fhL(\theta_t,\eta_t^i)\in \mathbf{R}$ is a scalar, this implies $|\nabla_{\eta^i}\fhL(\theta_t, \eta_{t}^i)|=\|\nabla_{\eta^i}\fhL(\theta_t, \eta_{t}^i)\|$.} holds for all $t\leq T$. Then, for $|\nabla_{\eta}\fhL(\theta_t, \eta_{t+1})w|$, we have
    \begin{align}
        |\nabla_{\eta} \fhL(\theta_t, \eta_{t+1})w|\leq \hL_2\gamma\mu_t\|Z_tw_t\| + |\nabla_{\eta}\fhL(\theta_t, \eta_t)w|\leq\hL_2\gamma\mu_t\|Z_tw_t\|+\Lambda,
    \end{align}
    holds for all $t\leq T$.
\end{corollary}
\begin{proof}
Expanding $|\nabla_{\eta}\fhL(\theta_t,\eta_{t+1})w|$, we have
\begin{align}
    |\nabla_{\eta} \fhL(\theta_t, \eta_{t+1})w|&=\sum_{i=1}^{m}w^i|\nabla_{\eta^i}\fhL^i(\theta_t, \eta_{t+1}^i) |\nonumber\\
    &\leq \hL_2\langle w, |\eta_{t+1}-\eta_t|\rangle + |\nabla_{\eta} \fhL(\theta_t, \eta_{t})w|\nonumber\\
    &=\hL_2\gamma\mu_t\langle w, |Z_tw_t| \rangle + |\nabla_{\eta} \fhL(\theta_t, \eta_{t})w|\nonumber\\
    &\leq \hL_2\gamma\mu_t\|Z_tw_t\| + |\nabla_{\eta}\fhL(\theta_t, \eta_t)w|\nonumber\\
    &\leq\hL_2\gamma\mu_t\|Z_tw_t\|+\Lambda,
    \label{eq: partial-smooth property}
\end{align}
where the first inequality applies $\hL_2$-smooth w.r.t. $|\nabla_{\eta}\fhL^i(\theta, \eta^i)|$; the second inequality applies Cauchy-Schwarz inequality and $\|Z_tw_t\| = \sqrt{\sum_{i=1}^{m}(Z_t^iw_t^i)^2}=\| | Z_tw_t| \|$.
\end{proof}

\begin{corollary}
If there exists $\Lambda$ such that $|\nabla_{\eta^i}\fhL(\theta_t, \eta_{t}^i)|<\Lambda$. For $\nabla_{\theta}\fhL(\theta_{t},\eta_{t+1})$ and $\nabla_{\eta} \fhL(\theta_t, \eta_{t})$, we have the following upper bound
    \begin{align}
        \|\nabla_{\theta}\fhL(\theta_t,\eta_{t+1})\|_F^2\leq 3mG^2 + 3\gamma^2\mu_t^2\hL_2^2 \|Z_tw_t\|^2+3m\Lambda^2.
        \label{eq: gradient relationship 2}
    \end{align}
\end{corollary}
\begin{proof}
    Expanding $\|\nabla_{\theta}\fhL(\theta_t, \eta_{t+1})\|^2$, we have
    \begin{align}
        \|\nabla_{\theta}\fhL(\theta_t, \eta_{t+1})\|_F^2&=\sum_{i=1}^{m}\|\nabla_{\theta}\fhL^i(\theta_t, \eta_{t+1}^i) \|^2\nonumber\\
        &\leq \sum_{i=1}^{m}(G+|\nabla_{\eta^i}\fhL^i(\theta_t, \eta_{t+1}^i)|)^2\nonumber\\
        &\leq \sum_{i=1}^{m}(G+\gamma\mu_t\hL_2|Z_t^iw_t^i|+\Lambda)^2\nonumber\\
        &\leq 3mG^2+3\gamma^2\mu_t^2\hL_2^2\sum_{i=1}^{m}|Z_t^iw_t^i|^2 + 3m\Lambda^2 \nonumber\\
        &= 3mG^2 + 3\gamma^2\mu_t^2\hL_2^2 \|Z_tw_t\|^2+3m\Lambda^2,
    \end{align}
    where the first inequality leverages \eqref{eq: gradient relationship} for each $\nabla_{\theta}\fhL^i(\theta_t, \eta_{t+1})$ and $\nabla_{\eta^i}\fhL^i(\theta_t, \eta_{t+1})$; the second inequality leverages the $\hL_2$-smooth of $\nabla_{\eta^i}\fhL^i(\theta_t, \eta_{t+1}^i)$, such that $|\nabla_{\eta^i}\fhL^i(\theta_t, \eta_{t+1}^i)|\leq \hL_2|\eta_{t+1}^i-\eta_{t}^i|+|\nabla_{\eta^i}\fhL^i(\theta_t, \eta_{t}^i)|\leq \gamma\mu_t\hL_2|Z_t^iw_t^i|+\Lambda$; the third inequality leverages $(a+b+c)^2\leq 3a^2+3b^2+3c^2$.
\end{proof}

\begin{corollary}
    For $\HGamma_{t}={X}_t-\nabla_{\theta}\fhL(\theta_t, \eta_{t+1})$, if there exists $\Lambda<\infty$ such that
    $\|\nabla_{\eta^i}\fhL^i(\theta_t, \eta^i_{t})\|\leq \Lambda$ holds~\footnote{Notice, $\nabla_{\eta^i}\fhL(\theta_t,\eta_t^i)\in \mathbf{R}$ is a scalar, this implies $|\nabla_{\eta^i}\fhL(\theta_t, \eta_{t}^i)|=\|\nabla_{\eta^i}\fhL(\theta_t, \eta_{t}^i)\|$.}. Then, we have following upper bound
    \begin{align}
        \mathbb{E}\big[\|\HGamma_{t}w\|\big]\leq \Xi_3':=\sqrt{\tfrac{\hat{K}_0}{N_1}}+\sqrt{\tfrac{2\hat{K}_1\hL_2^2\gamma^2f_2^2}{N_1}}+\sqrt{\tfrac{2\hat{K}_1\Lambda^2}{N_1}},
        \label{eq: estimation error bound thm2}
    \end{align}
\end{corollary}
\begin{proof}
    Expanding $\mathbb{E}\big[\|\HGamma_{t}w\|\big]$, we have
    \begin{align}
        \mathbb{E}\big[\|\HGamma_{t}w\|\big]&=\mathbb{E}\big[\mathbb{E}_{\xi_t,\bar{\xi}_t,\eta_{t+1}}[\| \HGamma_tw\||\theta_t,\eta_t,w_t,t]\big]\nonumber\\
        &\leq \mathbb{E}\big[\sqrt{\mathbb{E}_{\xi_t,\bar{\xi}_t,\eta_{t+1}}\big[\|\HGamma_{t}w\|^2 |\theta_t,\eta_t,w_t,t\big]}\big]\nonumber\\
        &\overset{(i)}{\leq} \mathbb{E}\big[\sqrt{\mathbb{E}_{\xi_t,\bar{\xi}_t,\eta_{t+1}}[\sum_{i=1}^{m}w^i\|\HGamma_{t}^i\|^2|\theta_t,\eta_t,w_t,t]}\big]\nonumber\\
        &\overset{(ii)}{\leq}\mathbb{E}\Big[ \sqrt{\mathbb{E}_{\xi_t,\eta_{t+1}}\big[\sum_{i=1}^{m}w^i(\tfrac{\hat{K}_0}{N_1}+\tfrac{\hat{K}_1}{N_1}|\nabla_{\eta^i}\fhL^i(\theta_t, \eta_{t+1}^i)|^2|\theta_t,\eta_t,w_t,t\big]}\Big]\nonumber\\
        &\overset{(iii)}{\leq} \mathbb{E}\Big[\sqrt{\mathbb{E}\big[\sum_{i=1}^{m}w^i(\tfrac{\hat{K}_0}{N_1}+\tfrac{\hat{K}_1}{N_1}(\hL_2\gamma\mu_t|Z_t^iw_t^i|+\Lambda)^2)|\theta_t,\eta_t,w_t,t\big]}\Big]\nonumber\\
        &\overset{}{\leq} \mathbb{E}\Big[\sqrt{\mathbb{E}_{\xi_t}\big[\tfrac{\hat{K}_0}{N_1}+\tfrac{2\hat{K}_1\hL_2^2\gamma^2\mu_t^2}{N_1}\|Z_tw_t\|^2+\tfrac{2\hat{K}_1}{N_1}\Lambda^2|\theta_t,\eta_t,w_t,t\big]}\Big]\nonumber\\
        &\overset{(iv)}{\leq} \sqrt{\tfrac{\hat{K}_0}{N_1}+\tfrac{2\hat{K}_1\hL_2^2\gamma^2f_2^2}{N_1}+\tfrac{2\hat{K}_1\Lambda^2}{N_1}}\nonumber\\
        &\overset{(v)}{\leq} {\Xi'_3}:=\sqrt{\tfrac{\hat{K}_0}{N_1}}+\sqrt{\tfrac{2\hat{K}_1\hL_2^2\gamma^2f_2^2}{N_1}}+\sqrt{\tfrac{2\hat{K}_1\Lambda^2}{N_1}},
        \label{eq: inner product bound thm 2}
    \end{align}
    where (i) applies Jensen's inequality; (ii) applies the fact that $\mathbb{E}_{\bar{\xi}_{t}^i}\|\HGamma_{t}^i\|^2\leq \hat{K}_0/N_1+(\hat{K}_1/N_1)\cdot\|\nabla_{\eta^i}\fhL^i(\theta_t,\eta^{i}_{t+1})\|^2$; (iii) applies the fact $|\nabla_{\eta^i}\fhL^i(\theta_t,\eta^{i}_{t+1})|^2\leq (\hL_2\mu_t\gamma|Z_t^iw_t^i|+\|\nabla_{\eta^i}\fhL^i(\theta_t, \eta_{t}^i)\|)^2\leq 2\hL_2^2\mu_t^2\gamma^2 |Z_t^iw_t^i|^2+2\Lambda^2$;
    (iv) utilizes gradient clipping rule, $\mu_t\leq \frac{f_2}{\|Z_tw_t\|}$; the last inequality leverages $\sqrt{a+b+c}\leq \sqrt{a}+\sqrt{b}+\sqrt{c}$.
\end{proof}

\section{Convergence Analysis of Algorithm~\ref{alg1-3}}
\subsection{Formal Statement of Theorem \ref{thm: convergence of alg3} and Proof}\label{Appendix: Convergence of Algorithm 3}
Given well-defined problem dependent parameters $m$,$G$,$\hL_0=L+G^2M\lambda^{-1}$,$\hL_1=\fL(G\sqrt{m})^{-1}$,$\hL_2=G^2Mm\lambda^{-1}$,$\hat{K}_0=8G^2+10G^2M^2\lambda^{-2}\kappa^2$, $\hat{K}_1=8/m$, $\hat{K}_2=mG^2M^2\lambda^{-2}\kappa^2$,$\Lambda=\sup\{u\geq 0|u^2\leq2\hL_2\bar{F}(u+1) \}$, and constant $C_0$.
Let $d_1'\dots d_8'$, $e_1\dots e_3$ be constant. Let $\bar{F}<\infty$ be a constant such that
\begin{align}
    \tfrac{\bar{F}}{8}\geq \max\big\{ \tfrac{m(\hat{\Delta}_{\theta_0,\eta_0}+d'_1+\dots+d'_8)}{\delta},e_1+e_2+ \sum_{j=3}^{8}d_j' \big\},
\end{align}
where $\hat{\Delta}_{\theta_0,\eta_0}
:= \max_{i\in[m]}\{\fhL^i(\theta_0,\eta_0^i)-\fhL^{i,*} \}$ is the initial objective gap.
Set parameters of Algorithm \ref{alg1-3} as follows
\begingroup
\setlength{\jot}{5pt}
\begin{align}
c_2,f_2
&= \delta\epsilon, \qquad
\alpha_t = \min\bigl\{
    \tfrac{1}{2},
    \tfrac{\delta\epsilon}{\|X_tw_t\|}
   \bigr\},
\qquad
\mu_t
= \min\bigl\{
    \tfrac{1}{2},
    \tfrac{\delta\epsilon}{\|Z_tw_t\|}
  \bigr\},
\nonumber\\
N_1
&= \max\bigl\{
    9\hat{K}_0\delta^{-3}\epsilon^{-2},
    18\hat{K}_1\delta^{-1},
    18\Lambda^2\hat{K}_1\delta^{-3}\epsilon^{-2}
   \bigr\} = \mathcal{O}\bigl(
    \max\{G^2M^2\lambda^{-2}\kappa^2,{\Lambda^2}m^{-1}\}
    \delta^{-3}\epsilon^{-2}
   \bigr),
\nonumber\\
N_2
&= \max\bigl\{
    18\hat{K_2}\delta^{-3}\epsilon^{-2}
   \bigr\}
 = \mathcal{O}\bigl(
    mG^2M^2\lambda^{-2}\kappa^2\delta^{-3}\epsilon^{-2}
   \bigr),
\nonumber\\
\Xi_1'
&= mG^2+\gamma^2\hL_2^2\delta^2\epsilon^2+m\Lambda^2
 \leq mG^2+m\Lambda^2+1,
\nonumber\\
\Xi_2'
&= \tfrac{m\hat{K}_0}{N_1}
  +\tfrac{2\hL_2^2\gamma^2\delta^2\epsilon^2\hat{K_1}}{N_1}
  +\tfrac{2m\Lambda^2\hat{K}_1}{N_1}
 \leq \tfrac{m\delta^3\epsilon^2}{3},
\nonumber\\
\Xi_3'
&= \sqrt{\tfrac{\hat{K}_0}{N_1}}
  +\sqrt{\tfrac{2\hat{K}_1\hL_2^2\gamma^2\delta^2\epsilon^2}{N_1}}
  +\sqrt{\tfrac{2\hat{K}_1\Lambda^2}{N_1}}
 \leq \delta^{3/2}\epsilon,
\nonumber\\[3pt]
\gamma
&= \min\bigl\{
    \tfrac{2}{(\hL_0+\hL_1\Lambda)},
    \tfrac{1}{\hL_2+2\vartheta\hL_1^2\hL_2^2},
    \Bigl(\tfrac{16\vartheta}{10\delta^2\epsilon^2}\Bigr)^{1/3},
    \beta d_3',
    \tfrac{9me_1^2}{C_0^2},
    \tfrac{\min\{d_1',d_2'\}}{T\delta^{5/2}\epsilon^2},
    \tfrac{d_6'}{(3\Xi_2'+9\Xi_1')\delta^2\epsilon^2T\beta},
\nonumber\\
&\qquad
    \tfrac{d_7'}{3m(\Lambda^2+\delta^3\epsilon^2)\delta^2\epsilon^2\beta T},
    \Bigl(\tfrac{16\vartheta d_8'}{\delta^4\epsilon^4T}\Bigr)^{1/4}
   \bigr\}
\nonumber\\
&= \min\bigl\{
    \mathcal{O}\bigl(\tfrac{1}{\hL_0+\hL_1\Lambda}\bigr),
    \mathcal{O}\bigl(\tfrac{1}{\hL_2+\vartheta\hL_1^2\hL_2^2}\bigr),
    \mathcal{O}(\beta),
    \mathcal{O}\bigl(
      \tfrac{1}{(mG^2+m\Lambda^2)\beta T\delta^2\epsilon^2}
    \bigr)
   \bigr\},
\nonumber\\[3pt]
\beta
&= \min\bigl\{
    \tfrac{4\rho^{-1}}{3},
    \tfrac{1}{90(mG^2+m\Lambda^2+1)},
    \tfrac{1}{30m\delta^3\epsilon^2},
    \tfrac{1}{30m\Lambda^2},
    \tfrac{6e_2}{C_0^2},
    \tfrac{4d_5'}{3d_4'\rho}
   \bigr\}= \mathcal{O}\bigl(\tfrac{1}{mG^2+m\Lambda^2}\bigr),
\nonumber\\
\rho
&= \min\bigl\{
    \tfrac{1}{20}\delta^2\epsilon^2,
    \tfrac{d_4'}{2\gamma T}
   \bigr\}
 = \min\bigl\{
    \mathcal{O}(\delta^2\epsilon^2),
    \mathcal{O}\bigl(\tfrac{1}{\gamma T}\bigr)
   \bigr\},
\nonumber\\
T
&= \max\bigl\{
    10\hat{\Delta}_{\theta_0,\eta_0}\gamma^{-1},
    10\beta^{-1}
   \bigr\}\delta^{-2}\epsilon^{-2} = \Theta\bigl(
    \max\{\hat{\Delta}_{\theta_0,\eta_0}\gamma^{-1},\beta^{-1}\}
    \delta^{-2}\epsilon^{-2}
   \bigr)
\nonumber\\
\delta\epsilon
&\leq \sqrt{
    \min\bigl\{1,
      \tfrac{C_0^2}{560m^2\hat{\Delta}_{\theta_0,\eta_0}},
      \tfrac{C_0^2\beta}{560m^2\gamma}
    \bigr\}
   } = \min\big\{1,
    \mathcal{O}\bigl(\sqrt{\tfrac{1}{m^2\hat{\Delta}_{\theta_0,\eta_0}}}\bigr),
    \mathcal{O}(\sqrt{\tfrac{\beta}{m^2\gamma}})
   \big\}.
\end{align}
\endgroup
Then we are ready to present the following Convergence result.
\begin{restatable}[Formal Statement of Theorem~\ref{thm: convergence of alg3}]{theorem}{mainthmthree}\label{restatethm: convergence of alg3}
     Let Assumption \ref{assumption 1} hold. Denote $\hat{\Delta}_{\theta_0,\eta_0}$ as the initial objective gap,  $\hL_0=L+G^2M\lambda^{-1}$,$\hL_1=\fL(G\sqrt{m})^{-1}$,$\hL_2=G^2Mm\lambda^{-1}$.
     Fix $\delta\in (0,1)$, for any $\epsilon>0$  such that $\delta\epsilon\leq \min\{1,\mathcal{O}({m^{-1}\Delta_{\theta_0,\eta_0}^{-1/2}}),\mathcal{O}(m^{-1}{\beta}^{1/2}{\gamma}^{-1/2})\}$. Set the hyperparameters in Algorithm~\ref{alg1-3} as $c_1,f_1=\tfrac{1}{2}$, $c_2,f_2=\delta\epsilon$, 
    $\rho=\min\{\mathcal{O}(\delta^2\epsilon^2),\mathcal{O}(\tfrac{1}{\gamma T})\}$, $\beta=\mathcal{O}(\tfrac{1}{mG^2+m\Lambda^2})$ and $\gamma = \min\big\{ \mathcal{O}(\tfrac{1}{\hL_0+\hL_1\Lambda}), \mathcal{O}(\tfrac{1}{\hL_2+\vartheta\hL_1^2\hL_2^2}),\mathcal{O}(\beta),\mathcal{O}(\tfrac{1}{(mG^2+m\Lambda^2)\beta T\delta^2\epsilon^2}) \big\}$.
    Choose batch sizes $N_1=\mathcal{O}(\max\{G^2M^2\lambda^{-2}\kappa^2,\Lambda^2m^{-1} \}\delta^{-3}\epsilon^{-2})$, $N_2= \mathcal{O}(mG^2M^2\lambda^{-2}\kappa^2\delta^{-3}\epsilon^{-2})$. Then, after
    $T
    =\Theta(\max\{ \hat{\Delta}_{\theta_0,\eta_0}\gamma^{-1}\delta^{-2}\epsilon^{-2},\beta^{-1}\delta^{-2}\epsilon^{-2}\})$ iterations, with probability at least $1-\delta$, we have 
    \begin{align}
        \frac{1}{T}\sum_{t=0}^{T-1}\|\nabla_{\theta,\eta}\fhL(\theta_t, \eta_{t+1})w_t\|\leq 34\epsilon
        .\nonumber
    \end{align}
\end{restatable}
\begin{proof} \textbf{Part I: Tail Event $\{\htau<T \}$}. Denote $\mathcal{F}_t$ as $\sigma$-algebra of initialization, and batch samples before time $t$, i.e., $\sigma\big(\theta_0, \eta_0, w_0, \{\xi_s,\bar{\xi}_s\}_{s<t} \big)$. Define event $\calhatE_t = \{\forall i\in[m], \fhL(\theta_t,\eta_t^i) - \hL^{i,*}\leq \bar{F} \}$, and stopping time
    \begin{align}
        \htau_1 &= \min\{t|\exists i\in[m],\|\HGamma^i_{t}\|\text{ or }\|\HUpsilon_t^i\|\geq \tfrac{C_0}{6\delta\epsilon\sqrt{\gamma m}} \} \wedge T \nonumber\\
        \htau_2 &= \min\{ t|\exists i\in[m], \fhL^i(\theta_{t+1},\eta_{t+1}^i)-\fhL^{i,*}>\bar{F}\}\wedge T\nonumber\\
        \htau &= \min\{\htau_1,\htau_2\}.
    \end{align}
We partition event $\{\htau<T \}$ to be $A= \{\htau_1\leq \htau_2, \htau_1< T \}$ and $B = \{\htau_2< \htau_1, \htau_2<T \}$. 
For $\mathbb{E}[\mathbf{1}(\calhatE_t)\|\HGamma_t^i\|^2]$ 
\begin{align}
    \mathbb{E}\Big[ \mathbf{1}(\calhatE_t) \|\hat{\Gamma}_{t}^i\|^2\Big] &=\mathbb{E}[\mathbb{E}[\mathbf{1}(\calhatE_t) \|\hat{\Gamma}_{t}^i\|^2|\mathcal{F}_t]] \overset{}{\leq} \Big(\frac{\hat{K}_0}{N_1}+\frac{2\hL_2^2\gamma^2f_2^2\hat{K_1}}{N_1}+\frac{2\Lambda^2\hat{K}_1}{N_1}\Big),
\end{align}
where the derivation follows~\eqref{eq: variance intermediate bound alg3}. Similar arguments gives $\mathbb{E}[\mathbf{1}(\calhatE_t)\|\HUpsilon^i_t\|^2]\leq \frac{\hat{K_2}}{N_2}$. Then, for event $A$, since $\htau_1\leq \htau_2$, we have $\|\nabla_{\eta}\fhL^i(\theta_t,\eta_t^i)\|\leq \Lambda,\forall i\in[m]$ still holds. By Markov inequality, we know
\begin{align}
    P(\calhatE_t \cap\{\|\HGamma_{t}^i\|\geq \tfrac{C_0}{6\delta\epsilon\sqrt{\gamma   m}}\}) &= P(\calhatE_t\cap \{ \|\HGamma_{t}^i\|^2\geq \tfrac{{C_0^2}}{36m\gamma \delta^2\epsilon^2 }\})\nonumber\\
    &\leq \tfrac{36m\gamma\delta^2\epsilon^2}{{C_0^2}}(\tfrac{\hat{K}_0}{N_1}+\tfrac{2\hat{K}_1\gamma^2\hL_2^2\delta^2\epsilon^2}{N_1}+\tfrac{2\hat{K}_1\Lambda^2}{N_1}),\nonumber\\
    P(\calhatE_t \cap \{\|\HUpsilon_t^i\|\geq \tfrac{C_0}{6\delta\epsilon\sqrt{\gamma m}}\})&= P(\calhatE_t\cap \{ \|\HUpsilon_t^i\|^2\geq \tfrac{{C_0^2}}{36\delta^2\epsilon^2\gamma m}\})\leq \tfrac{36 m\gamma\delta^2\epsilon^2\hat{K}_2}{{C_0^2}N_2}.
\end{align}
Then for event $A$, its probability $P(A)$ can be bounded as
\begin{align}
    P(A)&\leq P\big(\cup_{t=0}^{T-1} \cup_{i=1}^m\big[\calhatE_t \cap(\{\|\HGamma_t^i\| \geq \tfrac{C_0}{6\delta\epsilon\sqrt{\gamma m}}\} \cup\{\|\HUpsilon_t^i\| \geq \tfrac{C_0}{6\delta\epsilon\sqrt{\gamma m}}\}\})\big]\big)\nonumber\\
    &\leq \sum_{i=1}^{m}\sum_{t=0}^{T-1}\big( P(\calhatE_t\cap \|\HGamma_{t}^i\|\geq \tfrac{C_0}{6\delta\epsilon\sqrt{\gamma   m}})+P(\calhatE_t\cap \|\HUpsilon_t^i\|\geq \tfrac{C_0}{6\delta\epsilon\sqrt{\gamma m}})\big)\nonumber\\
    &\leq \tfrac{36m^2\gamma T\delta^2\epsilon^2}{{C_0^2}}(\tfrac{\hat{K}_0}{N_1}+\tfrac{2\hat{K}_1\gamma^2\hL_2^2\delta^2\epsilon^2}{N_1}+\tfrac{2\hat{K}_1\Lambda^2}{N_1}+\tfrac{\hat{K}_2}{N_2})\leq \tfrac{\delta}{4},
\end{align}
where the last inequality is due to pre-chosen $\delta,\epsilon$ satisfies ${56m^2\delta^4\epsilon^4\gamma T\leq{C_0^2}}$.

For event $B$, we know $\htau=\htau_2<T$, by setting index $i$ with  $w^i=1, w_{m/[i]}=0$, evaluate \eqref{eq: one-step moo descent alg3} at $t=\htau$,
we have
\begin{align}
&\fhL^i(\theta_{\htau+1},\eta_{\htau+1})\nonumber\\
&\leq \fhL^i(\theta_{\htau},\eta_{\htau}) +\gamma \mu_t\|\HUpsilon_t^i\|\| Z_tw_t\|+\gamma\alpha_t\| \HGamma_t^i\|\|X_tw_t\|+\frac{\gamma}{\beta}+\frac{3\gamma\beta\rho^2}{2}+2\gamma\rho+ 3\gamma\beta\delta^2\epsilon^2\|\HGamma_t\|_F^2\nonumber\\
&\quad +9\gamma\beta\Xi'_1\delta^2\epsilon^2+3\beta\gamma\delta^2\epsilon^2\|\HUpsilon_t\|_F^2+3\gamma\beta\delta^2\epsilon^2m\Lambda^2+\frac{\gamma^4\delta^4\epsilon^4}{16\vartheta}\nonumber\\
&\leq \fhL^i(\theta_{\htau},\eta_{\htau})+\gamma \delta\epsilon\|\HUpsilon_t^i\|+\gamma\delta\epsilon\|\HGamma_t^i\|+\frac{\gamma}{\beta}+\frac{3\gamma\beta\rho^2}{2}+2\gamma\rho+ 3\gamma\beta\delta^2\epsilon^2\|\HGamma_t\|_F^2\nonumber\\
&\quad +9\gamma\beta\Xi_1'\delta^2\epsilon^2+3\beta\gamma\delta^2\epsilon^2\|\HUpsilon_t\|_F^2+3\gamma\beta\delta^2\epsilon^2m\Lambda^2+\frac{\gamma^4\delta^4\epsilon^4}{16\vartheta}\nonumber\\
&\leq \fhL^i(\theta_{\htau},\eta_{\htau})+\frac{C_0}{3}\sqrt{\frac{\gamma}{m}}+\frac{\gamma}{\beta}+\frac{3\gamma\beta\rho^2}{2}+2\gamma\rho+\frac{C_0^2\beta}{6}+9\gamma\beta\Xi_1'\delta^2\epsilon^2+3\gamma\beta\delta^2\epsilon^2m\Lambda^2+\frac{\gamma^4\delta^4\epsilon^4}{16\vartheta},
\end{align}
where the first inequality leverages Cauchy-schwarz inequality to upper bound $\gamma\mu_t\langle\HUpsilon_tw,Z_tw_t\rangle$ and $\gamma\alpha_t\langle\HGamma_tw,X_tw_t \rangle$; the second inequality leverages clipping rule $\alpha_t<\tfrac{\delta\epsilon}{\|X_tw_t\|},\mu_t\leq \tfrac{\delta\epsilon}{\|Z_tw_t\|}$; the last inequality leverages fact $\|\HGamma_t^i\|,\|\HUpsilon_t^i\|\leq \tfrac{C_0}{6\delta\epsilon\sqrt{\gamma m}}$ since $\htau=\htau_2<\htau_1$.

Since we know, for ${0<\delta<1}$, we have
\begin{align}
    \frac{\gamma}{\beta}+\frac{3\gamma\beta\rho^2}{2}+2\gamma\rho+9\gamma\beta\Xi_1'\delta^2\epsilon^2+3\gamma\beta\delta^2\epsilon^2m\Lambda^2+\frac{\gamma^4\delta^4\epsilon^4}{16\vartheta}\leq d_3'+\dots+d_8'\leq\frac{\delta \bar{F}}{8}<\frac{\bar{F}}{2}.
\end{align}
By setting $\gamma,\beta$ satisfying ${\gamma\leq \frac{9me_1^2}{C_0^2},\beta\leq \frac{6e_2}{C_0^2}}$,
we conclude $ \fhL^i(\theta_{\tau+1},\eta_{\tau+1})-\fhL^i(\theta_\tau,\eta_\tau)\leq \frac{\bar{F}}{2}$ since $e_1+e_2+\frac{\gamma}{\beta}+2\gamma\rho+9\gamma\beta\Xi_1\delta^2\epsilon^2+3\gamma\beta\delta^2\epsilon^2m\Lambda^2+\frac{3\gamma\beta\rho^2}{2}+\frac{\gamma^4\delta^4\epsilon^4}{16\vartheta}\leq \frac{F}{2}$.

However, we have $\max_{i\in[m]}\fhL^i(\theta_{\htau+1},\eta^i_{\htau+1})-\hL^{i,*}\geq \bar{F}$, thus, for this task, we have $\max_{i\in [m]}\fhL(\theta_{\htau},\eta_{\htau}^i)-\hL^{i,*}>\frac{\bar{F}}{2}$.
Leveraging Lemma~\ref{lemma: descent lemma of alg1-3} and Markov inequality, this further implies
\begin{align}
    P(B)&\leq P(\max_{i\in [m]}\fhL^i(\theta_{\htau},\eta_{\htau})-\fhL^{i,*}>\frac{\bar{F}}{2})\nonumber\\
    &\leq P(\cup_{i=1}^{m} [\fhL^i(\theta_{\htau},\eta_{\htau})-\fhL^{i,*}>\frac{\bar{F}}{2}]) \nonumber\\
    &\leq \sum_{i=1}^{m} P(\fhL^i(\theta_{\htau},\eta_{\htau})-\fhL^{i,*}>\frac{\bar{F}}{2})\nonumber\\
    &\leq m\cdot\frac{\bar{F}\delta/8m}{\bar{F}/2}= \frac{\delta}{4}.\nonumber
\end{align}
Thus, we have $P(\htau < T)\leq P(A)+P(B)\leq \frac{\delta}{2}$. 

\textbf{Part II: Convergence of $\frac{\gamma}{2T}\mathbb{E}[(\sum_{t=0}^{T-1}\alpha_t\|X_tw_t\|^2+\sum_{t=0}^{T-1}\mu_t\|Z_tw_t\|^2)\mathbf{1}(\htau = T)]$}.
Reorganize and divide both sides with $T$ on ~\eqref{eq: descent lemma alg1-3} in Lemma~\ref{lemma: descent lemma of alg1-3}, leveraging specified choices of parameter $\gamma$,$\beta$,$\rho$, and $N_1, N_2$. We have
\begin{align}
    &\frac{\gamma}{2T}\mathbb{E}\big[\big(\sum_{t=0}^{T-1}\alpha_t\|X_tw_t\|^2+\sum_{t=0}^{T-1}\mu_t\|Z_tw_t\|^2\big)\mathbf{1}(\htau = T)\big] \nonumber\\
    & \leq \frac{\gamma}{2T}\mathbb{E}\big[\big(\sum_{t=0}^{\htau-1}\alpha_t\|X_tw_t\|^2+\sum_{t=0}^{\htau-1}\mu_t\|Z_tw_t\|^2\big)\big] \nonumber\\
    &\leq \frac{\hat{\Delta}_{\theta_0,\eta_0}}{T} +\gamma \delta\epsilon\sqrt{\hat{K}_2/N_2}+\gamma \delta\epsilon\Xi_3'+\frac{\gamma}{\beta T}+\frac{3\gamma\beta\rho^2}{2} +2\gamma\rho\nonumber\\
    &
    \quad + 3\gamma\beta \delta^2\epsilon^2\Xi_2' + 9\gamma\beta\Xi_1'\delta^2\epsilon^2 + 3m\beta\gamma{\delta^2\epsilon^2}(\hat{K}_2/N_2)+3m\gamma\beta{\delta^2\epsilon^2}\Lambda^2 + \frac{\gamma^4\delta^4\epsilon^4}{16\vartheta}
    \nonumber\\
    &\leq (9\cdot \frac{1}{10}+2)\gamma\delta^2\epsilon^2 \leq 3\gamma\delta^2\epsilon^2.
\end{align}
This further implies
\begin{align}
    \frac{1}{T}\mathbb{E}[\sum_{t=0}^{T-1}\alpha_t\|X_tw_t\|^2\mathbf{1}(\htau =T)]\leq 6\delta^2\epsilon^2, \text{ and } \frac{1}{T}\mathbb{E}[\sum_{t=0}^{T-1}\mu_t{\|Z_tw_t\|^2}\mathbf{1}(\htau=T)]\leq 6 \delta^2\epsilon^2,\nonumber
\end{align}
Leveraging $\min\{\frac{a^2}{2},a\}\geq a-\frac{1}{2}$, we have
\begin{align}
    \frac{1}{T}\mathbb{E}[\sum_{t=0}^{T-1}\alpha_t\|X_tw_t\|^2\mathbf{1}(\htau = T)] &=\frac{1}{T}\mathbb{E}\big[\sum_{t=0}^{T-1}\delta^2\epsilon^2 \min\{\frac{\|X_tw_t\|^2}{2\delta^2\epsilon^2},\frac{\|X_tw_t\|}{\delta\epsilon} \} \mathbf{1}(\htau = T)\big]\nonumber\\
    &\geq\frac{\delta\epsilon}{T}\mathbb{E}[\sum_{t=0}^{T-1}\|X_tw_t\|\mathbf{1}(\htau = T)]-\frac{\delta^2\epsilon^2}{2} P(\tau= T).\nonumber
\end{align}
Similar lower bound also applies to $\frac{1}{T}\mathbb{E}[\sum_{t=0}^{T-1}\mu_t\|Z_tw_t\|^2\mathbf{1}(\htau=T)]$.
This implies 
\begin{align}
    \frac{1}{T}\mathbb{E}[\sum_{t=0}^{T-1}\|X_tw_t\|\mathbf{1}(\htau=T)] \text{ and } \frac{1}{T}\mathbb{E}[\sum_{t=0}^{T-1}\|Z_tw_t\|\mathbf{1}(\htau = T)]\leq (6\delta\epsilon+\frac{\delta\epsilon}{2}P(\tau= T))\leq 7\delta\epsilon.\nonumber
\end{align}
Futhermore, for $\mathbb{E}[\|\HUpsilon_tw_t\|^2\mathbf{1}(\htau = T)]$, leveraging Lemma~\ref{lemma: conditioned GammaXi}, we have
\begin{align}
   \mathbb{E}[\|\HUpsilon_tw_t\|\mathbf{1}(\htau = T)]\leq \delta\epsilon, \text{ and } \mathbb{E}[\|\HGamma_tw_t\|\mathbf{1}(\htau = T)]\leq \delta\epsilon.\nonumber
\end{align}
This further implies
\begin{align}
\frac{1}{T}\mathbb{E}[\sum_{t=0}^{T-1}\| \nabla_{\theta}\fhL(\theta_t,\eta_{t+1})w_t\|\mathbf{1}(\htau = T)]\leq \frac{1}{T}\mathbb{E}[\sum_{t=0}^{T-1}\| \HGamma_tw_t\|\mathbf{1}(\htau = T)] + \frac{1}{T}\mathbb{E}[\sum_{t=0}^{T-1}\|X_tw_t\|\mathbf{1}(\htau = T)]\leq 8\delta\epsilon,\nonumber
\end{align}
and
\begin{align}
     \frac{1}{T}\mathbb{E}[\sum_{t=0}^{T-1}\|\nabla_{\eta}\fhL(\theta_t,\eta_t)w_t\|\mathbf{1}(\htau=T)] \leq \frac{1}{T}\mathbb{E}[\sum_{t=0}^{T-1}\| \HUpsilon_tw_t\|\mathbf{1}(\htau = T)] + \frac{1}{T}\mathbb{E}[\sum_{t=0}^{T-1}\|Z_tw_t\|\mathbf{1}(\htau = T)]\leq 8\delta\epsilon.
\end{align}
Additionally, we know
\begin{align}
    &\frac{1}{T}\mathbb{E}[\sum_{t=0}^{T-1}\|\nabla_{\eta}\fhL(\theta_t,\eta_{t+1})w_t\|\mathbf{1}(\htau = T)]\nonumber\\
    &\leq \frac{1}{T}\mathbb{E}[ \sum_{t=0}^{T-1}\|\nabla_{\eta}\fhL(\theta_t,\eta_{t+1})w_t - \nabla_{\eta}\fhL(\theta_t,\eta_t)w_t\|\mathbf{1}(\htau = T)]+\frac{1}{T}\mathbb{E}[\sum_{t=0}^{T-1}\|\nabla_{\eta}\fhL(\theta_t,\eta_t)w_t\|\mathbf{1}(\htau = T)]\nonumber\\
    &\leq\frac{1}{T}\mathbb{E}[\sum_{t=0}^{T-1}\sqrt{\sum_{i=1}^{m}\hL_2^2\gamma^2\mu_t^2(Z_t^iw_t^i)^2}\mathbf{1}(\htau = T)] + \frac{1}{T}\mathbb{E}[\sum_{t=0}^{T-1}\|\nabla_{\eta}\fhL(\theta_t,\eta_t)w_t\|\mathbf{1}(\htau= T)]\nonumber\\
    &\overset{(i)}{=} \frac{1}{T}\mathbb{E}[\sum_{t=0}^{T-1}\hL_2\gamma\mu_t\|Z_tw_t\|\mathbf{1}(\htau = T)]+ \frac{1}{T}\mathbb{E}[\sum_{t=0}^{T-1}\|\nabla_{\eta}\fhL(\theta_t,\eta_t)w_t\|\mathbf{1}(\htau= T)]\nonumber\\
    &\overset{(ii)}{\leq} \delta\epsilon+8\delta\epsilon = 9\delta\epsilon,
\end{align}
where (i) utilizes $\hL_2$-smooth property of $\nabla_{\eta^i}\fhL(\theta_t,\eta_t)$, and update equation $\eta_{t+1}^i-\eta_t^i=-\gamma \mu_tZ_t^iw_t^i$; (ii) utilizes $\mu_t\leq \frac{\delta\epsilon}{\|Z_tw_t \|}$ and $\gamma\leq \frac{1}{\hL_2}$.
Thus, in conclusion, we have
\begin{align}
    \frac{1}{T}\mathbb{E}[(\sum_{t=0}^{T-1}\|\nabla_{\theta}\fhL(\theta_t,\eta_{t+1})w_t\|+\|\nabla_{\eta}\fhL(\theta_t,\eta_{t+1})w_t \|)\mathbf{1}(\htau = T)]\leq 17\delta\epsilon.
\end{align}
By Markov inequality, this implies
\begin{align}
    P(
    \frac{1}{T}[\sum_{t=0}^{T-1}\|\nabla_{\theta}\fhL(\theta_t,\eta_{t+1})w_t\|+\|\nabla_{\eta}\fhL(\theta_t,\eta_{t+1})w_t \|]>34{\epsilon},\htau = T)\leq \frac{\delta}{2}.
\end{align}
Thus we have
\begin{align}
    &P(\frac{1}{T}[\sum_{t=0}^{T-1}\|\nabla_{\theta}\fhL(\theta_t,\eta_{t+1})w_t\|+\|\nabla_{\eta}\fhL(\theta_t,\eta_{t+1})w_t \|]{>} 34{\epsilon})\nonumber\\
    &\leq P(\tau<T) + P(
    \frac{1}{T}[\sum_{t=0}^{T-1}\|\nabla_{\theta}\fhL(\theta_t,\eta_{t+1})w_t\|+\|\nabla_{\eta}\fhL(\theta_t,\eta_{t+1})w_t \|]>34{\epsilon},\htau = T)\nonumber\\
    &\leq \frac{\delta}{2} + \frac{\delta}{2}= \delta,
\end{align}
This implies $P(\frac{1}{T}[\sum_{t=0}^{T-1}\|\nabla_{\theta,\eta}\fhL(\theta_t,\eta_{t+1})w_t\|]\leq 34{\epsilon})\geq 1-\delta$, which completes proof.
\end{proof}

\subsection{Descent Lemma of Algorithm~\ref{alg1-3}}\label{Appendix: Descent Lemma of Algorithm 1-3}
Under the same parameter choices stated in Theorem~\ref{restatethm: convergence of alg3}, we have the following descent lemma.
\begin{restatable}[Descent Lemma of Algorithm \ref{alg1-3}]{lemma}{maindescentlemmathmtwo}\label{lemma: descent lemma of alg1-3}
Under the same hyper-parameters choices stated in Theorem \ref{thm: convergence of alg3}, then for any fixed $w\in \mathcal{W}$, we have
    \begin{align}
        &\mathbb{E}\big[\fhL(\theta_{\htau}, \eta_{\htau})w\big]-\fhL^*w \leq \frac{\bar{F}\delta}{8m}-\mathbb{E}\Big[\frac{\gamma}{2}\sum_{t=0}^{\htau-1}\mu_t\|Z_tw_t\|^2+\frac{\gamma}{2}\sum_{t=0}^{\htau-1}\alpha_t\|X_tw_t\|^2 \Big],
        \label{eq: descent lemma alg1-3}
    \end{align}
holds.
\end{restatable}

\begin{proof} 
Denote $\mathcal{F}_t$ as $\sigma$-algebra of initialization, and batch samples before time $t$, i.e., $\sigma\big(\theta_0, \eta_0, w_0, \{\xi_s,\bar{\xi}_s\}_{s<t} \big)$. Define event $\calhatE_t = \{\forall i\in[m], \fhL(\theta_t,\eta_t^i) - \hL^{i,*}\leq \bar{F} \}$. Since $\calhatE_t \in \mathcal{F}_t$ and $\hat{\tau}_2$ detects an excessive objective gap at the next iterate, we have $\{t<\hat{\tau}\} \subseteq \calhatE_t$. Thus, on $\calhatE_t$, $\|\nabla_{\eta^i}\fhL(\theta,\eta^i)\|$ is bounded by $\Lambda$.

For term ${-\alpha_t \langle {X}_tw, X_tw_t \rangle-\mu_t\langle Z_tw,Z_tw_t\rangle}$, 
Expanding $\|w_{t+1}-w\|^2$, we have
\begin{align}
    \|{w}_{t+1} - {w} \|^2&= \| \Pi_{\mathcal{W}} \big({w}_t - \beta \big[{\alpha_t{X}_t^{\top}X_t{w}_t}+
   \mu_t{{Z}_t^{\top}Z_t{w}_t}+\rho {w}_t \big]\big) - {w} \|^2\nonumber\\
    &\leq \|{w}_t - \beta \big[{\alpha_t{X}_t^{\top}X_t{w}_t}+\mu_t{{Z}_t^{\top}Z_t{w}_t}+\rho {w}_t \big]- {w} \|^2 \nonumber\\
    &= \|{w}_t - {w} \|^2 -2\beta \langle {w}_t-{w}, {\alpha_t{X_t}^{\top}X_t{w}_t}+\rho {w}_t \rangle -2\beta\langle w_t-w,\mu_t{{Z}_t^{\top}Z_t{w}_t} \rangle \nonumber\\
    &\quad+ \beta^2\|\alpha_t{{X}_t^{\top}X_t{w}_t}+\rho {w}_t + \mu_t{{Z}_t^{\top}Z_t{w}_t} \|^2\nonumber \\
    &\leq \|{w}_t - {w} \|^2 -2\beta{\alpha_t} \langle {w}_t-{w}, {{X_t}^{\top}X_t{w}_t} \rangle -2\beta\mu_t\langle w_t-w,{{Z}_t^{\top}Z_t{w}_t} \rangle \nonumber\\
    &\quad + 3\beta^2\alpha_t^2 \|{{X}_t^{\top} X_t w_t}\|^2+{3\beta^2\mu_t^2}\|{Z}_t^{\top}Z_tw_t\|^2+3\beta^2 \rho^2 + {4}\beta\rho,
    \label{eq: coro MGDA from 1st principle}
\end{align}
where the first inequality is due to non-expansiveness of projection over probability simplex; the second inequality applies $(a+b+c)^2\leq 3a^2+3b^2+3c^2$ and $\|w\|\|w_t\|\leq 1$.

For $\|{{X}_t^{\top}X_t{w}_t}\|^2$, we further decompose it as
\begin{align}
    &\|{{X}_t^{\top}X_t{w}_t}\|^2 \nonumber\\
    &= \|{\big(\underbrace{{X}_t-\nabla_{\theta}\fhL(\theta_t, \eta_{t+1})}_{\hat{\Gamma}_{t}}+\nabla_{\theta}\fhL(\theta_t, \eta_{t+1})\big)^{\top}X_t{w}_t\|^2}\nonumber\\
    &\leq  2\|\HGamma_{t}\|^2_F\|X_tw_t \|^2 + 2\|\nabla_{\theta}\fhL(\theta_t, \eta_{t+1})\|_F^2 \|X_tw_t\|^2\nonumber\\
    & \leq 2\|\HGamma_{t}\|^2_F\|X_tw_t \|^2 + (6mG^2+6\gamma^2\mu_t^2\hL_2^2\|Z_tw_t\|^2+6m\Lambda^2)\|X_tw_t\|^2\nonumber\\
    &\leq2\|\HGamma_{t}\|^2_F\|X_tw_t \|^2 + (6mG^2+6\gamma^2\hL_2^2{f_2}^2+6m\Lambda^2)\|X_tw_t\|^2,
    \label{eq: upper bound MGDA pre-conditioning}
\end{align}
where the last three inequalities leverages Cauchy-Schwarz inequality, Multiplicative property of Frobenius norm, $(a+b)^2\leq 2a^2+2b^2$, and \eqref{eq: gradient relationship 2} respectively. 

Similarly, $\|{Z}_t^{\top}Z_tw_t\|^2$ can be upper bounded as
\begin{align}
    \|{Z}_t^{\top}Z_tw_t\|^2&=\|({\HUpsilon}_t+\nabla_{\eta}\fhL(\theta_t,\eta_t))^{\top}Z_tw_t\|^2\nonumber\\
    &\leq 2\|{\HUpsilon}_t^{\top}Z_tw_t\|^2+2\|\nabla_{\eta}\fhL(\theta_t,\eta_t)^{\top}Z_tw_t\|^2\nonumber\\
    &\leq 2\| \HUpsilon_t\|^2_F\|Z_tw_t\|^2+2m\Lambda^2\|Z_tw_t\|^2,
    \label{eq: upper bound MGDA pre-conditioning 2}
\end{align}
where above inequalities leverages $(a+b)^2\leq 2a^2+2b^2$; Cauchy-Schwarz inequality, sub-multiplicative property of Frobenius norm and $\|\nabla_{\eta}\fhL(\theta_t,\eta_t)\|_F\leq \sqrt{m}\Lambda$, respectively.

Merge \eqref{eq: upper bound MGDA pre-conditioning}, \eqref{eq: upper bound MGDA pre-conditioning 2} into \eqref{eq: coro MGDA from 1st principle}, re-organize it and multiply both sides with $\gamma/2\beta$, we have
\begin{align}
    &-\gamma\alpha_t\langle X_t{w},X_t{w}_t \rangle -\gamma{\mu_t}\langle{Z}_tw, Z_tw_t \rangle\nonumber \\
    &\leq -\gamma\alpha_t \|X_tw_t\|^2-\gamma\mu_t\| Z_tw_t\|^2+\frac{\gamma}{2\beta}(\|w_t-w\|^2-\|w_{t+1}-w\|^2)+\frac{3\gamma\beta\rho^2}{2}+2\gamma\rho\nonumber\\
    &\quad +3\beta\gamma\alpha_t^2 \|\HGamma_{t}\|_F^2 \|X_t w_t\|^2+9\beta\gamma\alpha_t^2\underbrace{(mG^2+\gamma^2\hL_2^2{f_2^2}+m\Lambda^2)}_{\Xi'_1}\|X_tw_t\|^2\nonumber\\
    &\quad+ 3\beta\gamma\mu_t^2\| {\HUpsilon}_t\|^2_F\|Z_tw_t\|^2+3\beta\gamma\mu_t^2m\Lambda^2\|Z_tw_t\|^2\nonumber\\
    &\leq-\gamma\alpha_t \|X_tw_t\|^2-\gamma\mu_t\| Z_tw_t\|^2+\frac{\gamma}{2\beta}(\|w_t-w\|^2-\|w_{t+1}-w\|^2)+\frac{3\gamma\beta\rho^2}{2}+2\gamma\rho \nonumber\\
    &\quad+ 3\gamma\beta c_2^2\|\HGamma_t\|_F^2 + 9\gamma\beta\Xi_1'c_2^2 + 3\beta\gamma{f_2^2}\|\HUpsilon_t\|_F^2+3\beta\gamma{f_2^2}m\Lambda^2,
    \label{eq: coro MDGA}
\end{align}
where the last inequality leverages clipping structure $\alpha_t = \min\{c_1,\frac{c_2}{\|X_tw_t\|} \} \text{ and }\mu_t = \min\{f_1,\frac{f_2}{\|Z_tw_t\|} \}$.

Next, for descent lemma with respect to $\nabla_{\eta}\fhL(\theta_t, \eta_t)$~\eqref{eq: moo descent lemma eta}, Merge $\eta_{t+1}=\eta_t-\gamma\mu_tZ_tw_t$, we have
\begin{align}
    \fhL(\theta_t,\eta_{t+1})w
    &\leq \fhL(\theta_t, \eta_t)w - {\gamma\mu_t} \langle \nabla_{\eta}\fhL(\theta_t, \eta_t)w,Z_tw_t \rangle + \frac{\hL_2{\gamma^2\mu_t^2}}{2}\|Z_tw_t\|^2\nonumber\\
    &=\fhL(\theta_t, \eta_t)w +{\gamma\mu_t} \langle{\HUpsilon}_tw,Z_tw_t\rangle-{\gamma\mu_t}\langle{Z}_tw,Z_tw_t \rangle + \frac{\hL_2\gamma^2\mu_t^2}{2}\|Z_tw_t\|^2,
    \label{eq: coro descent 1}
\end{align}
where the last equality decomposes $\nabla_{\eta}\fhL(\theta_t,\eta_t)w=(\nabla_{\eta}\fhL(\theta_t,\eta_t)-Z_t+Z_t)w = (-\HUpsilon_t+Z_t)w$.

For descent lemma with respect to $\nabla_{\theta}\fhL(\theta_t, \eta_{t+1})$~\eqref{eq: moo descent lemma theta}, Merge update rule $\theta_{t+1} = \theta_{t}-\gamma\alpha_tX_tw_t$ into descent lemma \eqref{eq: moo descent lemma theta}, we have
\begin{align}
    &\fhL({\theta}_{t+1}, {\eta}_{t+1}){w}\nonumber\\
    &\leq \fhL({\theta}_{t}, {\eta}_{t+1}){w} + \langle \nabla_{\theta}\fhL({\theta}_{t}, {\eta}_{t+1})w,{\theta}_{t+1}-{\theta}_t \rangle + \frac{\hL_0}{2}\|{\theta}_{t+1}-{\theta}_t \|^2+ \frac{\hL_1 | \nabla_{\eta} \fhL({\theta}_t, {\eta}_{t+1})w|}{2}\|{\theta}_{t+1}-{\theta}_t \|^2  \nonumber \\
    &=  \fhL({\theta}_t, {\eta}_{t+1}){w} -\gamma\alpha_t \langle \nabla_{\theta}\fhL({\theta}_t, {\eta}_{t+1})w,X_tw_t \rangle + \frac{\gamma^2\alpha_t^2\hL_0}{2}\|X_tw_t\|^2 \nonumber\\
    &\quad + \frac{\gamma^2\alpha_t^2\hL_1 | \nabla_{\eta} \fhL({\theta}_t, {\eta}_{t+1})w|}{2} \|X_tw_t\|^2\nonumber\\
    &\leq \fhL({\theta}_t, {\eta}_{t+1}){w} -\gamma\alpha_t \langle (\underbrace{\nabla_{\theta}\fhL({\theta}_t, {\eta}_{t+1})-{{X}_t}}_{-\HGamma_t})w,X_tw_t \rangle - \gamma\alpha_t \langle {{X}_tw},X_tw_t \rangle+ \frac{\hL_0}{2}\alpha_t^2\gamma^2\|X_tw_t\|^2  \nonumber\\
    &\quad  + \frac{\hL_1\Lambda+{\gamma\mu_t\hL_1\hL_2\|Z_tw_t\|}}{2}\alpha_t^2\gamma^2\| X_tw_t \|^2,
    \label{eq: coro descent 2}
\end{align}
where the last inequality utilizes  $|\nabla_{\eta} \fhL(\theta_t, \eta_{t+1})w|\leq \hL_2\gamma\mu_t\|Z_tw_t\| + |\nabla_{\eta}\fhL(\theta_t, \eta_t)w|\leq\hL_2\gamma\mu_t\|Z_tw_t\|+\Lambda$ via gradient-clipping update,
$\eta_{t+1}=\eta_t-\gamma\mu_tZ_tw_t$.

Merge \eqref{eq: coro MDGA},\eqref{eq: coro descent 1}, \eqref{eq: coro descent 2}, we have
\begin{align}
    &\fhL(\theta_{t+1},\eta_{t+1})w\nonumber\\
    &\leq \fhL(\theta_t, \eta_t)w +{\gamma\mu_t} \langle{\HUpsilon}_tw,Z_tw_t\rangle+\gamma\alpha_t \langle\HGamma_tw,X_tw_t \rangle-{\gamma\mu_t}\langle{Z}_tw,Z_tw_t \rangle - \gamma\alpha_t \langle {{X}_tw},X_tw_t \rangle \nonumber\\
    &\quad + \frac{\hL_0}{2}\alpha_t^2\gamma^2\|X_tw_t\|^2+ \frac{\hL_2\gamma^2\mu_t^2}{2}\|Z_tw_t\|^2+ \frac{\hL_1\Lambda+{\gamma\mu_t \hL_1\hL_2\|Z_tw_t\|}}{2}\alpha_t^2\gamma^2\| X_tw_t \|^2\nonumber\\
    &\leq  \fhL(\theta_t, \eta_t)w +{\gamma\mu_t} \langle{\HUpsilon}_tw,Z_tw_t\rangle+\gamma\alpha_t \langle\HGamma_tw,X_tw_t \rangle -\gamma\alpha_t \|X_tw_t\|^2-\gamma\mu_t\| Z_tw_t\|^2 \nonumber\\
    &\quad +\frac{\gamma}{2\beta}(\|w_t-w\|^2-\|w_{t+1}-w\|^2)+\frac{3\gamma\beta\rho^2}{2}+2\gamma\rho + 3\gamma\beta c_2^2\|\HGamma_t\|_F^2 + 9\gamma\beta\Xi_1'c_2^2+ 3\beta\gamma{f_2^2}\|\HUpsilon_t\|_F^2\nonumber\\
    &\quad +3\beta\gamma{f_2^2}m\Lambda^2+ \frac{\hL_0}{2}\alpha_t^2\gamma^2\|X_tw_t\|^2+ \frac{\hL_2\gamma^2\mu_t^2}{2}\|Z_tw_t\|^2+ \frac{\hL_1\Lambda+{\gamma\mu_t \hL_1\hL_2\|Z_tw_t\|}}{2}\alpha_t^2\gamma^2\| X_tw_t \|^2\nonumber\\
    &\leq \fhL(\theta_t, \eta_t)w +{\gamma\mu_t} \langle{\HUpsilon}_tw,Z_tw_t\rangle+\gamma\alpha_t \langle\HGamma_tw,X_tw_t \rangle -\gamma\alpha_t \|X_tw_t\|^2-\gamma\mu_t\| Z_tw_t\|^2 \nonumber\\
    &\quad +\frac{\gamma}{2\beta}(\|w_t-w\|^2-\|w_{t+1}-w\|^2)+\frac{3\gamma\beta\rho^2}{2}+2\gamma\rho + 3\gamma\beta c_2^2\|\HGamma_t\|_F^2 + 9\gamma\beta\Xi_1'c_2^2+ 3\beta\gamma{f_2^2}\|\HUpsilon_t\|_F^2\nonumber\\
    &\quad +3\gamma\beta{f_2^2}m\Lambda^2+ \frac{\hL_0}{2}\alpha_t^2\gamma^2\|X_tw_t\|^2+ \frac{\hL_2\gamma^2\mu_t^2}{2}\|Z_tw_t\|^2+ \frac{\hL_1\Lambda}{2}\alpha_t^2\gamma^2\| X_tw_t \|^2 \nonumber\\
    &\quad+\vartheta\gamma^2\mu_t^2\hL_1^2\hL_2^2\|Z_tw_t\|^2+\frac{\gamma^4c_2^4}{16\vartheta}\nonumber\\
    &=\fhL(\theta_t,\eta_t)w+{\gamma\mu_t} \langle{\HUpsilon}_tw,Z_tw_t\rangle+\gamma\alpha_t \langle\HGamma_tw,X_tw_t \rangle-\gamma\alpha_t(1-\frac{\hL_0+\hL_1\Lambda}{2}\gamma\alpha_t)\|X_tw_t\|^2\nonumber\\
    &\quad - \gamma\mu_t(1-\frac{\hL_2+2\vartheta\hL_1^2\hL_2^2}{2}\gamma\mu_t)\|Z_tw_t\|^2+\frac{\gamma}{2\beta}(\|w_t-w\|^2-\|w_{t+1}-w\|^2)+\frac{3\gamma\beta\rho^2}{2}+2\gamma\rho  \nonumber\\
    &\quad+ 3\gamma\beta c_2^2\|\HGamma_t\|_F^2+ 9\gamma\beta\Xi_1'c_2^2+ 3\beta\gamma{f_2^2}\|\HUpsilon_t\|_F^2+3\gamma\beta{f_2^2}m\Lambda^2 + \frac{\gamma^4c_2^4}{16\vartheta},
    \label{eq: moo descent alg3 intermediate}
\end{align}
where we apply young's inequality $\frac{ab}{2}\leq  \vartheta a^2+\frac{b^2}{16\vartheta}$ on $\frac{\hL_1\hL_2\gamma\mu_t\|Z_tw_t\|\cdot\gamma^2\alpha_t^2\|X_tw_t\|^2}{2}$,  and clipping rule ${\alpha_t = \min\{c_1, \frac{c_2}{\|X_tw_t\|} \}}\leq c_1 \text{ and }\frac{c_2}{\|X_tw_t\|}$ and $\mu_t$.

For $\gamma\mu_t\langle \HUpsilon_tw,Z_tw_t\rangle+\gamma\alpha_t\langle \HGamma_tw, X_tw_t\rangle$, we have
\begin{align}
    \gamma\mu_t\langle \HUpsilon_tw,Z_tw_t\rangle+\gamma\alpha_t\langle \HGamma_tw, X_tw_t\rangle&\leq \gamma\mu_t\| \HUpsilon_tw\| \|Z_tw_t\| + \gamma\alpha_t\|\HGamma_tw\|\|X_tw_t\|\nonumber\\
    &\leq \gamma {f_2}\|\HUpsilon_tw\|+\gamma c_2\| \HGamma_tw\|,
    \label{eq: non-martingale-sequence alg3}
\end{align}
where the first inequality utilizes Cauchy–Schwarz inequality and the second inequality utilizes $\alpha_t\leq \frac{c_2}{\|X_tw_t\|}$ and $\mu_t\leq \frac{f_2}{\| Z_tw_t\|}$. This further reduces \eqref{eq: moo descent alg3 intermediate} into
\begin{align}
    &\fhL(\theta_{t+1},\eta_{t+1})w\nonumber\\
    &\leq \fhL(\theta_t,\eta_t)w+\gamma {f_2}\|\HUpsilon_tw\|+\gamma c_2\| \HGamma_tw\|-\gamma\alpha_t(1-\frac{\hL_0+\hL_1\Lambda}{2}\gamma\alpha_t)\|X_tw_t\|^2\nonumber\\
    &\quad - \gamma\mu_t(1-\frac{\hL_2+2\vartheta\hL_1^2\hL_2^2}{2}\gamma\mu_t)\|Z_tw_t\|^2+\frac{\gamma}{2\beta}(\|w_t-w\|^2-\|w_{t+1}-w\|^2)+\frac{3\gamma\beta\rho^2}{2}+2\gamma\rho  \nonumber\\
    &\quad+ 3\gamma\beta c_2^2\|\HGamma_t\|_F^2+ 9\gamma\beta\Xi_1'c_2^2+ 3\beta\gamma{f_2^2}\|\HUpsilon_t\|_F^2+3\gamma\beta{f_2^2}m\Lambda^2 + \frac{\gamma^4c_2^4}{16\vartheta}.
    \label{eq: one-step moo descent alg3}
\end{align}
Taking conditional expectation over $\xi_t$, for $\|\HUpsilon_tw\|$, we have
\begin{align}
    \mathbf{1}(\calhatE_t)\mathbb{E}_{\xi_t}[\|\HUpsilon_tw\||\mathcal{F}_t]\leq \mathbf{1}(\calhatE_t) \sqrt{\mathbb{E}_{\xi_t}[\|\HUpsilon_tw\|^2|\mathcal{F}_t]}\leq \mathbf{1}(\calhatE_t)\sqrt{\hat{K}_2/N_2},
    \label{eq: variance bound thm 3}
\end{align}
where the last inequality follows from variance upper bound of $\HUpsilon_t^i,\forall i\in[m]$~\eqref{eq: variance bound for rescaled function}.

Following similar proof logic as~\eqref{eq: estimation error bound thm2},  taking conditional expectation over ${\xi}_t,\bar{\xi}_t$, $ \|\hat{\Gamma}_{t}\|_F^2$ becomes
\begin{align}
    &\mathbf{1}(\calhatE_t)\mathbb{E}_{{\xi}_t,\bar{\xi}_t,\eta_{t+1}}\Big[  \|\hat{\Gamma}_{t}\|_F^2|\mathcal{F}_t\Big]\nonumber\\
    &\overset{(i)}{\leq} \mathbf{1}(\calhatE_t)\mathbb{E}_{{\xi}_t,\bar{\xi}_t,\eta_{t+1}}\Big[\frac{m\hat{K}_0}{N_1}+\frac{\hat{K}_1}{N_1}\|\nabla_{\eta}\fhL(\theta_t, \eta_{t+1})\|_F^2|\mathcal{F}_t\Big]\nonumber\\
    &\overset{}{\leq} \mathbf{1}(\calhatE_t) \mathbb{E}_{{\xi}_t}\Big[ \frac{m\hat{K}_0}{N_1}+\frac{\hat{K}_1}{N_1}(2\hL_2^2\gamma^2\mu_t^2\|Z_tw_t\|^2+2m\Lambda^2))|\mathcal{F}_t\Big]\nonumber\\
    &\overset{(iii)}{\leq} \mathbf{1}(\calhatE_t) \mathbb{E}[
    \frac{m\hat{K}_0}{N_1}+\frac{\hat{K}_1}{N_1}(2\hL_2^2\gamma^2\mu_t^2\|Z_tw_t\|^2+2m\Lambda^2))|\mathcal{F}_t]\nonumber\\
    &\overset{(iv)}{\leq}  \mathbf{1}(\calhatE_t)\big(\frac{m\hat{K}_0}{N_1}+\frac{2\hL_2^2\gamma^2f_2^2\hat{K_1}}{N_1}+\frac{2m\Lambda^2\hat{K}_1}{N_1}\big)=\mathbf{1}(\calhatE_t){\Xi'_2},
    \label{eq: variance intermediate bound alg3}
\end{align}
where the (i) applies the fact $\mathbb{E}_{\bar{\xi}_t^i}[\|\HGamma_{t}^i\|^2|\mathcal{F}_t,\xi_t]\leq (\hat{K}_0+\hat{K}_1\|\nabla_{\eta^i}\fhL(\theta_t,\eta^i_{t+1})\|^2)/N_1$; (iii) further upper bound $\|\nabla_{\eta}\fhL(\theta_t, \eta_{t+1})\|^2_F=\sum_{i=1}^{m}\|\nabla_{\eta^i}\fhL(\theta_t, \eta_{t+1}^i)\|^2\leq \sum_{i=1}^{m}(\hL_2\gamma\mu_t \|Z_t^iw_t^i\|+\Lambda)^2\leq 2\hL_2^2\gamma^2\mu_t^2\|Z_tw_t\|^2+2m\Lambda^2$; (iii) leverages $\mu_t\leq \frac{f_2}{\|Z_tw_t\|}$. 

Similarly, repeating above conditional on $\mathcal F_t$, and using the gradient bound on $\calhatE_t$, we obtain, for any fixed \(w\in\mathcal W\),
$$
\mathbf1({\calhatE_t})
\mathbb E[\|\widehat\Gamma_t w\|\mid\mathcal F_t]
\le \mathbf1({\widehat{\mathcal E}_t})\Xi'_3.
$$

Combining~\eqref{eq: estimation error bound thm2},~\eqref{eq: variance bound thm 3},~\eqref{eq: variance intermediate bound alg3} into~\eqref{eq: one-step moo descent alg3}, sum it from $0$ to $\htau-1$. Since $\{t< \htau\}\subseteq \calhatE_t$, $\mathbf{1}(\calhatE_t)=1$, {take expectation on both sides} and utilizing $\|w_0-w\|^2\leq 2$, we have
\begin{align}
    &\mathbb{E}[\fhL(\theta_{\htau},\eta_{\htau})w]-\fhL^*w\nonumber\\
    &\leq \fhL(\theta_0,\eta_0)w-\fhL^*w-\mathbb{E}[\sum_{t=0}^{\htau-1}\gamma\alpha_t(1-\frac{\hL_0+\hL_1\Lambda}{2}\gamma\alpha_t)\|X_tw_t\|^2]\nonumber\\
    &\quad -\mathbb{E}[\sum_{t=0}^{\htau-1}\gamma\mu_t(1-\frac{\hL_2+2\vartheta\hL_1^2\hL_2^2}{2}\gamma\mu_t)\|Z_tw_t\|^2] +\mathbb{E}[\sum_{t=0}^{\htau-1}\gamma{f_2}\|\HUpsilon_tw\|]+\mathbb{E}[\sum_{t=0}^{\htau-1}\gamma  c_2\| \HGamma_tw\|]\nonumber\\
    &\quad+\frac{\gamma}{\beta}+\frac{3\gamma\beta\rho^2T}{2}
    +2\gamma\rho T + \mathbb{E}[\sum_{t=0}^{\htau-1}3\gamma\beta c_2^2\|\HGamma_t\|_F^2] + \mathbb{E}[\sum_{t=0}^{\htau-1}[9\gamma\beta\Xi_1'c_2^2]+ \mathbb{E}[\sum_{t=0}^{\htau-1}3\beta\gamma{f_2^2}\|\HUpsilon_t\|_F^2]\nonumber\\
    &\quad+\mathbb{E}[\sum_{t=0}^{\htau-1}3\beta\gamma{f_2^2}m\Lambda^2]+ \frac{\gamma^4c_2^4T}{16\vartheta}\nonumber\\
    &\leq \fhL(\theta_0,\eta_0)w-\fhL^*w-\mathbb{E}[\sum_{t=0}^{\htau-1}\gamma\alpha_t(1-\frac{\hL_0+\hL_1\Lambda}{2}\gamma\alpha_t)\|X_tw_t\|^2] \nonumber\\
    &\quad -\mathbb{E}[\sum_{t=0}^{\htau-1}\gamma\mu_t(1-\frac{\hL_2+2\vartheta\hL_1^2\hL_2^2}{2}\gamma\mu_t)\|Z_tw_t\|^2]+\gamma T{f_2}\sqrt{\hat{K}_2/N_2}+\gamma c_2\Xi_3'T+\frac{\gamma}{\beta}+\frac{3\gamma\beta\rho^2T}{2} \nonumber\\
    &
    \quad +2\gamma\rho T+ 3\gamma\beta c_2^2\Xi_2' T + 9\gamma\beta\Xi_1'c_2^2 T+ 3m\beta\gamma{f_2^2}(\hat{K}_2/N_2)T+3m\beta\gamma{f_2^2}\Lambda^2T + \frac{\gamma^4c_2^4T}{16\vartheta}\nonumber\\
    &\leq \fhL(\theta_0,\eta_0)w-\fhL^*w-\mathbb{E}[\frac{\gamma}{2}\sum_{t=0}^{\htau-1}\alpha_t\|X_tw_t\|^2] -\mathbb{E}[\frac{\gamma}{2}\sum_{t=0}^{\htau-1}\mu_t\|Z_tw_t\|^2] +\underbrace{\gamma T{f_2}\sqrt{\hat{K}_2/N_2}}_{\leq d'_1}+\underbrace{\gamma c_2\Xi_3'T}_{\leq d'_2}\nonumber\\
    &\quad +\underbrace{\frac{\gamma}{\beta}}_{\leq d'_3}+\underbrace{2\gamma\rho T}_{\leq d'_4}+\underbrace{\frac{3\gamma\beta\rho^2T}{2}}_{\leq d'_5} 
    + \underbrace{3\gamma\beta c_2^2\Xi_2' T + 9\gamma\beta\Xi_1'c_2^2 T}_{\leq d'_6}+ \underbrace{3m\beta\gamma{f_2^2}(\hat{K}_2/N_2)T+3m\beta\gamma{f_2^2}\Lambda^2T}_{\leq d'_7} \nonumber\\
    &\quad + \underbrace{\frac{\gamma^4c_2^4T}{16\vartheta}}_{\leq d'_{8}}\nonumber\\
    &\leq \hat{\Delta}_{\theta_0, \eta_0} -\mathbb{E}[\frac{\gamma}{2}\sum_{t=0}^{\htau-1}\alpha_t\|X_tw_t\|^2] -\mathbb{E}[\frac{\gamma}{2}\sum_{t=0}^{\htau-1}\mu_t\|Z_tw_t\|^2] + d_1'+\dots d_8'    \nonumber\\
    &\leq \frac{\bar{F}\delta}{8m} -\frac{\gamma}{2}\mathbb{E}[\sum_{t=0}^{\htau-1}\alpha_t\| X_tw_t\|^2]-\frac{\gamma}{2}\mathbb{E}[\sum_{t=0}^{\htau-1}\mu_t\|Z_tw_t\|^2],
    \label{eq: stopping-time bound alg3}
\end{align}
where the first inequality $\mathbb{E}[\sum_{t=0}^{\htau-1}{\|\hat{\Gamma}}\|_F^2] \leq \mathbb{E}[\sum_{t=0}^{T-1}{\mathbf{1}(\calhatE_t)\|\hat{\Gamma}}\|_F^2] \leq \sum_{t=0}^{T-1}\mathbb{E}[\mathbf{1}(\calhatE_t)\|\HGamma\|_F^2]\leq \sum_{t=0}^{T-1}\mathbb{E}[\mathbf{1}(\calhatE_t)\mathbb{E}[\|\HGamma\|_F^2|\mathcal{F}_t]]\leq T\Xi_2'$, a similar argument also applies to upper bound $\mathbb{E}[\sum_{t=0}^{\htau-1}\|\HGamma_tw\|]\leq T\Xi_3'$,  $\mathbb{E}[\sum_{t=0}^{\htau-1}\|\HUpsilon_t\|_F^2]\leq Tm\hat{K}_2/N_2$, $\mathbb{E}[\sum_{t=0}^{\htau-1}\|\HUpsilon_tw\|]\leq T\sqrt{\hat{K}_2/N_2}$; The last inequality follows the parameter choice of $\gamma,\beta, \rho, \alpha_t,\mu_t$, $N_1$, $N_2$, construction of $\bar{F}$, and $\htau \leq T$.
\end{proof}

\begin{lemma}\label{lemma: conditioned GammaXi}
   For event $\{\htau = T \}$, under the parameter choices as Theorem~\ref{thm: convergence of alg3}, we have
   \begin{align}
        \mathbb{E}[\|\HGamma_tw_t\|\mathbf{1}(\htau = T)]\leq \delta\epsilon \text{ and }\mathbb{E}[\|\HUpsilon_tw_t\|\mathbf{1}(\htau = T)]\leq \delta\epsilon.
   \end{align}
\end{lemma}
\begin{proof}
    Notice $\theta_t, \eta_t$ is $\mathcal{F}_t$-measurable. Moreover, since $\calhatE_t$ only depends on $\theta_t,\eta_t$, we have $\calhatE_t \in \mathcal{F}_t$.
    At $\htau = T$, every state encountered before $T$ is safe. Therefore we conclude $\{\htau = T \}\subseteq \calhatE_t$ and $\mathbf{1}(\htau = T)\leq \mathbf{1}(\calhatE_t)$. Thus, applying tower property, we have
    \begin{align}
        \mathbb{E}[\|\HGamma_tw_t\|^2\mathbf{1}(\htau = T)] &\leq \mathbb{E}[\|\HGamma_tw_t\|^2\mathbf{1}(\calhatE_t)]\nonumber\\
        &= \mathbb{E}[\mathbb{E}[\|\HGamma_tw_t\|^2\mathbf{1}(\calhatE_t)|\mathcal{F}_t]]\nonumber\\
        & = \mathbb{E}[\mathbf{1}(\calhatE_t)\mathbb{E}[\|\HGamma_tw_t\|^2|\mathcal{F}_t]]\nonumber\\
        & =  \mathbb{E}[\mathbf{1}(\calhatE_t)\mathbb{E}[\mathbb{E}[\|\HGamma_tw_t\|^2|\mathcal{F}_t,\xi_t]|\mathcal{F}_t]],
    \end{align}
    For $\mathbb{E}[\|\HGamma_tw_t\|^2|\mathcal{F}_t,\xi_t]$, leveraging variance bound Lemma~\ref{lemma: affine bounded noise}, we know
    \begin{align}
        \mathbb{E}[\|\HGamma_t^i\|^2 |\mathcal{F}_t, \xi_t] \leq \frac{\hat{K}_0+\hat{K}_1|\nabla_{\eta^i} \fhL^i(\theta_t, \eta_{t+1}^i)|^2}{N_1}\leq \frac{\hat{K}_0+2\hat{K}_1( \Lambda^2+{\hL}_2^2 \gamma^2 f_2^2) }{N_1} \text{ on } \calhatE_t,
    \end{align}
    where the last inequality leverages the $\hL_2$-smooth in terms of $\nabla_{\eta}\fhL^i(\theta_t, \eta_{t})$ such that $|\nabla_{\eta^i} \widehat{L}^i\left(\theta_t, \eta_{t+1}^i\right)| \leq \Lambda+\hL_2 \|\eta_{t+1}-\eta_t\|$ and clipping step-size design such that $\|\eta_{t+1}-\eta_t\|\leq \gamma \mu_t \|Z_tw_t\|\leq \gamma f_2$.
    Substitute parameter choice of $N_1 = \max\big\{9\hat{K}_0\delta^{-3}\epsilon^{-2}, 18\hat{K}_1\delta^{-1}, 18\Lambda^2\hat{K}_1\delta^{-3}\epsilon^{-2}\}$ gives 
    \begin{align}
        \mathbb{E}[\|\HGamma_tw_t\|^2\mathbf{1}(\htau=T)]\leq \frac{\hat{K}_0+2\hat{K}_1( \Lambda^2+{\hL}_2^2 \gamma^2 f_2^2) }{N_1} P(\calhatE_t)\leq \frac{\delta^3\epsilon^2}{3}.
    \end{align}
    Leveraging Cauchy-schwarz inequality, we have
    \begin{align}
        \mathbb{E}[\|\HGamma_tw_t\|\mathbf{1}(\htau=T)] 
        &\leq \sqrt{\mathbb{E}[(\|\HGamma_tw_t\|^2\mathbf{1}(\htau = T))P(\htau=T)}\nonumber\\
        &\leq \frac{\delta^{3/2}\epsilon}{\sqrt{3}}\leq \delta\epsilon,
    \end{align}
    which completes proof. Similar arguments also applies to $\mathbb{E}[\|\HUpsilon_tw_t\|\mathbf{1}(\htau = T)]$.
\end{proof}

%
%


\end{document}